\documentclass{amsart}

\makeatletter

\newtheorem{theorem}{Theorem}[section]
\newtheorem{proposition}{Proposition}[section]
{
	\theoremstyle{plain}
	\newtheorem{assumption}{A.\hspace{-1.1mm}}
}
\newtheorem{lemma}[theorem]{Lemma}

\theoremstyle{definition}
\newtheorem{definition}[theorem]{Definition}

\usepackage[margin=1.5in]{geometry} 
\usepackage{stackrel}

\usepackage{hyperref}
\usepackage{xcolor,colortbl}
\usepackage{hyperref}
\hypersetup{
	colorlinks=true,
	linkcolor=blue,
	filecolor=magenta,      
	urlcolor=cyan,
}

\definecolor{Gray}{gray}{0.5}
\definecolor{LightCyan}{rgb}{0.88,1,1}

\usepackage{amsmath}
\usepackage{accents}

\newcommand{\expect}{{\rm I\!E}}
\newcommand{\real}{{\rm I\!R}}

\newcommand{\integer}{{\rm I\!N}}
\newcommand{\prob}{{\rm I\!P}}
\newcommand{\median}{{\rm I\!M}}
\newcommand{\ball}{{\rm I\!B}}

\usepackage{amsmath,amsthm,amssymb,bm}
\usepackage{algpseudocode}
\usepackage{mathtools}
\usepackage{listings}
\usepackage{color} 
\usepackage{graphicx}
\usepackage{easybmat}
\usepackage{bm}
\usepackage{subfigure}
\usepackage{hyperref}
\usepackage{algorithm}
\usepackage{dsfont}
\usepackage{enumitem}
\usepackage{amsmath}
\usepackage{algpseudocode}
\usepackage{bbm}
\usepackage{dsfont}
\usepackage{epstopdf}
\usepackage{graphicx,subfigure}
\usepackage[labelfont=bf]{caption}
\usepackage[export]{adjustbox}
\usepackage{amsbsy}
\usepackage{romannum}
\usepackage{tikz}
\usepackage{lettrine}

\let\exp\relax

\DeclareMathOperator{\exp}{exp}

\algnewcommand{\IfThenElse}[3]{
	\State \algorithmicif\ #1\ \algorithmicthen\ #2\ \algorithmicelse\ #3}

\let \expect \relax
\newcommand{\expect}{{\rm I\!E}}

\usepackage{filecontents}
\usepackage{mathrsfs}

\newcommand{\df}{\stackrel{\text{def}}{=}}

\theoremstyle{remark}
\newtheorem{remark}[theorem]{Remark}

\usepackage{accents}

\definecolor{Gray}{RGB}{192,192,192}

\newcommand{\vertiii}[1]{{\left\vert\kern-0.25ex\left\vert #1 
		\right\vert\kern-0.25ex\right\vert\kern}}

\usepackage{hyperref}
\hypersetup{
	colorlinks=true,
	linkcolor=[RGB]{110,25,25},
	citecolor=[RGB]{25,25,110},
	filecolor=magenta,      
	urlcolor=black,
}
\usepackage{lipsum}

\numberwithin{equation}{section}

\begin{document}
	
	\title[A Robust Optimization Method for Adversarial Multiple Kernel Learning]{A Distributionally Robust Optimization Method for Adversarial Multiple Kernel Learning}

	
	

	\author[Masoud Badiei Khuzani, et al.]{\hspace{-1mm}Masoud Badiei Khuzani$^{\dagger}$, Hongyi Ren$^{\ddagger}$, MD Tauhidul Islam$^{\ddagger}$, Lei Xing$^{\ddagger}$
}

	\address{Stanford University, 450 Serra Mall, Stanford, CA 94305\\
		\\	
	}
	\curraddr{}
	\email{mbadieik,hongyi,devan,tauhid,lei@stanford.edu}

	{\let\thefootnote\relax\footnotetext{
			\hspace{-4mm}$^{\dagger}$Department of Management Science, 	Stanford University, CA 94043, USA.\\
			$^{\ddagger}$Department of Radiation Oncology, Stanford University, CA 94043, USA.\\
			\textit{Revised paper. Current Revision: April 2021, Last Edition: December 2018.}
			\vspace{2mm}
		\vspace{-.2mm}
	}}

\begin{abstract}
We propose a novel data-driven method to learn a mixture of multiple kernels with random features that is certifiabaly robust against adverserial inputs. Specifically, we consider a distributionally robust optimization of the kernel-target alignment with respect to the distribution of training samples over a distributional ball defined by the Kullback-Leibler (KL) divergence. The distributionally robust optimization problem can be recast as a min-max optimization whose objective function includes a log-sum term. We develop a mini-batch biased stochastic primal-dual proximal method to solve the min-max optimization. To debias the minibatch algorithm, we use the Gumbel perturbation technique to estimate the log-sum term. We establish theoretical guarantees for the performance of the proposed multiple kernel learning method. In particular, we prove the consistency, asymptotic normality, stochastic equicontinuity, and the minimax rate of the empirical estimators.  In addition, based on the notion of Rademacher and Gaussian complexities, we establish distributionally robust generalization bounds that are tighter than previous known bounds. More specifically, we leverage matrix concentration inequalities to establish distributionally robust generalization bounds. We validate our kernel learning approach for classification with the kernel SVMs on synthetic dataset generated by sampling multvariate Gaussian distributions with differernt variance structures. We also apply our kernel learning approach to the MNIST data-set and evaluate its robustness to perturbation of input images under different adversarial models. More specifically, we examine the robustness of the proposed kernel model selection technique against FGSM, PGM, C\&W, and DDN adversarial perturbations, and compare its performance with alternative state-of-the-art multiple kernel learning paradigms. 
\end{abstract}
	\maketitle
\pagestyle{plain}

\section{Introduction}

\lettrine{\textbf{K}}ernel methods are a class of statistical machine learning algorithms that capture the non-linear relationship between the representations of input data and class labels. The kernel methods circumvent explicit feature mapping that is required to learn a non-linear function or decision boundary in linear learning algorithms. Instead, these methods only rely on the inner product of feature maps in the feature space, which is often known as the ``kernel trick" in machine learning literatures. In the kernel methods, the underlying kernel function must be specified by the user which poses a model selection problem. In practice, good representations of data are unknown a priori, and it is thus unclear how to select a good kernel function. In particular, choosing a good kernel often entails some domain knowledge about the underlying learning application. Multiple kernel learning (MKL) methods aim to address this model selection issue by learning a mixture model of a set of user-defined base kernels.

While MKL techniques in conjunction with the kernel SVMs often yield a good accuracy in experiments with standard test (out-of-sample) data, they are susceptible to perturbed inputs commonly referred to as \textit{adversarial examples} \cite{szegedy2013intriguing}. For instance, in imaging applications, it has been observed that a certain hardly perceptible perturbation to the input image, which is found by maximizing the classifer’s prediction error, causes the machine learning model to mis-classify an image that is conveniently recognizable by human.  Nevertheless, such adversarial examples are primarily devised to showcase the vulnerabilities of the deep neural networks to adversarial inputs rather than the kernel SVMs. As a result, methodologies that are developed in the literature to mitigate the impact of such adversarial examples on the prediction accuracy of deep neural networks are inadequate for the kernel SVMs.

To ensure the robustness of the kernel SVM models against adversarial inputs, we propose a novel MKL procedure. In particular, we propose a novel kernel learning algorithm based on a distributionally robust optimization procedure with respect to a distribution ball measured by the KL divergence, and centered at the distribution of training data. We show that such distributional optimization procedure can be characterized as the standard min-max optimization problem that can be solved efficiently via the primal-dual proximal algorithms in conjunction with the Monte-Carlo sample average approximation. We provide theoretical guarantees for the consistency and the asymptotic normlality of the proposed optimization procedure. We also establish novel distributionally robust generalization bounds based on the notions of Rademacher and Gaussian complexities of function classes.

\subsection{Related works}
We situate our results in the connection to the multiple kernel learning, adversarial training, and distributionally robust optimization problem.
\subsubsection{Multiple kernel learning} Multiple kernel learning for classification problems using kernel SVMs has been studied extensively over the past decade. Cortes, \textit{et al}. \cite{cortes2009new,cortes2010two,cortes2012algorithms} studied a kernel learning procedure from the class of mixture of base kernels.  They have also studied the generalization bounds of the proposed methods. The same authors have also studied a two-stage kernel learning in \cite{cortes2010two} based on a notion of alignment. The first stage of this technique consists of learning a kernel that is a convex combination of a set of base kernels. The second stage consists of using the learned kernel with a standard kernel-based learning algorithm such as SVMs to select a prediction hypothesis.  In \cite{lanckriet2004learning}, the authors have proposed a semi-definite programming for the kernel-target alignment problem. An alternative approach to kernel learning is to sample random features from an arbitrary distribution (without tuning its hyper-parameters) and then apply a supervised feature screening method to distill random features with high discriminative power from redundant random features.

The current paper is also closely related to the multiple kernel learning of Yang, \textit{et al.} \cite{yang2012multiple} and Sinha and Duchi \cite{sinha2016learning}. In \cite{yang2012multiple}, a primal-dual proximal algorithm for the multiple kernel learning from training data with noisy class labels is proposed. More specifically, the work of \cite{yang2012multiple} proposes a chance constraint optimization problem to deal with data-sets in which binary class labels are noisy. Our multiple kernel learning framework generalizes \cite{yang2012multiple} by providing robustness against perturbations of both class labels and input features. The work of Sinha and Duchi \cite{sinha2016learning} proposes a distributionally robust optimization with respect to $\chi^{2}$-divergence for the weights of the mixture model in multiple kernel learning. In contrast, in this paper, we consider a distributionally robust optimization with respect to KL divergence of the training data distribution. Therefore, while the work of Sinha and Duchi \cite{sinha2016learning} can be considered as a form of robust feature selection for the Rahimi and Recht random features \cite{rahimi2008random,rahimi2009weighted}, our work can be viewed through the lense of adverserial robustness against the perturbation of the input features and their class labels.

\subsubsection{Distributionally robust optimization} There are different paradigms in the distributional robust optimization literature corresponding to different methodologies that are used to characterize the uncertainty distribution set, such as constraint sets for moments, support, or directional deviations  \cite{chen2007robust,delage2010distributionally,duchi2016statistics}, distributional balls with respect to the $f$-divergences \cite{ben2013robust,bertsimas2018data,lam2015quantifying,namkoong2016stochastic,sinha2016learning}, and  with respect to the Wasserstein distance \cite{esfahani2018data,blanchet2019robust,gao2016distributionally,blanchet2017data,abadeh2015distributionally,sinha2017certifying}.  In this paper, we leverage a distributionally robust optimization with respect to the KL divergence distributional ball. Such distributional optimization procedure is closely related to the notion of the \textit{tilted empirical risk minimization} (TERM) \cite{li2020tilted}. Specifically, TERM is a generalization of the standard empirical risk minimization, where the empirical loss function is parameterized with ``tilt".  In the work of \cite{li2020tilted}, the tilt parameter can take any real values, and its value is determined by a grid search over some set of candidate values.  In our framework, the tilt is a non-negative value that emerges as a dual variable in the primal-dual characterisation of the distributionally robust optimization. Therefore, the value of the tilt in our framework can be determined in a principled manner from an optimization problem.

\subsubsection{Adversarial training}  In the past few years, several empirical defenses have been proposed for training classifiers to be robust against adversarial perturbations \cite{kurakin2016adversarial,madry2017towards,miyato2018virtual,papernot2016distillation,samangouei2018defense,zhang2018efficient}. These robust classifiers are devised to mitigate a particular adversarial model, and as such, they can still be vulnerable against alternative adversarial examples \cite{athalye2018robustness,athalye2018obfuscated,uesato2018adversarial}. To overhaul the dichotomy between defenses and attacks, certified defenses have been popularized, where the objective is to train classifiers whose predictions are provably robust; see, \textit{e.g.}, \cite{bunel2018unified,carlini2017adversarial,cheng2017maximum,croce2019provable,dutta2018output,dvijotham2018training,ehlers2017formal} for a non-exhasutive list.  Another line of defense work focuses on randomized smoothing where the prediction is robust within some region around the input with a user-defined probability \cite{cao2017mitigating,cohen2019certified,lecuyer2019certified,li2018certified}. All the aforementioned methods, however, are primarily concerned with deep neural networks training, and cannot be applied to kernel SVMs.

\subsection{Paper outline} The rest of this paper is organized as follows:

\begin{itemize}[leftmargin=*]
	\item \textit{Empirical Risk Minimization in Reproducing Kernel Hilbert Spaces}: In Section \ref{Section:Preliminaries of Kernel Methods in Classification and Regression}, we review some preliminaries regarding the empirical risk minimization in reproducing kernel Hilbert spaces. We also provide the notion of the kernel-target alignment for optimizing the kernel in support vector machines (SVMs). We then characterize a distributionally robust optimization problem for multiple kernel learning. 
	
	\item \textit{Primal-Dual Method for Distributionally Robust Optimization}: In Section \ref{Section:Primal-Dual Method},  we describe a procedure to transform the distributionally robust optimization into a min-max optimization problem. We then apply a Gumbel perturbation in conjunction with SAA to devise a stochastic primal-dual method.
	
	\item \textit{Consistency, Min-Max Rate, and Robust Generalization Bounds}: In Section \ref{Section:Theoretical_Results}, we provide the theoretical guarantees for the kernel learning algorithm. In particular, we establish the non-asymptotic consistency of the finite sample approximations. 
	
	\item \textit{Empirical Evaluation on Synthetic and Benchmark Data-Sets}: In Section \ref{Section:Empirical Evaluation for Classification of Real-World Datasets}, we evaluate the performance of  
    our proposed kernel learning model on synthetic and benchmark data-sets. 	
    
\end{itemize}

\section{Preliminaries and the Optimization Problem for Kernel Learning
} 
\label{Section:Preliminaries of Kernel Methods in Classification and Regression}

In this section, we review preliminaries of kernel methods in classification and regression problems. 

\subsection{Reproducing kernel Hilbert spaces (RKHS)}

Let $\mathcal{X}$ be a metric space. A \textit{Mercer kernel} on $\mathcal{X}$ is a continuous and symmetric function $K:\mathcal{X}\times \mathcal{X}\rightarrow \real$ such that for any finite set of points $\{\bm{x}_{1},\cdots,\bm{x}_{N}\}\subset \mathcal{X}$, the kernel matrix $(K(\bm{x}_{i},\bm{x}_{j}))_{1\leq i,j\leq N}$ is positive semi-definite.  The \textit{reproducing kernel Hilbert space} (RKHS) $\mathcal{H}_{K}$ associated with the kernel $K$ is  the completion of the linear span of the set of functions $\{K_{\bm{x}}\df K(\bm{x},\cdot), \bm{x}\in \mathcal{X}\}$ with the inner product structure $\langle\cdot,\cdot \rangle_{\mathcal{H}_{K}}$ defined by $\langle K_{\bm{x}_{i}},K_{\bm{x}_{j}} \rangle_{\mathcal{H}_{K}}=K(\bm{x}_{i},\bm{x}_{j})$. That is 
\begin{align}
\left\langle \sum_{i=1}^{N}\alpha_{i}K_{\bm{x}_{i}},\sum_{j=1}^{N}\beta_{j}K_{\bm{x}_{j}} \right\rangle_{\mathcal{H}_{K}}=\sum_{i,j=1}^{N}\alpha_{i}\beta_{j}K(\bm{x}_{i},\bm{x}_{j}).
\end{align}
The \textit{reproducing property} takes the following form
\begin{align}
\langle K_{\bm{x}},f\rangle_{\mathcal{H}_{K}}=f(\bm{x}), \quad \forall \bm{x}\in \mathcal{X}, f\in \mathcal{H}_{K}.
\end{align}

In this paper, we focus on the supervised learning of classifiers from RKHS. In the classical supervised learning settings, we are given $n$ feature vectors and their corresponding univariate class labels  $(\bm{x}_{1},y_{1}),\cdots,(\bm{x}_{n},y_{n})\sim_{\text{i.i.d.}} P^{0}_{\bm{X},Y}$,  $(\bm{x}_{i},y_{i})\in \mathcal{Z}\df \mathcal{X}\times \mathcal{Y}\subset \real^{d}\times \real$. For the binary classification and regression tasks, the target spaces are given by $\mathcal{Y}=\{-1,1\}$ and $\mathcal{Y}=\real$, respectively. Given a loss function $\ell:\mathcal{Y}\times \real\rightarrow \real$, a classifier $f$ is learned from the function class $\mathcal{F}$ by minimizing the empirical risk with a quadratic regularization term,
\begin{align}
\label{Eq:Empirical_Loss_Minimization}
\inf_{f\in \mathcal{F}}\widehat{R}[f]\df \dfrac{1}{n}\sum_{i=1}^{n} \ell(y_{i},f(\bm{x}_{i}))+\dfrac{\lambda}{2}\|f\|_{\mathcal{F}}^{2},
\end{align}
where $\|\cdot\|_{\mathcal{F}}$ is a function norm, and $\lambda>0$ is the parameter of the regularization.  In the adversarial learning models, the out-of-sample (test) data distribution $\bm{z}_{i}=(y_{i},\bm{x}_{i})\sim P\df P_{\bm{X},Y}$ deviate from the in sample data distribution $Q\df P_{\bm{X}^{0},Y^{0}}$, and there is a ``cost" $c:\mathcal{Z}\times \mathcal{Z}\rightarrow \real_{+}\cup \{+\infty\}$ associated with such a perturbation.  In the adversarial learning litearture,  the quadratic norm $c(\bm{z}_{i},\bm{z}_{i}^{0})=\|\bm{z}_{i}-\bm{z}_{i}^{0}\|_{p}^{2}$ with $p=\infty$ or $p=2$ is typically imposed on the actual test samples $\bm{z}_{i}^{0}\sim Q$ and their corresponding perturbations $\bm{z}_{i}\sim P$ \cite{sinha2017certifying}. However, such a cost function is oblivious of the underlying generative distribution of samples. Alternatively, the cost function can penalize the shift of the joint distribution of the test samples $c(\bm{z},\bm{z}^{0})=D_{f}(P||Q)$ for all realizations of the samples and their perturbations $\bm{z},\bm{z}^{0}\in \mathcal{Z}\times \mathcal{Z}$, where $D_{f}(\cdot||\cdot)$ is the $f$-divergence. In the sequel, we focus on the Kullback-Leibler divergence  $D_{\mathrm{KL}}(\cdot||\cdot)$ between two distributions $P$ and $Q$, namely
\begin{align}
	\nonumber
	D_{\mathrm{KL}}(P||Q)\df \int_{\mathcal{X}\times \mathcal{Y}} \log\left(\dfrac{\mathrm{d}P}{\mathrm{d}Q} \right)\mathrm{d}P,
\end{align}
when $P$ is absolutely continuous with respect to $Q$ (denoted by $P\ll Q$), and $+\infty$ otherwise.  Now, consider a Reproducing Kernel Hilbert Space (RKHS) $\mathcal{H}_{K}$ with the kernel function $K:\mathcal{X}\times \mathcal{X}\rightarrow \real$, and suppose $\mathcal{F}=\mathcal{H}_{K}\oplus 1$. Then, using the expansion $f(\bm{x})=\omega_{0}+\sum_{i=1}^{n-1}\omega_{i}K(\bm{x},\bm{x}_{i})$, and optimizing the kernel over the kernel class $\mathcal{K}$ yields the following penalized primal population optimization
	\begin{align}
\label{Eq:setup}
\min_{\bm{\omega}\in \real^{n}}\min_{K\in \mathcal{K}} \expect_{P}\left[\ell\left(y,\omega_{0}+\sum_{i=1}^{n-1}\omega_{i}K(\bm{x},\bm{x}_{i}) \right)+\dfrac{\lambda}{2}\|\bm{\omega}\|_{2}^{2}\right]- \rho D_{\mathrm{KL}}(P||Q),
	\end{align}
where the risk function is caliberated with the cost function $c(\bm{z},\bm{z}^{0})=D_{\mathrm{KL}}(P(\bm{z})||Q(\bm{z}^{0}))$ associated with the perturbations. The risk minimization in Eq. \eqref{Eq:setup} is the Lagrangian form of the following distributionally robust optimization problem
	\begin{align}
	\label{Eq:setup}
	\min_{\bm{\omega}\in \real^{n}}\min_{K\in \mathcal{K}}\inf_{P\in \mathcal{P}} \expect_{P}\left[\ell\left(y,\omega_{0}+\sum_{i=1}^{n-1}\omega_{i}K(\bm{x},\bm{x}_{i}) \right)+\dfrac{\lambda}{2}\|\bm{\omega}\|_{2}^{2}\right],
\end{align}
where $\mathcal{P}\df \{P\in \mathcal{M}(\mathcal{Z}): D_{\mathrm{KL}}(P||Q)\leq r, P\ll Q \}$ is a distributional ball centered at $Q$. In the particular case of the soft margin SVM classifier $\ell(y,z)=[1-yz]_{+}\df \max\{0,1-yz\}$, the primal and dual \textit{empirical} risk optimizations can be written as follows
\begin{subequations}
	\label{Eq:Suggest}
	\begin{align}
	&\text{Primal:}\min_{\bm{\omega}\in \real^{n}}\min_{K\in \mathcal{K}}\inf_{\widehat{P}^{n}\in \mathcal{P}^{n}}\expect_{\widehat{P}^{n}}\left[1-\omega_{0}y+\sum_{i=1}^{n-1}\omega_{i}yK(\bm{x},\bm{x}_{i})\right]_{+}\hspace{-3mm}+\dfrac{\lambda}{2}\|\bm{\omega}\|_{2}^{2},    \\ 
	&\text{Dual:}\max_{\bm{\alpha}\in \real^{n}}\min_{\gamma_{0},\bm{\gamma},\tilde{\bm{\gamma}} } \min_{K\in \mathcal{K}} \inf_{\widehat{P}^{n}\in \mathcal{P}^{n}} F(\bm{\alpha},\bm{y},\bm{\gamma},\tilde{\bm{\gamma}},\gamma_{0}) -\dfrac{1}{2}\mathrm{Tr}(\bm{K}(\bm{\alpha}\odot \bm{y})(\bm{\alpha}\odot \bm{y})^{T}),
	\end{align}
\end{subequations}
where $\odot$ is Hadamard's (element-wise) product,  and $F(\bm{\alpha},\bm{y},\bm{\gamma},\tilde{\bm{\gamma}},\gamma_{0})\df \langle \bm{\alpha},\bm{1} \rangle-\gamma_{0}\langle \bm{\alpha},\bm{y} \rangle-\bm{\gamma}(C\bm{1}-\bm{\alpha})^{T}+\tilde{\bm{\gamma}}\bm{\alpha}^{T} $, where $\gamma_{0}\in \real_{+},\bm{\gamma}\in \real_{+}^{n}$, and $\tilde{\bm{\gamma}}\in \real_{+}^{n}$. In Equation \eqref{Eq:Suggest}, $\widehat{P}^{n}={1\over n}\sum_{i=1}^{n}\delta_{\bm{x}_{i},y_{i}}(\bm{x},y)$ is the empirical measure associated with the observed (training) samples, and $\widehat{\mathcal{P}}^{n}\df \{\widehat{P}^{n}\in \mathcal{M}(\mathcal{Z}): D_{\mathrm{KL}}(\widehat{P}^{n}||\widehat{Q}^{n})\leq r, \widehat{P}^{n}\ll \widehat{Q}^{n} \}$ is the empirical distribution ball centered at the empirical measure $\widehat{Q}^{n}$. 

\subsection{MKL as the importance sampling of the random features}
The form of the dual optimization in Eq. \eqref{Eq:Suggest} suggests that for a fixed dual vector $\bm{\alpha}\in \real^{n}$ and a tuple of Lagrange multipliers $(\gamma_{0},\bm{\gamma},\tilde{\bm{\gamma}})\in \real_{+}\times \real_{+}^{n}\times \real_{+}^{n}$, the optimal kernel can be computed by optimizing the following unbiased statistics known as the \textit{kernel-target alignment}, \textit{i.e.}
\begin{align}
\label{Eq:kernel_target}
&\max_{K\in \mathcal{K}}\min_{\widehat{P}^{n}\in \mathcal{P}^{n}} \expect_{\widehat{P}^{n^{\otimes 2}}}[y\tilde{y}K(\tilde{\bm{x}},\bm{x})]=\dfrac{2}{n(n-1)} \sum_{1\leq i\leq j\leq n}y_{i}y_{j}K(\bm{x}_{i},\bm{x}_{j}).
\end{align}
The empirical optimization problem in Eq. \eqref{Eq:kernel_target} is a unbiased estimator of the following population value
\begin{align}
\label{Eq:Maximization_over_kernel}
\max_{K\in \mathcal{K}} \inf_{P\in \mathcal{P}} \expect_{P^{\otimes 2}}\big[y\tilde{y}K(\bm{x},\tilde{\bm{x}})\big],
\end{align}
Now, consider the class of the mixture kernels 
\begin{align}
\nonumber
&\mathcal{K}_{m}\df \Big\{K:\mathcal{X}\times \mathcal{X}\rightarrow \real:K(\bm{x},\bm{x}')=\sum_{i=1}^{m}\omega_{i}K_{i}(\bm{x},\bm{x}'),\bm{\omega}\in \mathrm{S}^{+}_{m} \Big\},
\end{align}
where $K_{i},i=1,2,\cdots,m$ are the set of known base kernels (\textit{e.g.} Gaussian kernels with different bandwidth parameters), and $\mathrm{S}_{m}^{+}\df \{\bm{\omega}\in \real^{m}:\langle \bm{\omega},\bm{1} \rangle=1,\bm{\omega}\succeq \bm{0}\}$ is the simplex of the probability distribution. To establish the connection between the multiple kernel learning and the importance sampling of the random feature distribution, we require the following two definitions:

\begin{definition}(\textsc{Bochner's Representation of the Shift Invariant Kernels})
\label{Definition:Translation_Invariant_Kernel}	
A kernel $K:\mathcal{X}\rightarrow \mathcal{X}\rightarrow \real$ is \textit{translation invariant} if there exists a symmetric positive definite function, $\psi$ such that $K(\bm{x},\bm{y})=\psi(\bm{x}-\bm{y})$ for all $\bm{x},\bm{y}\in \real^{d}$. Bochner's theorem \cite{bochner2005harmonic} provides a complete characterization of a
positive definite function $\psi:\real^{d}\rightarrow \real$. A continuous function $\psi:\real^{d}\rightarrow \real$ is positive definite if it admits the integral representation
\begin{align}
\psi(\bm{x})=\dfrac{1}{(2\pi)^{d/2}}\int_{\real^{d}}e^{-i\langle \bm{\xi},\bm{x}\rangle}\mu(\mathrm{d}\bm{\xi}), \quad \forall\bm{x}\in \real^{d}, \mu\in \mathcal{B}_{+}(\real^{d}),
\end{align}
where $\mathcal{B}_{+}(\real^{d})$ is the set of all finite  non-negative
Borel measures on $\real^{d}$.
\end{definition}

\begin{definition}(\textsc{Sch\"{o}nberg's representation of the Radial Kernels})
		\label{Definition:Radial_Kernel}	
A translation invariant kernel $K:\mathcal{X}\times \mathcal{X}\rightarrow \real$ is a \textit{radial kernel} if there exists a function $\psi:\real^{d}\rightarrow \real$ such that $K(\bm{x},\bm{y})=\psi(\|\bm{x}-\bm{y}\|_{2})$. The Sch\"{o}nberg's representation theorem \cite{schoenberg1938metric} provides a complete characterization of a
positive definite RBF kernel $\psi:\real\rightarrow \real$. A radial basis function $\psi:\real\rightarrow \real$ is positive definite if it admits the integral representation
\begin{align}
\label{Eq:Sch}
\psi(\|\bm{x}\|_{2})=\int_{\real_{+}}e^{-s\|\bm{x}\|_{2}^{2}}\nu(\mathrm{d}s), \quad \forall\bm{x}\in \real^{d}, \nu\in \mathcal{B}_{+}(\real_{+}).
\end{align}
\end{definition}

We consider a kernel optimization scheme in conjunction with the random features model of Rahimi and Recht \cite{rahimi2008random,rahimi2009weighted} for the translation invariant kernels. Specifically, let $\varphi: \real^{d}\times \real^{D} \rightarrow [-1, 1]$ denotes the explicit feature maps and $\mu\in \mathcal{M}(\real^{D})$ denotes a probability measure from the space of probability measures $\mathcal{M}(\real^{D})$ on $\real^{D}$. The kernel function is characterized via the random feature maps using Bochner's theorem (cf. \cite{rahimi2008random,rahimi2009weighted})
\begin{align}
\label{Eq:Generative_Distribution}
K_{\mu}(\bm{x},\tilde{\bm{x}})&=\int_{\Xi} \varphi(\bm{x};\bm{\xi})\varphi(\tilde{\bm{x}};\bm{\xi})\mu(\mathrm{d}\bm{\xi}), \quad \forall \bm{x},\tilde{\bm{x}}\in \mathcal{X}\subset \real^{d}.
\end{align}
For RBF kernels in Definition \ref{Definition:Radial_Kernel} with the spectral density $\nu$, it can be shown that 
\begin{subequations}
\label{Eq:II}	
\begin{align}
\bm{\xi}&=(\bm{\zeta},b)\in \Xi=\real^{d}\times [-\pi,\pi], \\
 \bm{\xi}&\sim  \mu(\mathrm{d}\bm{\xi})=\pi(\bm{\zeta})\otimes \mathrm{Uniform}[-\pi,\pi],\\
\varphi(\bm{x},\bm{\xi})&=\cos(\langle (\bm{x},1),(\bm{\zeta},b)\rangle),
\end{align}
\end{subequations}
where $\pi(\mathrm{d\bm{\zeta}})\df \int_{0}^{\infty}\pi_{s}(\mathrm{d}\bm{\zeta})\nu(\mathrm{d}s)$, and $\pi_{s}(\mathrm{d}\bm{\zeta})={1\over (4\pi s)^{d\over 2}}\exp\left(-{ \|\bm{\zeta}\|_{2}^{2}\over 4s}\right)\mathrm{d}\bm{\zeta}$. Let $\mu_{1},\cdots,\mu_{m}$ denotes the distribution of random features corresponding to the kernels $K_{1},\cdots,K_{m}$, \textit{i.e.}, $K_{i}=K_{\mu_{i}},i=1,2,\cdots,m$. Then, we can recast the multiple kernel learning in terms of the distribution of random features random features
\begin{align}
	\label{Eq:Jeff}
	\hspace{-4mm}\max_{\bm{\omega}\in \mathrm{S}_{m}^{+}}\inf_{P\in \mathcal{P}}\expect_{P^{\otimes 2}}\left[{y\tilde{y}}\int_{\Xi} \varphi(\bm{x};\bm{\xi})\varphi(\tilde{\bm{x}};\bm{\xi})\mu_{\bm{\omega}}(\mathrm{d}\bm{\xi})\right].
	\end{align}
where the optimization in Eq. \eqref{Eq:Jeff} corresponds to the importance sampling of the random feature distribution $\mu_{\bm{\omega}}(\bm{\xi})\df \sum_{i=1}^{m}\omega_{i}\mu_{i}(\bm{\xi})$.

\subsection{Connections to the maximum mean discrepancy (MMD)}
\label{Eq:Connections to the maximum mean discrepancy (MMD)}
The distributionally robust optimization in Eq. \eqref{Eq:Jeff} can also be interpreted as the problem of finding a kernel that maximizes the distance between the marginal distributions of the features given their class labels in the settings that the marginals are adversarially chosen to minimize their distance. To provide a more precise description, in the following definition we formalize the notion of discrepancy between two distributions:
\begin{definition}\textsc{(Maximum Mean Discrepancy \cite{gretton2012kernel})}
	\label{Def:MMD}
	Let $(\mathcal{Z},d)$ be a metric space, $\mathcal{F}$ be a class of functions $f:\mathcal{Z}\rightarrow \real$, and $P,Q\in \mathcal{M}(\mathcal{Z})$ be two probability measures from the set of all Borel probability measures $\mathcal{M}(\mathcal{Z})$ on $\mathcal{Z}$. The maximum mean discrepancy (MMD) between the distributions $P$ and $Q$ with respect to the function class $\mathcal{F}$ is defined below
	\begin{align}
	\label{Eq:different_choices}
	\mathrm{MMD}_{\mathcal{F}}[P,Q]\df \sup_{f\in \mathcal{F}}\int_{\mathcal{Z}} f(\bm{z})(P-Q)(\mathrm{d}\bm{z}).
	\end{align}
\end{definition}

When the function class $\mathcal{F}$ is restricted to a RKHS norm ball $\mathcal{F}\df \{f\in \mathcal{H}_{K}:\|f\|_{\mathcal{H}_{K}}\leq 1\}$, the kernel MMD can be rewritten as follows
\begin{align}
\nonumber
\mathrm{MMD}_{K}[P,Q]&=\sup_{f:\|f\|_{\mathcal{H}_{K}}\leq 1}\expect_{P}[f]-\expect_{Q}[f] \\ \nonumber
&=\|\mu_{P}-\mu_{Q}\|_{\mathcal{H}_{K}}\\ 
&=\Big(\expect_{P^{\otimes 2}}[K(\bm{z},\bm{z}')]+\expect_{Q^{\otimes 2}}[K(\bm{s},\bm{s}')]-2\expect_{P,Q}[K(\bm{z},\bm{s})]\Big)^{1\over 2},
\end{align}
where $\mu_{P},\mu_{Q}\in \mathcal{H}_{K}$ are the kernel mean embeddings of the distributions $P$ and $Q$, \textit{i.e.},
\begin{subequations}
\begin{align}
\nonumber
\mu_{P}\df \int_{\mathcal{X}}K(\cdot,\bm{x})\mathrm{d}P(\bm{x}),\quad
\mu_{Q}\df \int_{\mathcal{X}}K(\cdot,\bm{x}) \mathrm{d}Q(\bm{x}).
\end{align}
\end{subequations}
Hence, $\expect_{P}[f]=\langle f,\mu_{P}\rangle_{\mathcal{H}_{K}}$, and $\expect_{Q}[f]=\langle f,\mu_{Q}\rangle_{\mathcal{H}_{K}}$, for all  functions $f:\mathcal{X}\rightarrow \real$. Let $\mathcal{Z}=\mathcal{X}\times \mathcal{Y}$ and $P\in \mathcal{M}(\mathcal{Z})$. Furthermore, consider the conditional distributions $P_{+}\df P_{\bm{x}|y=+1}$ and $P_{-}\df P_{\bm{x}|y=1}$. Then, using the Rahimi and Recht random feature model \cite{rahimi2008random,rahimi2009weighted}, $K_{\mu_{\bm{\omega}}}(\bm{x},\tilde{\bm{x}})=\int_{\Xi}\varphi(\bm{x};\bm{\xi})\varphi(\tilde{\bm{x}};\bm{\xi})\mu_{\bm{\omega}}(\mathrm{d}\bm{\xi})$, the following equality holds
\begin{align}
\mathrm{MMD}^{2}_{K_{\bm{\omega}}}[P_{+}, P_{-}]=\expect_{P^{\otimes 2}}\left[\int_{\Xi}y\tilde{y}\varphi(\bm{x};\bm{\xi})\varphi(\tilde{\bm{x}};\bm{\xi})\mu_{\bm{\omega}}(\mathrm{d}\bm{\xi})\right].
\end{align}
Furthermore, for the joint distributions $P,Q\in \mathcal{M}(\mathcal{Z})$ with the marginals $P(y)$ and $Q(y)$, respectively, we obtain that
\begin{align}
D_{\mathrm{KL}}(P(\bm{x},y)|| Q(\bm{x},y))=D_{\mathrm{KL}}(P(y)||Q(y))+\expect_{P(y)}\big[D_{\mathrm{KL}}(P(\bm{x}|y)||Q(\bm{x}|y))\big]. 
\end{align}
Suppose the restriction $P(y)=Q(y)$ applies to the marginals so that  $D_{\mathrm{KL}}(P(y)||Q(y))=0$. Under this restriction, the class labels in the training data-set are the ground truths and only the input features are adverserially perturbed. Furthermore, consider a balanced training data-set $Q(y=+1)=Q(y=-1)={1\over 2}$. Then, the KL divergence admits the following decomposition
\begin{align}
\nonumber
D_{\mathrm{KL}}(P(\bm{x},y)|| Q(\bm{x},y))&=\dfrac{1}{2}D_{\mathrm{KL}}(P(\bm{x}|y=+1)||Q(\bm{x}|y=+1))\\ \nonumber
&\hspace{4mm}+\dfrac{1}{2}D_{\mathrm{KL}}(P(\bm{x}|y=-1)||Q(\bm{x}|y=-1))\\  \label{Eq:Identity}
&=\dfrac{1}{2}D_{\mathrm{KL}}(P_{+}\otimes P_{-}||Q_{+}\otimes Q_{-}).
\end{align}
where $Q_{+}\df Q_{\bm{x}|y=+1}$, $Q_{-}\df Q_{\bm{x}|y=-1}$, and $P_{+}\otimes P_{-}$ and $Q_{+}\otimes Q_{-}$ denotes the tensor product of the distributions. Now, define the following distribution ball on the product space $\mathcal{M}(\mathcal{X})\times \mathcal{M}(\mathcal{X})$ of the marginal distributions:
\begin{align}
\label{Eq:distribution_ball_of_marginals}
\mathcal{P}_{\pm}\df \{(P_{+},P_{-})\in \mathcal{M}(\mathcal{X})\times \mathcal{M}(\mathcal{X}): D_{\mathrm{KL}}(P_{+}\otimes P_{-}||Q_{+}\otimes Q_{-})\leq 2r\}.
\end{align}
Due to  the identity \eqref{Eq:Identity}, $P\in \mathcal{P}$ implies $(P_{+},P_{-})\in \mathcal{P}_{\pm}$ and vice versa. Consequently, in the case of the balanced data-set and under the restriction that the class labels are not adverserially chosen $P(y)=Q(y)$, the kernel-target alignment in Eq. \eqref{Eq:Jeff} can be recast as the following MMD optimization problem
\begin{align}
\label{Eq:Plainly_explained}
\max_{\bm{\omega}\in \mathrm{S}_{+}^{m}}\min_{(P_{+},P_{-})\in \mathcal{P}_{\pm}} \mathrm{MMD}^{2}_{K_{\bm{\omega}}}[P_{+},P_{-}].
\end{align}
In words, Equation \eqref{Eq:Plainly_explained} says that the best mixture of base kernels is the one that increases the distance between the conditional distributions of features of each class label, as measured by the kernel MMD, when such conditional distributions are chosen adverserially to minimize the discrepancy.

\section{Biased and Unbiased Stochastic Primal-Dual Method to Solve the Distributionally Robust Optimization Problem}
\label{Section:Primal-Dual Method}

The distributional optimization problem in Eq. \eqref{Eq:Jeff} is intractable. To derive a solvabale optimization problem, we consider three main approximation steps in the sequel: (\textit{i}) a change of measure techniques due to \cite{hu2013kullback}, (\textit{ii}) a Monte-Carlo sample average approximation with respect to the random feature samples and data-samples, and (\textit{iii})  a Gumbel perturbation technique for the log-sum approximation. Each step is described in the sequel.

\subsection{Change of measure}

Define the Radon-Nikodym derivative $L=\mathrm{d}^{2}P^{\otimes 2}/\mathrm{d}^{2}Q^{\otimes 2}$. Then, we observe that $L(\bm{x},y)\geq 0, \forall \bm{x},y\in \mathcal{X}\times \mathcal{Y}$, and $\expect_{Q^{\otimes 2}}[L]=1$. Moreover, the KL divergence can be rewritten as follows
\begin{align}
\nonumber
D_{\mathrm{KL}}(P^{\otimes 2}||Q^{\otimes 2})&=\int_{\mathcal{X}\times \mathcal{Y}} \log\left(\dfrac{\mathrm{d}P^{\otimes 2}}{\mathrm{d}Q^{\otimes 2}}\right)\dfrac{\mathrm{d}P^{\otimes 2}}{\mathrm{d}Q^{\otimes 2}}\mathrm{d}Q^{\otimes 2}\\ \label{Eq:Change_of_measure_1}
&=\expect_{Q^{\otimes 2}}[L\log(L)].
\end{align}
The objective function in Eq. \eqref{Eq:Jeff} can also be rewritten in terms of $L$, as follows
\begin{align}
\nonumber
\expect_{P^{\otimes 2}}\left[{y\tilde{y}}\int_{\Xi} \varphi(\bm{x};\bm{\xi})\varphi(\tilde{\bm{x}};\bm{\xi})\mu_{\bm{\omega}}(\mathrm{d}\bm{\xi})\right]&=\int_{\mathcal{X}\times \mathcal{Y}}y\tilde{y}\int_{\Xi}\varphi(\bm{x};\bm{\xi})\varphi(\tilde{\bm{x}};\bm{\xi})\mu_{\bm{\omega}}(\mathrm{d}\bm{\xi})\mathrm{d}P^{\otimes 2}\\ \nonumber
&=\int_{\mathcal{X}\times \mathcal{Y}}y\tilde{y}\int_{\Xi}\varphi(\bm{x}_{0};\bm{\xi})\varphi(\bm{x}_{1};\bm{\xi})\mu_{\bm{\omega}}(\mathrm{d}\bm{\xi})\dfrac{\mathrm{d}P^{\otimes 2}}{\mathrm{d}Q^{\otimes 2}}\mathrm{d}Q^{\otimes 2}\\ \label{Eq:Change_of_measure_2}
&=\expect_{Q^{\otimes 2}}\left[L{y\tilde{y}}\int_{\Xi} \varphi(\bm{x};\bm{\xi})\varphi(\tilde{\bm{x}};\bm{\xi})\mu_{\bm{\omega}}(\mathrm{d}\bm{\xi})\right].
\end{align}
Using Eqs. \eqref{Eq:Change_of_measure_1}-\eqref{Eq:Change_of_measure_2}, the distributional  optimization in Eq. \eqref{Eq:Jeff} can be cast as a convex \textit{functional} optimization
\begin{subequations}
	\label{Eq:equivalent_optimization}
	\begin{align}
	\nonumber
	&\max_{\bm{\omega}\in \mathrm{S}^{+}_{m}}\min_{L\in \mathbb{L}(Q)}\expect_{Q^{\otimes 2}}\left[L{y\tilde{y}}\int_{\Xi} \varphi(\bm{x};\bm{\xi})\varphi(\tilde{\bm{x}};\bm{\xi})\mu_{\bm{\omega}}(\mathrm{d}\bm{\xi})\right]\\ \label{Eq:Constraint_1}
	&\hspace{16mm}\text{s.t.}:\expect_{Q^{\otimes 2}}[L\log L]\leq r/2,\\
	&\hspace{23mm} \expect_{Q^{\otimes 2}}[L]=1, L\geq 0,
	\end{align} 
\end{subequations}
where $\mathbb{L}(Q)$ is the set of such functionals, and in writing the constraint in Eq. \eqref{Eq:Constraint_1}, we used the \textit{tensorization} property of the KL divergence 
\begin{align}
D_{\mathrm{KL}}(P^{\otimes 2}||Q^{\otimes 2})=2D_{\mathrm{KL}}(P||Q).
\end{align}
In the sequel, let $\widehat{\mathbb{L}}(Q)=\{L\in \mathbb{L}(Q):\expect_{Q^{\otimes 2}}[L\log L]\leq r/2, \expect_{Q^{\otimes 2}}[L]=1,L\geq 0 \}$. The following lemma provides a method to solve the functional optimization problem in Eq. \eqref{Eq:Constraint_1} via the standard primal-dual optimization methods:
\begin{lemma}\textsc{(Saddle Point Characterization of Functional Optimization)}
	\label{Lemma:Saddle Point Characterization of Functional Optimization}
Consider the function $H(\lambda;\bm{\omega})$ defined below
\begin{align}
\label{Eq:Lambda}
&H(\lambda;\bm{\omega})\df \lambda \log \expect_{Q^{\otimes 2}}\left[\exp{\left(-{{y\tilde{y}\over \lambda}\expect_{\mu_{\bm{\omega}}}\left[\varphi(\bm{x};\bm{\xi})\varphi(\tilde{\bm{x}};\bm{\xi})\right]}\right)} \right]+{\lambda r\over 2}.
\end{align}		
Then, the inner functional optimization problem in Eq. \eqref{Eq:equivalent_optimization} can be recast as a minimization of the function $H(\lambda,\bm{\omega})$, namely
\begin{align}
		\label{Eq:Elon_Musk}
\min_{L\in \mathbb{L}(Q)}\expect_{Q^{\otimes 2}}\left[L{y\tilde{y}}\int_{\Xi} \varphi(\bm{x};\bm{\xi})\varphi(\tilde{\bm{x}};\bm{\xi})\mu_{\bm{\omega}}(\mathrm{d}\bm{\xi})\right]=-\min_{\lambda\in \Lambda}H(\lambda;\bm{\omega}), \quad \forall \bm{\omega}\in \mathrm{S}_{m}^{+}.
		\end{align}
Consequently, the inner distributional optimization in Eq. \eqref{Eq:Jeff} can be recast as a minimization of the function $H(\lambda,\bm{\omega})$, namely
\begin{align}
\label{Eq:Elon_Musk_1}
\inf_{P\in \mathcal{P}} \expect_{P^{\otimes 2}}\left[y\tilde{y}\int_{\Xi}\varphi(\bm{x};\bm{\xi})\varphi(\tilde{\bm{x}};\bm{\xi})\mu_{\bm{\omega}}(\mathrm{d}\bm{\xi}) \right]=-\min_{\lambda\in \Lambda}H(\lambda;\bm{\omega}),\quad \forall \bm{\omega}\in \mathrm{S}_{m}^{+}.
\end{align}		
\end{lemma}
The proof is provided in Section \ref{Appendix:Proof_of_Lemma_Saddle_Point_Characterization} of the Supplementary.

\subsection{Biased stochastic primal-dual (SPD) method}

To solve the optimization problem in Eq. \eqref{Eq:Elon_Musk}, we consider Monte-Carlo  sample average approximation (SAA) of the expectations with respect to the data distribution $Q$, and the random features distribution $\mu$. 

\begin{itemize}[leftmargin=*]
\item \textit{Monte-Carlo SAA of random features:} Consider the samples $\bm{\xi}_{i}^{1},\cdots,\bm{\xi}_{i}^{N}\sim_{\text{i.i.d.}} \mu_{i}$ for each distribution associated with the base kernel $K_{i}$, $i=1,2,\cdots,m$. Then, we define the empirical measure 
\begin{align}
\widehat{\mu}_{\bm{\omega}}^{N}(\bm{\xi})=\dfrac{1}{N}\sum_{k=1}^{N}\sum_{i=1}^{m}\omega_{i}\delta_{\bm{\xi}_{i}^{k}}(\bm{\xi}),
\end{align}
and consider the following SAA
\begin{align}
	\label{Eq:SAA}
	\int_{\Xi} \varphi(\bm{x}_{i};\bm{\xi})\varphi(\bm{x}_{j};\bm{\xi})\widehat{\mu}^{N}_{\bm{\omega}}(\mathrm{d}\bm{\xi})=\dfrac{1}{N}\big\langle \bm{\varphi}_{\bm{\omega}}(\bm{x}_{i}),\bm{\varphi}_{\bm{\omega}}(\bm{x}_{j})  \big\rangle,
	\end{align}
where $\bm{\varphi}_{\bm{\omega}}(\bm{x})\df \mathrm{vec}\big(\mathrm{diag}^{1\over 2}(\bm{\omega})\bm{\Phi}(\bm{x})\big)\in \real^{mN}$ is defined via the concatenation of random features associated with each base kernel. Here, $\bm{\omega}\df (\omega_{1},\cdots,\omega_{m})$, $\bm{\Phi}(\bm{x})=[\varphi(\bm{x};\bm{\xi}_{i}^{k})]_{i\in [m],k\in [N]}\in \real^{m\times N}$, and $\mathrm{vec}(\cdot)$ is the vectorization of the matrix by stacking its rows.

\begin{algorithm}[t!]\scriptsize{
		\caption{\small{``Biased" SPD for the Distributionally Robust Multiple Kernel Learning}}
		\label{Algorithm:1}
		\begin{algorithmic}
			\State {\bfseries Inputs:} {The learning rates $\beta,\eta\in \real_{+}$ , the small truncation $\epsilon>0$, the threshold $\tau\geq 0$, the number of random features $N\in \integer$, and the size of the mini-batch $B\in \integer$, with $B\geq 2$.}
			\State {\bfseries Initialize:} {The primal $\bm{\omega}_{0}\in \mathrm{S}_{+}^{m}$, and dual $\lambda\in \real_{+}$.}
			\While{$\max\{\|\mathsf{G}_{\bm{\omega}}^{\pi}(\lambda;\bm{\omega})\|_{2}$,$|\mathsf{G}_{\bm{\omega}}^{\pi}(\lambda;\bm{\omega})|\}\geq \tau$}
			\State{Sample the mini-batch $\{(\bm{x}_{\pi(i)},y_{\pi(i)})\}_{i=1}^{B}$ uniformly and independently from the training set.}
			\State{Compute the biased estimators of the gradients $\mathsf{G}^{\pi}_{\bm{\omega}}(\lambda;\bm{\omega})\df  \nabla_{\bm{\omega}}\widehat{H}^{\pi}_{B,N}(\lambda;\bm{\omega})$ and $\mathsf{G}_{\lambda}^{\pi}(\lambda;\bm{\omega})\df \nabla_{\lambda}\widehat{H}^{\pi}_{B,N}(\lambda;\bm{\omega})$.}
			\State{Update the primal variables via the exponentiation}
			\begin{subequations}
				\label{Eq:Primal_variables}
				\begin{align}
				Z_{\beta}&\leftarrow \sum_{i=1}^{m}\omega_{i}\exp\Big(\beta\big[\mathsf{G}_{\bm{\omega}}^{\pi}(\lambda;\bm{\omega})\big]_{i}\Big)\\
				\omega_{i}&\leftarrow\dfrac{\omega_{i}\exp\Big(\beta\big[\mathsf{G}_{\bm{\omega}}^{\pi}(\lambda;\bm{\omega})\big]_{i}\Big) }{Z_{\beta}}, \quad \forall i\in [n].
				\end{align}
				\State{Update the truncated dual variable}	
				\begin{equation}
				\hspace{-12mm}	\lambda \leftarrow \max\Big\{\lambda-\eta \mathsf{G}_{\lambda}^{\pi}(\lambda;\bm{\omega}),\epsilon\Big\}.
				\end{equation}
			\end{subequations}
			\EndWhile
		\end{algorithmic}
	}
\end{algorithm}\normalsize

\item \textit{Monte-Carlo SAA of data samples:}
The second sample average approximation is due to sampling the training data $(\bm{x}_{i},y_{i})_{1\leq i\leq n}\sim_{\text{i.i.d.}}Q=P_{\bm{X},Y}$. Then, we define the empirical measure associated with the observed training samples as follows
\begin{align}
{\widehat{Q}}^{n}(\bm{x},y)=\dfrac{1}{n}\sum_{i=1}^{n}\delta_{(\bm{x}_{i},y_{i})}(\bm{x},y).
\end{align}
Using the empirical measure defined above, the SAA with respect to the training data is given by 
\begin{align}
\nonumber
&\expect_{{\widehat{Q}}^{n,\otimes 2}}\left[\exp\left(-\dfrac{y\tilde{y}}{\lambda}	\expect_{\widehat{\mu}^{N}_{\bm{\omega}}}[\varphi(\bm{x};\bm{\xi})\varphi(\tilde{\bm{x}};\bm{\xi})]\right)\right]\\ \label{Eq:SAA_1}
&\hspace{10mm}=\dfrac{2}{n(n-1)}\sum_{1\leq i<j\leq n}\exp\left(-\dfrac{1}{N\lambda}\langle y_{i}\bm{\varphi}_{\bm{\omega}}(\bm{x}_{i}),y_{j}\bm{\varphi}_{\bm{\omega}}(\bm{x}_{j})  \rangle  \right),
\end{align}
where in writing the product measure ${\widehat{Q}}^{n,\otimes 2}$, we ignored the diagonal elements, and as a result Equation \eqref{Eq:SAA_1} is an unbiased estimator of the population value.
\end{itemize}

Now, combining Eqs. \eqref{Eq:SAA}-\eqref{Eq:SAA_1} and  the sample average approximations from  the following unbiased estimator of the population optimization problem 
\begin{align}
\label{Eq:log_sum_exp}
\widehat{H}_{n,N}(\lambda;\bm{\omega})={1\over 2}\lambda r+\lambda\log\dfrac{2}{n(n-1)}\sum_{1\leq i<j\leq n}\exp{\left(-{1\over N\lambda}\langle y_{i}\bm{\varphi}_{\bm{\omega}}(\bm{x}_{i}),y_{j}\bm{\varphi}_{\bm{\omega}}(\bm{x}_{j})  \rangle\right)}.
\end{align}
Let us make a few remarks regarding the form of the objective of Eq. \eqref{Eq:Elon_Musk_1} in Lemma \ref{Lemma:Saddle Point Characterization of Functional Optimization}. To solve the min-max optimization, we consider a mini-batch stochastic primal-dual method (SPD). In particular, consider a batch size of $B\in \integer$ such that $B\geq 2$. Let $\pi:\{1,2,\cdots,B\}\rightarrow \{1,2,\cdots,n\}$ denotes the stochastic sampling function. Define the mini-batch objective function
\begin{align}
\widehat{H}^{\pi}_{B,N}(\lambda;\bm{\omega})\df \dfrac{1}{2}\lambda r+\lambda \log\dfrac{2}{B(B-1)} \sum_{1\leq i<j\leq B}\exp{\left(-{1\over N\lambda}\Big\langle y_{\pi(i)}\bm{\varphi}_{\bm{\omega}}(\bm{x}_{\pi(i)}),y_{\pi(j)}\bm{\varphi}_{\bm{\omega}}(\bm{x}_{\pi(j)}) \Big\rangle\right)}.
\end{align}
Then, we define the stochastic gradients as follows
for all $\ell=1,2,\cdots,m$. In Algorithm \ref{Algorithm:1}, we describe a stochastic primal-dual method to solve the minimax optimization of the objective function in Eq. \eqref{Eq:log_sum_exp}. The proposed stochastic primal-dual algorithm is \textit{biased}. In particular, the expected value of the gradient estimators $\expect_{\pi}\big[\mathsf{G}_{\bm{\omega}}^{\pi}(\lambda;\bm{\omega})\big]$ and $\expect_{\pi}\big[\mathsf{G}_{\lambda}^{\pi}(\lambda;\bm{\omega})\big]$ deviate from their deterministic counterparts $\nabla_{\bm{\omega}}\widehat{H}_{n,N}(\lambda,\bm{\omega})$ and $\nabla_{\lambda}\widehat{H}_{n,N}(\lambda,\bm{\omega})$, respectively. Nevertheless, for a large batch size $B$, the biased estimators are good approximations of the deterministic gradients. The exponentiation step in Equation \eqref{Eq:Primal_variables} of Algorithm \ref{Algorithm:1} is equivalent to the following proximal optimization step
\begin{align}
\bm{\omega}\leftarrow \mathrm{Prox}^{\beta}_{\mathsf{G}^{\pi}_{\bm{\omega}}(\lambda;\cdot)}(\bm{\omega}),
\end{align}
where the proximal operator $\mathrm{Prox}_{f}:\real^{d}\rightarrow \mathcal{X}$ for a lower semi-continuous function $f:\mathcal{X}\rightarrow \real$ with the gradient estimator $\expect[\mathsf{G}_{f}(\bm{x})]=\nabla f(\bm{x})$ is defined as follows
\begin{align}
\mathrm{Prox}^{\beta}_{\mathsf{G}_{f}}(\bm{x})\df \arg\min_{\bm{y}\in \mathcal{X}}\left\{ \langle \mathsf{G}_{f}(\bm{x}),\bm{y}-\bm{x} \rangle +{1\over \beta}{D_{\phi}(\bm{y}||\bm{x})}\right\},
\end{align}
 for all $\bm{x}\in \mathcal{X}$. In the previous display, $D_{\phi}(\bm{y}||\bm{x})\df \phi(\bm{y})-\phi(\bm{x})-\langle \bm{y}-\bm{x},\nabla \phi(\bm{x}) \rangle$ is the Bregman divergence corresponding to the Legendre function $\phi$. The proximal optimization step in Equations \eqref{Eq:Primal_variables} is adapted to the geometry of the simplex $\mathrm{S}_{m}^{+}$, using the Bregman divergence corresponding to the negative entropy function $\phi(\bm{x})=\sum_{i=1}^{d}x_{i}\log(x_{i})$. In particular, it is a projection-free step for the convex optimization over the simplex. See the excellent monograph of Parikh and Boyd \cite{parikh2014proximal} for a treatise on the proximal algorithms.

\subsection{Debiasing via the Gumbel max perturbation scheme}

Due to the logarithmic term,  constructing unbiased estimators with
proper variance for the deterministic gradients $\nabla_{\bm{\omega}} H(\bm{\omega},\lambda)$ and $\nabla_{\lambda} H(\bm{\omega},\lambda)$ based on the uniform sampling of the training data is challenging. To debiase the stochastic optimization algorithm, in the sequel we leverage the Gumble max perturbation technique to estimate the log-partition function. Specifically, let $\{\zeta(z)\}_{z\in \mathcal{Z}}$ denotes a collection of the independent random variables $\zeta(z)$ indexed by $z\in \mathcal{Z}$, each
following the Gumbel distribution whose cumulative
distribution function  is $F(\zeta) = \exp(-\exp(-(\zeta + \gamma)))$,
where $\gamma\approx 0.577$ is the Euler-Mascheroni constant. Then the random variable $\max_{z\in \mathcal{Z}}\{\phi(z)+\zeta(z)\}$ is distributed according to
the Gumbel distribution and its expected value is the
logarithm of the partition function (see \cite{hazan2012partition})
\begin{align}
\label{Eq:Gumbel}
\log \sum_{z\in \mathcal{Z}}\exp(\phi(z))=\expect_{\nu}\Big[\max_{z\in \mathcal{Z}}\{\phi(z)+\zeta(z)\}\Big],	
\end{align}	
where $\nu=\mathrm{d}F/\mathrm{d}\zeta$ is the Lebesgue density.
The Gumbel distribution can be sampled efficiently using the inverse transform sampling technique. Specifically, let $U\sim \textsc{Uniform}[0,1]$, and $\zeta=-\gamma-\ln(-\ln(U))$. Then, $\zeta\sim \textsc{Gumbel}(-\gamma,1)$ with the desired cumulative
distribution function $F(\zeta)$.  We employ the identity in Eq. \eqref{Eq:Gumbel} to replace the log-sum term in Eq. \eqref{Eq:log_sum_exp} with the Gumbel max formulation tain
\begin{align}
	\widehat{H}_{n,N}(\lambda;\bm{\omega})=&{1\over 2}\lambda r+\lambda\log\left({2\over n(n-1)}\right)\\ \nonumber
	&+\lambda\expect_{\nu}\left[ \max_{1\leq i<j\leq n}\left\{-{1\over N\lambda}\langle y_{i}\bm{\varphi}_{\bm{\omega}}(\bm{x}_{i}),y_{j}\bm{\varphi}_{\bm{\omega}}(\bm{x}_{j})  \rangle+\zeta_{ij}\right\}\right],
	\end{align}
where $\zeta_{ij}\sim \nu$ are i.i.d. for all $1\leq i<j\leq n$. Using SAA for the expectation with respect to the Gumbel distribution yields
\begin{align}
\nonumber
	\widehat{H}^{D}_{n,N}(\lambda;\bm{\omega})=&{1\over 2}\lambda r+\lambda\log\left({2\over n(n-1)}\right)\\ \label{Eq:Evaluate}
	&+\dfrac{1}{D}\sum_{s=1}^{D} \max_{1\leq i<j\leq n}\left\{-\dfrac{1}{N}\langle y_{i}\bm{\varphi}_{\bm{\omega}}(\bm{x}_{i}),y_{j}\bm{\varphi}_{\bm{\omega}}(\bm{x}_{j}) \rangle+\lambda\zeta^{s}_{ij}\right\}.
	\end{align}

We now apply a penalty function to evaluate the maximum in Eq. \eqref{Eq:Evaluate}. In particular, the maximum of the real-valued functions $(g_{i}(x))_{1\leq i\leq n},g_{i}:\real\rightarrow \real$ can be computed as $\max_{1\leq i\leq n}g_{i}(x)=\min_{\alpha\in \real} \alpha+\rho \sum_{i=1}^{n}[g_{i}(x)-\alpha]_{+}^{2},$
provided that the penalty factor $\rho$ is sufficiently large.\footnote{More precisely, suppose $\rho^{\ast}_{1},\cdots,\rho^{\ast}_{n}$ are the optimal Lagrange multipliers of the constrained optimization problem $\min_{\alpha\in \real} \alpha$ subject to $g_{i}(x)\leq \alpha, i=1,2,\cdots,n$. Then, letting $\rho=\max_{1\leq i\leq n}\rho^{\ast}_{i}$ guarantees the equality.
} Therefore,

\begin{align}
&\widehat{H}^{D}_{n,N}(\lambda;\bm{\omega})=\min_{\bm{\alpha}\in \real^{D}}\widehat{G}_{n,N}^{D}(\lambda,\bm{\alpha};\bm{\omega}),
\end{align}
where $\widehat{G}_{n,N}^{D}(\lambda,\bm{\alpha};\bm{\omega})$ is defined below
\begin{align}
	\nonumber
	\widehat{G}_{n,N}^{D}(\lambda,\bm{\alpha};\bm{\omega})&\df {1\over 2}\lambda r+\lambda\log\left({2\over n(n-1)}\right)+\dfrac{1}{D}\sum_{s=1}^{D}\alpha_{s}\\
	&\hspace{4mm}+\dfrac{\rho}{D}\sum_{s=1}^{D}\sum_{1\leq i<j\leq n}\left[\lambda\zeta_{ij}^{s}-\dfrac{y_{i}y_{j}}{N}\langle \bm{\varphi}_{\bm{\omega}}(\bm{x}_{i}),\bm{\varphi}_{\bm{\omega}}(\bm{x}_{j}) \rangle -\alpha_{s}\right]_{+}^{2}.
	\end{align}
We now consider the following saddle point optimization problem
\begin{align}
\label{Eq:Swap}
 \min_{\bm{\omega}\in \mathrm{S}^{+}_{m}}\min _{(\lambda,\bm{\alpha})\in \real_{+}\times  \real^{D}}\widehat{G}_{n,N}^{D}(\lambda,\bm{\alpha};\bm{\omega}).
\end{align}
We devise a stochastic primal-dual proximal method to solve the min-max optimization in Eq. \eqref{Eq:Swap}. In particular, we consider the following \textit{unbiased} estimators of the deterministic gradients with the batch size $B\geq 2$:
\small{\begin{subequations} 
	\label{Eq:Gradient_estimators}		
\begin{align}
\Big[\mathsf{Q}^{\pi}_{\bm{\omega}}(\lambda,\bm{\alpha};\bm{\omega})\Big]_{\ell}&=\sum_{s=1}^{D}\sum_{1\leq i<j\leq B}\widehat{\rho}\delta_{\ell}(\pi(i),\pi(j))\left[\lambda\zeta_{\pi(i)\pi(j)}^{s}+\langle \bm{\omega},\bm{\delta}(\pi(i),\pi(j)) \rangle -\alpha_{s}\right]_{+}, \\ 
\Big[\mathsf{Q}^{\pi}_{\bm{\alpha}}(\lambda,\bm{\alpha};\bm{\omega})\Big]_{s}&=\dfrac{1}{D}-\sum_{1\leq i<j\leq B}\widehat{\rho}\left[\lambda\zeta_{\pi(i)\pi(j)} ^{s}+\langle \bm{\omega},\bm{\delta}(\pi(i),\pi(j)) \rangle-\alpha_{s} \right]_{+},\\ \nonumber
\mathsf{Q}^{\pi}_{\lambda}(\lambda,\bm{\alpha};\bm{\omega})&=\dfrac{1}{2}r+\log\left(\dfrac{2}{n(n-1)}\right)\\
&\hspace{4mm}+\sum_{s=1}^{D}\sum_{1\leq i<j\leq B}\widehat{\rho}\Big[\lambda \zeta^{s}_{\pi(i)\pi(j)}+\langle \bm{\omega},\bm{\delta}(\pi(i),\pi(j)) \rangle -\alpha_{s}\Big]_{+},
\end{align}
\end{subequations}}\normalsize
for all  $\ell=1,2,\cdots,m$ and $s=1,2,\cdots,D$. In Eqs. \eqref{Eq:Gradient_estimators}, $\widehat{\rho}\df {2n(n-1)\rho\over DB(B-1)}$ is the normalized penalty factor, and the vector $\bm{\delta}(\pi(i),\pi(j))\df (\delta_{\ell}(\pi(i),\pi(j)))_{1\leq \ell\leq  m}$ has the following elements
\begin{align}
\delta_{\ell}(\pi(i),\pi(j))\df \dfrac{1}{N}\sum_{k=1}^{N}y_{\pi(i)}y_{\pi(j)}\varphi(\bm{x}_{\pi(i)};\bm{\xi}^{k}_{\ell})\varphi(\bm{x}_{\pi(j)};\bm{\xi}^{k}_{\ell}),
\end{align}
and the features $(y,\bm{x})$ and $(\tilde{y},\tilde{\bm{x}})$ are sampled uniformly from the training data. Using the unbiased estimators of the sub-gradients, we provide a stochastic primal-dual (SPD) optimization procedure in Algorithm \ref{Algorithm:2}.

\begin{algorithm}[t!]\scriptsize{
		\caption{\small{``Unbiased" SPD for the Distributionally Robust Multiple Kernel Learning}}
		\label{Algorithm:2}
		\begin{algorithmic}
			\State {\bfseries Inputs:} {The learning rates $\beta,\eta\in \real_{+}$ , the small number $\epsilon>0$, the threshold $\tau\geq 0$, the number of random features $N\in \integer$, and the size of the mini-batch $B\in \integer$, with $B\geq 2$.}
			\State {\bfseries Initialize:} {The primal $\bm{\omega}_{0}\in \mathrm{S}_{+}^{m}$, and dual $\lambda\in \real_{+}$.}
			\While{$\max\{\|\mathsf{Q}_{\bm{\omega}}^{\pi}(\lambda,\bm{\alpha};\bm{\omega})\|_{2}$,$\|\mathsf{Q}_{\bm{\alpha}}^{\pi}(\lambda,\bm{\alpha};\bm{\omega})\|_{2}$,$|\mathsf{Q}_{\lambda}^{\pi}(\lambda,\bm{\alpha};\bm{\omega})|\}\geq \tau$}
			\State{Sample the mini-batch $\{(\bm{x}_{\pi(i)},y_{\pi(i)})\}_{i=1}^{B}$ uniformly and independently from the training set.}
			\State{Compute the unbiased estimators of the gradients $\mathsf{Q}_{\bm{\omega}}^{\pi}(\lambda,\bm{\alpha};\bm{\omega})$, $\mathsf{Q}_{\bm{\alpha}}^{\pi}(\lambda,\bm{\alpha};\bm{\omega})$ and $\mathsf{Q}_{\lambda}^{\pi}(\lambda,\bm{\alpha};\bm{\omega})$ from Eqs. \eqref{Eq:Gradient_estimators}.}
			\State{Update the primal variables via the exponentiation}
			\begin{subequations}
				\begin{align}
				Z_{\beta}&\leftarrow \sum_{i=1}^{m}\omega_{i}\exp\Big(\beta\big[\mathsf{Q}_{\bm{\omega}}^{\pi}(\lambda,\bm{\alpha};\bm{\omega})\big]_{i}\Big),\\
				\omega_{i}&\leftarrow\dfrac{\omega_{i}\exp\Big(\beta\big[\mathsf{Q}_{\bm{\omega}}^{\pi}(\lambda,\bm{\alpha};\bm{\omega})\big]_{i}\Big) }{Z_{\beta}}, \quad \forall i\in [n].
                \end{align}
            \State{Update the dual variables}	
            \begin{equation}
  \hspace{-12mm}          \bm{\alpha}\leftarrow \bm{\alpha}-\eta  \mathsf{Q}_{\bm{\alpha}}^{\pi}(\lambda,\bm{\alpha};\bm{\omega}),
            \end{equation}
            \begin{equation}
            \lambda \leftarrow \max\Big\{\lambda-\eta \mathsf{Q}_{\lambda}^{\pi}(\lambda,\bm{\alpha};\bm{\omega}),\epsilon\Big\}.
            \end{equation}
             \end{subequations}
			\EndWhile
		\end{algorithmic}
	}
\end{algorithm}\normalsize

\section{Theoretical Results}
\label{Section:Theoretical_Results}

In this section, we state our main theoretical results regarding the performance guarantees of the proposed kernel learning procedure. The proofs of theoretical results are presented in Appendix. Before we delve into the theoretical results, we state the main assumptions underlying our results. 
\begin{assumption}
	\label{Assumption:1}
	The random feature maps are bounded, \textit{i.e.}, $	\sup_{\bm{x}\in \mathcal{X}}|\varphi_{i}(\bm{x},\bm{\xi})|\leq L$, for all $i=1,2,\cdots,m$.
\end{assumption}

\begin{assumption}
	\label{Assumption:2}
	The base kernels are shift invariant, \textit{i.e.}, $	K_{i}(\bm{x},\bm{y})=K_{i}(\bm{x}-\bm{y})$ for all $i=1,2,\cdots,m$.
\end{assumption}

\begin{assumption}
	\label{Assumption:3}
	The feature space $\mathcal{X}\subset \real^{d}$ has a finite diameter, \textit{i.e.}, $\mathrm{diam}(\mathcal{X})\df\sup_{\bm{x},\bm{y}\in \mathcal{X}} \|\bm{x}-\bm{y}\|_{2}<\infty.$
\end{assumption}

%
%
%
%
%
%

\subsection{Asymptotic normality and stochastic equicontinuity}

In this part, we establish a few properties of the underlying $U$-statistics in this paper:

\begin{lemma}\textsc{(Asymptotic Normality)} 
	\label{Lemma:Asymptotic_Normality}
Consider the i.i.d. random vectors $\bm{z}_{i}\df (y_{i},\bm{x}_{i})\in \mathcal{Z}=\mathcal{X}\times \mathcal{Y}$, where $\bm{z}_{1:n}\df (\bm{z}_{1},\cdots,\bm{z}_{n}), (\bm{z}_{i})_{1\leq i\leq n}\sim_{\text{i.i.d.}} P\in \mathcal{M}(\mathcal{Z})$, and define the following $U$-statistics
\begin{align}
\label{Eq:Regarding}
U_{\lambda}(\psi_{N},\bm{z}_{1:n})\df \dfrac{2}{n(n-1)}\sum_{1\leq i<j\leq n}\exp\left(\dfrac{1}{\lambda}\psi_{N}(\bm{z}_{i},\bm{z}_{j})\right),
\end{align}
where $\psi_{N}(\bm{z}_{i},\bm{z}_{j})\df {1\over N}\langle y_{i}\bm{\varphi}_{\bm{\omega}}(\bm{x}_{i}), y_{j}\bm{\varphi}_{\bm{\omega}}(\bm{x}_{j}) \rangle$. Let $\Psi_{N}(\bm{z})\df \expect_{\tilde{\bm{z}}\sim P}\left[\exp\left(\dfrac{1}{\lambda}\psi_{N}(\bm{z},\tilde{\bm{z}})\right)\right]$. Then,
\begin{align}
\label{Eq:Asymptotic_normlaity}
\dfrac{\sqrt{n}\left(\lambda \log  U_{\lambda}(\psi_{N},\bm{z}_{1:n})-\lambda\log\expect_{P^{\otimes 2}}[U_{\lambda}]\right)}{2\mathrm{Var}^{1\over 2}[\Psi_{N}]}\stackrel{d}{\leadsto}\mathsf{N}\left(0,\dfrac{\lambda^{2}}{\expect_{P^{\otimes 2}}[U_{\lambda}]}\right),
\end{align}
where $\stackrel{d}{\leadsto}$ denotes the convergence in the distribution.
\end{lemma}
\begin{proof}
	The proof follows from \cite[Theorem 2]{fan2018dnn} and subsequently applying the standard $\delta$-method to $\lambda\log U_{\lambda}$. For completeness, the proof is provided in Appendix \ref{Appendix:Asymptotic_Normality}. 
\end{proof}

The next result is concerned with the stochastic equicontinuity of the $U$-statistics in Eq. \eqref{Eq:Regarding} which is formally defined below:

\begin{definition}(\textsc{Stochastic Equicontinuity})  	
Let $\{Q_{n}(\theta):\theta\in \Theta,n\in \integer\}$ denotes a family of real-valued functions indexed by $n$, where $(\Theta,d)$ is any normed metric space. Then $\{Q_{n}\}_{n\in \integer}$ is stochastically equicontinuous if for every $\epsilon>0$, and $\eta>0$, there is a $\delta>0$ such that
\begin{align}
\lim\sup_{n\rightarrow\infty} \prob\left(\sup_{\theta\in \Theta}\sup_{\widehat{\theta}\in \ball_{\delta}(\theta)}\left|Q_{n}(\widehat{\theta})-Q_{n}(\theta)\right|>\epsilon \right)<\eta,
\end{align}
where $\ball_{\delta}(\theta)\df \{\widehat{\theta}\in \Theta: d(\widehat{\theta},\theta)\leq \delta \}$ is the ball of the radius $\delta$.
\end{definition}
The stochastic equicontinuity is a useful property as it can yield weak convergence. Moreover, the stochastic equicontinuity together with the pointwise convergence implies uniform convergence in probability to equicontinuous functions on a compact set. The following result establishes the stochastic equicontinuity of the $U$-statistic  in Eq. \eqref{Eq:Regarding} of Theorem \ref{Lemma:Asymptotic_Normality}:

\begin{theorem}\textsc{(Stochastic Equicontinuity of the Kernel $U$-statistic)} 
\label{Theorem:Stochastic_Equicontinuity}	
Suppose Assumptions \ref{Assumption:1}-\ref{Assumption:3} are satisfied. Consider the $U$-statistic $U_{\lambda}(\psi_{N},\bm{z}_{1:n})$ defined in Eq. \eqref{Eq:Regarding} of Theorem \ref{Lemma:Asymptotic_Normality}. Furthermore, consider the composite function class
\begin{align}
\label{Eq:GmN}
\mathcal{G}_{m,N}\df \left\{\psi_{N}:\mathcal{Z}\times \mathcal{Z}\rightarrow \real:(\bm{z},\tilde{\bm{z}})\mapsto \psi_{N}(\bm{z},\tilde{\bm{z}})=\exp\Big({1\over N}\langle \bm{\varphi}_{\bm{\omega}}(\bm{z}), \bm{\varphi}_{\bm{\omega}}(\tilde{\bm{z}}) \rangle\Big), \bm{\omega}\in \mathrm{S}_{+}^{m},\varphi\in \mathcal{F} \right\},
\end{align}	
where $\bm{z}=(y,\bm{x})$, $\tilde{\bm{z}}=(\tilde{y},\tilde{\bm{x}})$, and $\bm{\varphi}_{\bm{\omega}}(\bm{z})\df (\sqrt{\omega_{i}}y\varphi(\bm{x};\bm{\xi}_{i}^{k}))_{i=1,2,\cdots,m}^{k=1,2,\cdots,N}$. Moreover, $\mathcal{F}$ is the following function class
\begin{align}
\label{Eq:Fucktion}
\mathcal{F}=\{\varphi:\mathcal{X}\rightarrow \real: \bm{x}\mapsto \varphi_{0}(\langle \bm{x},\bm{\xi}\rangle+b ), \bm{\xi}\in \Xi, b\in \real, \varphi_{0}\in \mathrm{Lip}(1)\},
\end{align}
where $\mathrm{Lip}(1)$ is the class of $1$-Lipschitz functions. Then, with the probability of $1-\delta$, we have
\begin{align}
\nonumber
\sup_{g\in \mathcal{G}_{m,N}}\Bigg|U_{\lambda}(g,\bm{Z}_{1:n})-\expect_{P}\Bigg[\exp\Bigg({1\over \lambda}g(\bm{Z}_{1:n})\Bigg)\Bigg]\Bigg|=\dfrac{4}{\sqrt{\lceil n/2 \rceil }}\dfrac{1}{\delta} \int_{0}^{\eta}\sqrt{2N(d+2)\log(S(\epsilon,\lambda))}\mathrm{d}\epsilon,
\end{align}
where $\eta\df \exp(L^{2}/\lambda)$, $S(\epsilon,\lambda)\df {K\log (K/(\lambda \gamma(\epsilon) \exp(-L^{2}/\lambda))^{2})\over (\lambda\gamma(\epsilon) \exp(-L^{2}/\lambda))^{2}}$, $\gamma(\epsilon)\df {\lambda (\epsilon)^{1\over r}\over 2m\exp(L^{2}/\lambda)L}$, and $K=3e^{2}/(e-1)\approx 12.9008$.
\end{theorem}
\begin{proof}
The proof is presented in Appendix \ref{Appendix:Proof_of_Stochastic_Equicontinuity}.
\end{proof}

We remark that the function class in Eq. \eqref{Eq:Fucktion} of Theorem \ref{Theorem:Stochastic_Equicontinuity} includes the random Fourier feature model of Rahimi and Recht \cite{rahimi2008random,rahimi2009weighted} as a special case when $\varphi_{0}(x)=\cos(x)$. The stochastic equicontinutiy of Theorem \ref{Theorem:Stochastic_Equicontinuity} makes no assumptions on the underlying distributions of variables $\bm{\xi}$ and $b$ in the random feature model of Eq. \eqref{Eq:Fucktion}. Thus, it can be applied to base kernels whose distributions have with finite or infinite supports. In particular, the proof techniques of Theorem \ref{Theorem:Stochastic_Equicontinuity} is based on empirical process theory and the Vapnik–Chervonenkis theory of function classes.

\subsection{Non-asymptotic consistency and the minimax rate}

Recall the definition of the finite sample estimator $\widehat{H}_{n,N}(\lambda;\bm{\omega})$ from Equation \eqref{Eq:log_sum_exp}. Then, we observe that $\widehat{H}_{n,N}(\lambda;\bm{\omega})=\lambda\log U_{\lambda}+{\lambda r\over 2}$. Therefore, the asymptotic normality in Eq. \eqref{Eq:Asymptotic_normlaity} of Lemma \ref{Lemma:Asymptotic_Normality} holds for $\widehat{H}_{n,N}(\lambda;\bm{\omega})$. In the next lemma, we provide the non-asymptotic consistency of the sample average approximations (SAAs). To describe the consistency of SAAs, we define
\begin{subequations}
\begin{align}
T_{\bm{\omega}}(P)&\df \expect_{P^{\otimes 2}}\expect_{\mu_{\bm{\omega}}}\Big[y
\tilde{y}\varphi(\bm{x};\bm{\xi})\varphi(\tilde{\bm{x}};\bm{\xi})\Big]\\
\widehat{T}_{\bm{\omega}}(P)&\df \dfrac{1}{N}\expect_{P^{\otimes 2}}\Big[y\tilde{y}\langle \bm{\varphi}_{\bm{\omega}}(\bm{x}), \bm{\varphi}_{\bm{\omega}}(\tilde{\bm{x}}) \rangle \Big].
\end{align}
\end{subequations}
From Lemma \ref{Lemma:Saddle Point Characterization of Functional Optimization}, we recall the following identities
\begin{subequations}
\begin{align}
\label{Eq:State_11}
\sup_{\bm{\omega}\in \mathrm{S}_{+}^{m}}\min_{\lambda\in \real_{+}}H(\lambda;\bm{\omega})&= \sup_{\bm{\omega}\in \mathrm{S}_{+}^{m}}\inf_{P\in \mathcal{P}}T_{\bm{\omega}}(P)\\
\label{Eq:State_21}
\sup_{\bm{\omega}\in \mathrm{S}_{+}^{m}}\min_{\lambda\in \real_{+}}\widehat{H}_{n,N}(\lambda;\bm{\omega})&=\sup_{\bm{\omega}\in \mathrm{S}_{+}^{m}}\inf_{\widehat{P}^{n}\in \widehat{\mathcal{P}}^{n}}\widehat{T}_{\bm{\omega}}(\widehat{P}^{n}),
\end{align}
\end{subequations}
Now, the following consistency result can be established:
\begin{lemma}\textsc{{(Non-Asymptotic Consistency of the Monte Carlo Samplings)}}\normalsize
	\label{Thm:Consistency with respect to the Sampling Data}	
	Suppose Assumptions \ref{Assumption:1}-\ref{Assumption:3} hold. Let 
	\begin{align}
	\label{Eq:Min-Max_Opt}
\left(\bm{\omega}_{\ast},\widehat{P}_{\ast}^{n}\right)\df \sup_{\bm{\omega}\in \mathrm{S}_{+}^{m}}\inf_{\widehat{P}^{n}\in \widehat{\mathcal{P}}^{n}}\widehat{T}_{\bm{\omega}}(\widehat{P}^{n}),
	\end{align}
 denotes the saddle point of the finite sample max-min optimization problem. Then, with the probability of (at least) $1-3\rho$ we have
 \begin{align}
\nonumber
\left|\sup_{\bm{\omega}\in \mathrm{S}_{+}^{m}}\inf_{P\in \mathcal{P}}T_{\bm{\omega}}(P)-T_{\bm{\omega}_{\ast}}(\widehat{P}_{\ast}^{n})\right|&\leq   \dfrac{2KL}{n}\sqrt{ \log\left({2\times (16Ln)^{m}\over \rho\times K^{m} }\right)}+\dfrac{C^{3}K^{2}}{n^{2}}\\  \label{Eq:Consistency_Upper_Bound}
&\hspace{4mm}+4\sqrt{{2N(r+2\ln(2))\over L^{4}}\log\left({2\times (16L^{4})^{m}\over \rho (2N(r+2\ln(2))^{m\over 2}) }\right)}.
 \end{align}     
\end{lemma}
\begin{proof}
	The proof is presented in Section \ref{Appendix:Proof_of_Lemma_Consistency} of the supplementary materials.
\end{proof}

In the next theorem, we establish the minimax estimation rate of the weights of the mixture model in Eq. \eqref{Eq:Jeff}. Formally, let $K_{1}^{m}=(K_{1},\cdots,K_{m})$ denotes the set of base kernels. Given a distribution from the distributional ball $P\in \mathcal{P}$, we  consider estimating the optimal weights $\bm{\omega}$ given the i.i.d. training samples $(y_{i},\bm{x}_{i})_{1\leq i\leq n}\sim_{\text{i.i.d.}} P$. That is, we wish to estimate the map $\bm{\omega}:\mathcal{P}\rightarrow \mathrm{S}_{+}^{m}$ defined by
\begin{subequations}
	\label{Eq:min+max}
\begin{align}
P&\mapsto \bm{\omega}(P)=\max_{\bm{\omega}\in \mathrm{S}_{m}^{+}}\sum_{i=1}^{m}\omega_{i} \expect_{P^{\otimes 2}}\Big[y\tilde{y}K_{i}(\bm{x},\tilde{\bm{x}})\Big].
\end{align} 
\end{subequations}
Given the training samples $(y_{i},\bm{x}_{i})_{1\leq i\leq n}$, we denote an empirical estimator of the coefficients of the mixture model by $\widehat{\bm{\omega}}_{n}\df \widehat{\bm{\omega}}((y_{i},\bm{x}_{i})_{1\leq i\leq n})$. Recall the definition of the total variation (TV) metric $d_{\mathrm{TV}}:\mathrm{S}_{+}^{m}\times \mathrm{S}_{+}^{m}\rightarrow \real_{+}, (\bm{\omega},\widehat{\bm{\omega}}_{n})\mapsto d_{\mathrm{TV}}(\bm{\omega},\widehat{\bm{\omega}}_{n})\df {1\over 2}\|\bm{\omega}-\widehat{\bm{\omega}}_{n}\|_{1}$. The \textit{minmax rate} or \textit{minimax rate} of the empirical estimators is defined as the minimum TV distance between the best empirical estimator and the population estimator of Eq. \eqref{Eq:min+max} over the worst case data distribution that is drawn from the distribution ball $\mathcal{P}$, \textit{i.e.},
\begin{align}
R_{n}(\mathcal{P})\df \inf_{\widehat{\bm{\omega}}\in \mathrm{S}_{m}^{+}}\sup_{P\in \mathcal{P}}\expect_{P}\Big[d_{\mathrm{TV}}\big(\widehat{\bm{\omega}}_{n},\bm{\omega}\big)\Big],
\end{align}
where the infimum is taken with respect to all admissible estimators $\widehat{\bm{\omega}}_{n}$.  In the following theorem, we characterize the minmax estimation rate of $\bm{\omega}(P)$ in Eq. \eqref{Eq:min+max} for the special case of the RBF kernels:
\begin{theorem}\textsc{(The Minimax Estimation Rate)} 
\label{Thm:Min-Max_Rate}	
Suppose the base kernels are RBF, \textit{i.e.}, $K_{i}=\psi_{i}(\|\bm{x}-\bm{y}\|_{2})$ with the corresponding Sch\"{o}nberg measure $\nu_{i}$ for all $i\in [M]\df \{1,2,\cdots,M\}$, (cf. Definition \ref{Definition:Radial_Kernel}). Furthermore, consider the restriction $P(y)=Q(y)$ for the marginals, and assume that the distribution ball of the marginals $\mathcal{P}_{\pm}$ in Eq. \eqref{Eq:distribution_ball_of_marginals} contains the following normal distributions
\begin{subequations}
	\begin{align}
		P_{+}&=\mathsf{N}(\bm{\mu}_{+}^{P},\sigma^{2}\bm{I}_{d\times d}), \quad   P_{-}=\mathsf{N}(\bm{\mu}_{-}^{P},\sigma^{2}\bm{I}_{d\times d}),\\
		Q_{+}&=\mathsf{N}(\bm{\mu}_{+}^{Q},\sigma^{2}\bm{I}_{d\times d}), \quad Q_{-}=\mathsf{N}(\bm{\mu}_{-}^{Q},\sigma^{2}\bm{I}_{d\times d}),
	\end{align}
\end{subequations}
where $Q_{+}\otimes Q_{-}$ is the center of the distribution ball $\mathcal{P}_{\pm}$ with the radius $2r$. Furthermore, for some $i,j\in [M]$, suppose the inequalities $J_{k,d}(\infty)\leq {1\over 2}J_{i,d}(\tau_{P})$ and $J_{\ell,d}(\infty)\leq {1\over 2}J_{j,d}(\tau_{Q})$ hold for all $k\in [M]\backslash \{i\}$ and $\ell\in [M]\backslash \{j\}$, respectively, where
\small{\begin{align}
\label{Eq:Conditions_for_MMD}
J_{k,d}(\tau)&\df  \int_{0}^{\tau}\dfrac{t}{(1+4t\sigma^{2})^{(d+2)/2}}\nu_{k}(\mathrm{d}t),
\end{align}}\normalsize
and $\tau_{Q}\df {1\over \|\bm{\mu}_{+}^{Q}-\bm{\mu}_{-}^{Q}\|_{2}^{2}-4\sigma^{2}}$ and $\tau_{P}\df {1\over \|\bm{\mu}_{+}^{P}-\bm{\mu}_{-}^{P}\|_{2}^{2}-4\sigma^{2}}$. Then, the minimax rate of the empirical estimatiors is given by
\begin{align}
\label{Eq:minimax}
R_{n}(\mathcal{P})\geq \dfrac{1}{8}e^{-nD_{\mathrm{KL}}(P_{+}\otimes P_{-}||Q_{+}\otimes Q_{-})}\geq \dfrac{1}{8}e^{-nr}.
\end{align}
\end{theorem}
\begin{proof}
	The proof is presented in Appendix \ref{Appendix:The_Min_Max_Rate}.
\end{proof}

Let us make a remark about the restrictions $J_{k,d}(\tau_{P})\leq {1\over 2}J_{i,d}(\infty)$ and $J_{\ell,d}(\tau_{Q})\leq {1\over 2}J_{j,d}(\infty)$ in Theorem \ref{Thm:Min-Max_Rate} are sufficient conditions for the minimax inequality in Eq. \eqref{Eq:minimax} of Theorem \ref{Thm:Min-Max_Rate}. Alternatively, we can obtain less obsecure conditions by obtaining a lower bound and an upper bound on the integral formula of Eq. \eqref{Eq:Conditions_for_MMD}. To attain this objective, notice that the function $t\mapsto {t\over (1+4t\sigma^{2})^{(d+2)/2}}$ monotonically increases on the interval $[0,{1\over 2d\sigma^{2}}]$ and then decreases on $({1\over 2d\sigma^{2}}, \infty)$, thus attaining its maximum at $t={1\over 2d\sigma^{2}}$. Hence,
\begin{align}
J_{k,d}(\infty)&\leq \dfrac{1}{2d\sigma^{2}(1+(2/d))^{(d+2)/2}},  \quad \forall k\in \{1,2,\cdots,d\}. \label{Eq:Plug33}
\end{align}
Similarly, a lower bound can be obtained as follows
\begin{subequations}
\begin{align}
J_{k,d}(\tau_{P})&\geq \sup_{\tau \in (0,\tau_{P}]} \nu_{k}([\tau,\tau_{P}]) \dfrac{\tau}{e}\left(1-{2\over 2+d}\right)\\
&\geq \nu_{k}([\tau_{P}/2,\tau_{P}]) \dfrac{\tau_{P}}{2e}\left(1-{2\over 2+d}\right),  \quad \forall k\in \{1,2,\cdots,d\}.
\end{align}
\end{subequations}
where  $\nu_{i}([\tau_{0},\tau_{1}])\df \int_{\tau_{0}}^{\tau_{1}}\nu_{i}(\mathrm{d}s)$. The sufficient condition $J_{i,d}(\infty)\leq {1\over 2} J_{i,d}(\tau_{P})$ for the minimax lower bound of Eq. \eqref{Eq:minimax} now can be restated as follows
\small{
\begin{align*}
\nu_{i}([\tau_{P}/2,\tau_{P}])&\geq\dfrac{e}{\tau_{P}} \left(1-{2\over 2+d}\right)^{-1} \cdot\dfrac{1}{d\sigma^{2}(1+(2/d))^{(d+2)/2}}.
\end{align*}}\normalsize
A similar inequality can be stated for $J_{j,d}(\infty)\leq {1\over 2} J_{j,d}(\tau_{Q})$.

\subsection{Distributionally robust generalization bounds}

To obtain an out-of-sample generalization bound, we require the following notions of complexities of a function class:

\begin{definition}\textsc{(Rademacher and Gaussian Complexities)}
	\label{Def:Rademacher and Gaussian Complexities}
	Given a finite-sample set $S=(\bm{x}_{1},\cdots,\bm{x}_{n})$, and function class $\mathcal{F}$ consider the following random variable
	\begin{align}
	\label{Eq:Rademacher}
	\hspace{-3mm}	\widehat{\mathfrak{R}}^{n}_{S}(\mathcal{F})\df \expect_{P_{\bm{\varepsilon}}}\left[\sup_{f\in \mathcal{F}} \dfrac{1}{n}\sum_{i=1}^{n}\varepsilon_{i} f(\bm{x}_{i})\Bigg \vert \bm{x}_{1},\cdots,\bm{x}_{n}\right].
	\end{align}
	where $\bm{\varepsilon}\df (\varepsilon_{1},\cdots,\varepsilon_{n})$ are \textit{i.i.d.} Rademacher random variables, \textit{i.e.}, $(\varepsilon_{1},\cdots,\varepsilon_{n})_ {\sim_{\text{i.i.d.}}}P_{\varepsilon}=\mathrm{Uniform}\{-1, 1\}$, and $P_{\bm{\varepsilon}}=P_{\varepsilon}^{\otimes n}$. The \textit{Rademacher complexity} is then defined as the expected value of $\widehat{\mathfrak{R}}^{n}_{S}(\mathcal{F})$ with respect to the observed samples, \textit{i.e.}, 
	\begin{align}
	\label{Eq:Rademacher_tired}
	\mathfrak{R}^{n}(\mathcal{F})\df \expect_{P_{\bm{X}}}[\widehat{\mathfrak{R}}_{S}^{n}(\mathcal{F})]. 
	\end{align}
	Similarly, the \textit{Gaussian complexities} $\mathfrak{G}_{S}^{n}(\mathcal{F})$ and $\mathfrak{G}^{n}(\mathcal{F})$ are defined by letting $P_{\varepsilon}=\mathsf{N}(0,1)$ in Eqs. \eqref{Eq:Rademacher}-\eqref{Eq:Rademacher_tired}, respectively.
\end{definition}

In the following proposition, we establish novel  upper bounds on the Rademacher and Gaussian complexities of the class of functions based on the mixture of random features:

\begin{proposition}{\textsc{(Upper Bounds on the Function Class Complexities)}}
	\label{Proposition:1}
	Consider the class of functions $\mathcal{F}_{m}(R)$ defined as follows
	\begin{align}
	\label{Eq:Function_Class}
	\mathcal{F}_{m}(R)\df \left\{f:\real^{d}\rightarrow \real:\bm{x}\mapsto f(\bm{x})={1\over \sqrt{N}}\bm{\beta}^{T}\bm{\varphi}_{\bm{\omega}}(\bm{x}):\bm{\beta}\in \ball^{2}_{r}(\bm{0}),\bm{\omega}\in \mathrm{S}_{m}^{+}\right\},
	\end{align}
	where $\ball^{2}_{r}(\bm{0})\subset \real^{mN}$ is the Euclidean $\ell_{2}$-norm ball of radius $r={R\over \sqrt{mN}}$ centered at the origin, and $\bm{\varphi}_{\bm{\omega}}\in \real^{mN}$ is the concatenated random feature vector. Moreover, define the row matrix of the random feature maps as $\bm{\Phi}(\bm{\omega})\df (\bm{\varphi}_{\bm{\omega}}(\bm{x}_{i}))_{1\leq i\leq n} \in \real^{n\times mN}$. Then, the \textit{Rademacher complexity} of the class $\mathcal{F}_{m}(R)$ is 
	\begin{align}
		\label{Eq:Sharper_1}
		\mathfrak{R}^{n}(\mathcal{F}_{m}(R))\leq \dfrac{R}{nN} \sqrt{192\pi\over m}{\|\bm{\Phi}(\bm{\omega})\|_{2}}\mathcal{Q}\left(-\sqrt{192 \min\{mN,n\}}  \right),
	\end{align}
	and the \textit{Gaussian complexity} of the class $\mathcal{F}_{m}(R)$ is 
    \begin{align}
		\label{Eq:Sharper_2}
		&\mathfrak{G}^{n}(\mathcal{F}_{m}(R))\leq \dfrac{R}{nN}{\sqrt{\pi\over 2m}}\|\bm{\Phi}(\bm{\omega})\|_{F}\mathcal{Q}\left(-\sqrt{{n\over 4}} \right),
		\end{align}
	where $\mathcal{Q}(x)\df {1\over \sqrt{2\pi}}\int_{x}^{\infty}e^{-{x^{2}\over 2}}\mathrm{d}x$ is the complementary error function.
\end{proposition}
\begin{proof}
	The proof is presented in Section \ref{Appendix:Proof_of_RGComplexities} of the supplementary material. 
\end{proof}
Let us make some remarks about the implications of Proposition \ref{Proposition:1}.

The proof techniques we used to establish the upper bound on the Rademacher complexity are novel and based on concentration inequalities for the quadratic forms. In particular, the proof does \textit{not} rely on the Khintchin-Kahane type inequality which is a typical tool in the computation of the upper bounds on the Rademacher complexity; see, \textit{e.g.},  \cite{cortes2009new},\cite{shahrampourlearning}. As shown in Appendix \ref{Appendix:Proof_of_RGComplexities}, the Khintchin-Kahane type inequality yields the following upper bound on the Rademacher complexity:
\begin{align}
\label{Eq:Sharper_3}
\mathfrak{R}^{n}(\mathcal{F}_{m}(R))\leq \dfrac{R}{nN}\sqrt{\dfrac{23}{44m} } \|\bm{\Phi}(\bm{\omega})\|_{F}.
\end{align}
Since $0<\mathrm{erfc}(x)\leq 2$ for all $x\in \real$, and $\|\bm{\Phi}(\bm{\omega})\|_{2}\leq \|\bm{\Phi}(\bm{\omega})\|_{F}$ for any feature matrix $\bm{\Phi}(\bm{\omega})\in \real^{n\times mN}$, the upper bound we established in Eq. \eqref{Eq:Sharper_1} is sharper than that of Eq. \eqref{Eq:Sharper_3}. Similarly, the proof of the upper bound on the Gaussian complexity is based on the concentration of measure for $\chi^{2}$ random variables.

To compute explicit bounds from the Rademacher and Gaussian complexities, we require upper and lower bounds on the matrix norms of the feature matrix $\bm{\Phi}(\bm{\omega})$. The following lemma provides a first step to characterize such a result:
\begin{lemma}\textsc{(Uniform Concentration of the Spectra of Feature Matrices)}
	\label{Lemma:Phi_Norm}
	Suppose Assumptions \ref{Assumption:1}-\ref{Assumption:3} holds. Consider the feature matrix $\bm{\Phi}(\bm{\omega})=(	\bm{\varphi}_{\bm{\omega}}(\bm{x}_{i})))_{1\leq i\leq n}$ and the associated kernel matrix $\bm{K}({\bm{\omega}})\df (K_{\bm{\omega}}(\bm{x}_{i},\bm{x}_{j}))_{1\leq i,j\leq n}$. Then, the following concentration inequality holds
	\begin{align}
		\label{Eq:CONC2}
	\nonumber
	&\prob\left(\sup_{\bm{\omega}\in \mathrm{S}_{m}^{+}} \left| \|\bm{K}(\bm{\omega}) \|_{2}-\dfrac{1}{N}\|\bm{\Phi}\bm{\Phi}^{T}(\bm{\omega})\|_{2}  \right|\geq \delta \right)\\
	&\hspace{20mm}\leq  2^{14}\left(\dfrac{m^{2}n^{2}L^{2}}{\delta^{2}}\right)\exp\left({6n\log 3\over m+2}\right)\cdot\exp\left(-{cN\delta^{2}\over 2n^{2}m(m+2)L^{2}} \right),
		\end{align}
where $c>0$ is a universal constant. In particular, 
\begin{align}
\nonumber
&\lim\sup_{N\rightarrow \infty}\prob\left(\sup_{\bm{\omega}\in \mathrm{S}_{m}^{+}} \left| \|\bm{K}(\bm{\omega}) \|_{2}-\dfrac{1}{N}\|\bm{\Phi}\bm{\Phi}^{T}(\bm{\omega})\|_{2}  \right|\geq \delta \right)=0.
\end{align}
\end{lemma}
\begin{proof}
	The proof is presented in Section \ref{App:Proof_of_Lemma_Phi_Norm} of the supplementary materials.
\end{proof}

Lemma \ref{Lemma:Phi_Norm} suggests that the spectra of the feature matrix ${1\over N}\bm{\Phi}\bm{\Phi}^{T}(\bm{\omega})$ uniformly concentrates around that of  the kernel matrix $\bm{K}({\bm{\omega}})$ at the rate of $\mathcal{O}(N^{-{1\over 2}}\log(N))$ as the number of random feature samples $N$ tends to infinity. Consequently, to establish the spectral bound $\|\bm{\Phi}(\bm{\omega})\|_{2}$, it suffices to establish the spectral bound for the kernel matrix. The asymptotic analysis of the  spectrum of the (radial) kernel matrices has been studied extensively in the random matrix literature; see, \textit{e.g.}, \cite{el2010spectrum} for locally smooth kernel functions, and \cite{fan2015spectral} for the non-smooth kernel functions. Those asymptotic results are universal and consequently do not impose any conditions on the distribution of feature vectors. In the sequel, we are concerned with the \textit{non-asymptotic} results for the spectrum of the Gram matrices of the mixture kernels:

\begin{lemma}\textsc{(Spectral Norm of the Translation Invariant Kernels)}
	\label{Lemma:Kn_Norm}
	Suppose Assumptions \ref{Assumption:1}-\ref{Assumption:3} holds. Suppose the base kernels are centered and translation invariant in the sense that  $K_{i}(\bm{x},\widehat{\bm{x}})=\psi_{i}(\bm{x}-\widehat{\bm{x}})-\psi_{i}(\bm{0})\bm{1}_{\{\bm{x}=\widehat{\bm{x}}\}}$ for all $i=1,2,\cdots,m$. Then, with the probability of at least $1-\rho$, we have that
\begin{align}
\|\bm{K}(\bm{\omega})\|_{2}\leq 2n\sqrt{1-{1\over n}}+4\sqrt{\ln \dfrac{4\times 3^{2n}}{\rho}},
\end{align}
for all $\bm{\omega}\in \mathrm{S}_{+}^{m}$.
\end{lemma}

\begin{proof}
	The proof is deferred to Section \ref{Lemma:Kn_Norm} of the supplementary materials.
\end{proof}
We remark that our non-asymptotic analysis of the spectrum of kernel matrices in Lemma \ref{Lemma:Kn_Norm}  generalizes those of \cite{kasiviswanathan2015spectral} that are established for the special cases of the Gaussian and polynomial kernels. Moreover, the spectral bounds in \cite{kasiviswanathan2015spectral} are stated under the assumption that the features $(\bm{x}_{i})_{1\leq i\leq n}$ are i.i.d. sub-Gaussian random variables. More specifically, the proof we present employs the \textit{decoupling} technique from the random matrix theory \cite{vershynin2010introduction} as well as a concentration of measure for Lipschitz functions of the Gaussian random variables.

We now are in position to state the main result of this section which is a generalization bound for the risk function based on the Rademacher complexity of the class of function $\mathcal{F}_{m}(R)$ defined by Eq. \eqref{Eq:Function_Class}:
\begin{theorem}\textsc{(Distributionally Robust Generalization Bound)}
	\label{Thm:Bartlett}
	 Consider the function class $\mathcal{F}_{m}(R)$ defined in Eq. \eqref{Eq:Function_Class} mapping from $\mathcal{X}$ to $\mathcal{Z}$. Let $(\bm{x}_{i},y_{i})_{1\leq i\leq n}\sim_{\text{i.i.d.}} Q$  denote the training data sampled from the joint probability distribution $Q\in \mathcal{M}(\mathcal{Z})$. Suppose there exists a dominating  loss function in the sense that
	 \begin{align}
	 \ell(y,f(\bm{x}))\leq \tilde{\ell}(y,f(\bm{x})), 
	 \end{align}
	for all $z=(y,\bm{x})\in  \mathcal{Z}=\mathcal{Y}\times \mathcal{X}$, and all $f\in \mathcal{F}_{m}(R)$ such that  $\tilde{\ell}(y,\cdot):\mathcal{X}\rightarrow \mathcal{Z}$ is $L_{\psi}$-Lipschitz with respect to the Euclidean norm defined on $\mathcal{Z}$. Then, for any integer $n$ and any $0<\delta<1$, with probability at least $1-\delta$ over samples of length $n$, every $f$ in $\mathcal{F}_{R}$ satisfies
	\begin{align}
	\nonumber
	\label{Eq:Distributionally_Robust}
	\sup_{P\in \mathcal{P}}\expect_{P}[\ell(Y,f(\bm{X}))]&\leq 2
	\expect_{\widehat{Q}^{n}}\left[{\tilde{\ell}(y,f(\bm{x}))}\right]\\
	&\hspace{4mm}+{2L_{\psi}\over B}\mathfrak{R}^{n}(\mathcal{F})+B \sqrt{{8\ln(2/\delta)\over n}}+\dfrac{Br}{2},
		\end{align}
where $B\df \sup_{(y,z)\in \mathcal{Y}\times \mathcal{Z}} \tilde{\ell}(y,z)$, and $\mathcal{P}=\{P\in \mathcal{M}(\mathcal{X}\times \mathcal{Y}):D_{\mathrm{KL}}(P||Q)\leq r, P\ll Q \}$.
	\hfill $\square$
\end{theorem}

The proof is presented in Appendix \ref{App:Proof_of_Theorem_Bartlett}.

Let us make two remarks about the distributionally robust generalization bound in Eq. \eqref{Eq:Distributionally_Robust}.  First, the result of Bartlett and Mendelson \cite[Theorem 8]{bartlett2002rademacher} regarding the standard generalization bound is as follows
\begin{align}
\sup_{P\in \mathcal{P}}\expect_{P}[\ell(Y,f(\bm{X}))]&\leq 
\expect_{\widehat{Q}^{n}}\left[{\tilde{\ell}(y,f(\bm{x}))}\right]+{2L_{\psi}}\mathfrak{R}^{n}(\mathcal{F})+\sqrt{{8\ln(2/\delta)\over n}}.
\end{align}
Therefore, for $r=0$, the bound in Eq. \eqref{Eq:Distributionally_Robust} is not expected to be tight. Second, due to Theorem \ref{Lemma:Phi_Norm}, the Rademacher complexity of the function class $\mathcal{F}_{m}(\mathcal{R})$ in the generalization bound of Eq. \eqref{Eq:Distributionally_Robust} can be related to the spectral norm of the kernel matrix via the following inequality
\begin{align}
\nonumber
\mathfrak{R}_{n}(\mathcal{F}_{m}(R))\leq \dfrac{2R}{nD}\sqrt{192\pi N\over m} \left(\|\bm{K}(\bm{\omega})\|_{2}+\mathcal{O}\left(\dfrac{m}{\sqrt{N}} \right) \right)^{1\over 2},
\end{align}
where the upper bounds on the spectral norm of the mixture kernel matrix $\|\bm{K}(\bm{\omega})\|_{2}$ are established in Lemmas \ref{Lemma:Kn_Norm}.

\section{Empirical Evaluation: Perturbation of Input Features}
\label{Section:Empirical Evaluation for Classification of Real-World Datasets}

\begin{figure}[t!]
	\begin{center}
		\hspace*{-5mm}		\subfigure{
			\includegraphics[trim={.2cm .2cm .2cm  .2cm},width=.2\linewidth]{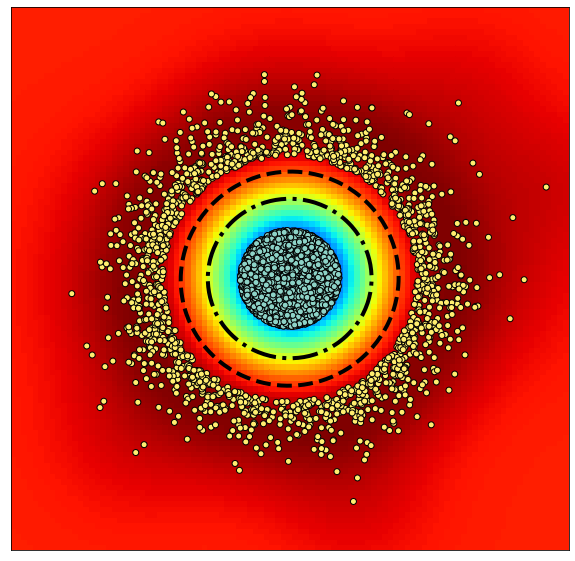} 
			\includegraphics[trim={.2cm .2cm .2cm .2cm},width=.2\linewidth]{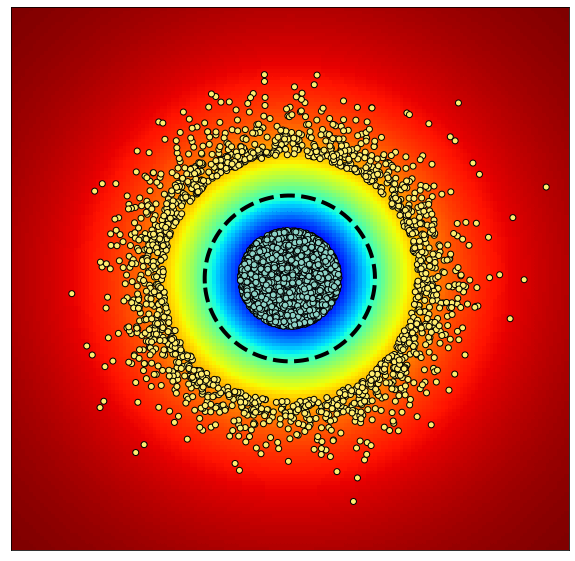}}  \\
		\subfigure{\footnotesize{\hspace{-2mm} (a) \hspace{24mm} (b)}}	\\ 
		\caption{\footnotesize{The decision boundaries of the kernel SVM. Panel (a): trained kernel mixture model using Algorithm \ref{Algorithm:1} with the distributional ball radius $r=0.01$ (dash-dot line) and $r=100$ (dash line). Panel (b): untrained kernel mixture model with the uniform weights $\bm{w}=\bm{1}/m$.}}
		\subfigure{
			\includegraphics[trim={.2cm .2cm .2cm  .2cm},width=.2\linewidth]{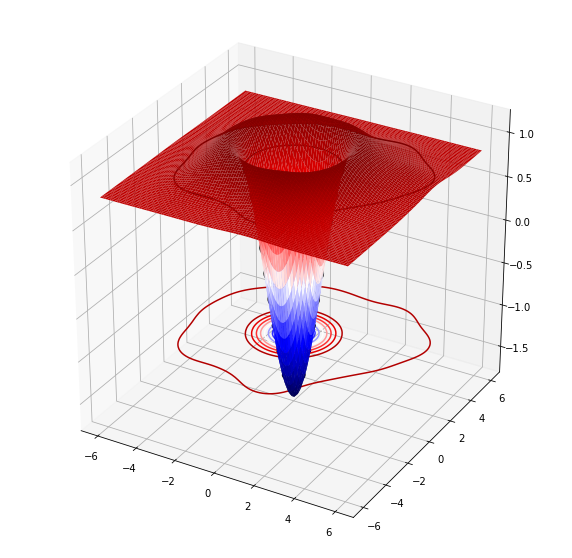} 
			\includegraphics[trim={.2cm .2cm .2cm  .2cm},width=.2\linewidth]{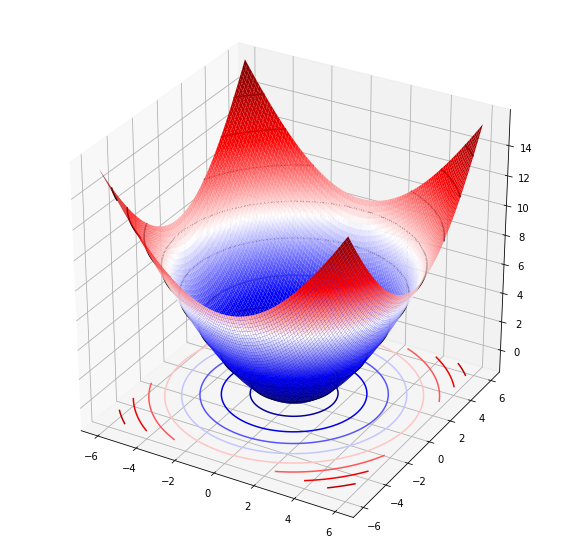} 
			\includegraphics[trim={.1cm .1cm .1cm .1cm},width=.2\linewidth]{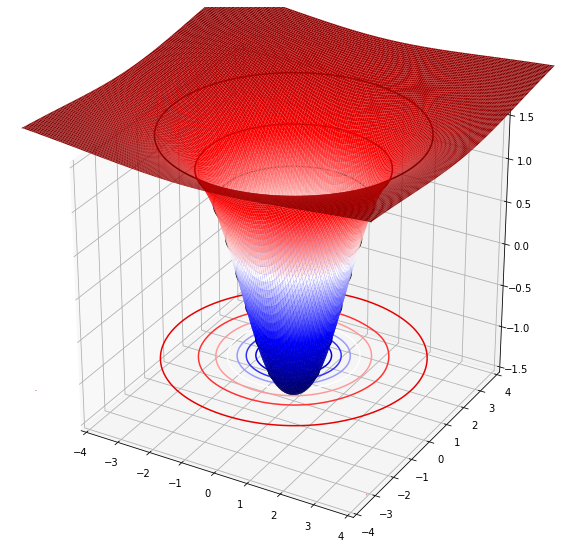}}
		\\ 		\vspace{-3mm}	\subfigure{\footnotesize{\hspace{3mm} (a) \hspace{25mm} (b) \hspace{25mm} (c) }}
		
		\caption{\footnotesize{The level sets of the decision function of the kernel SVMs with multiple kernels trained via the robust optimization method of Algorithm \ref{Algorithm:1}. Panel (a): $r=0.01$, Panel (b): $r=100$,  Panel (c): mixture kernel with the uniform weights  $\bm{w}=\bm{1}/m$. }}
		\label{Fig:Comparison_CV} 
	\end{center}
\end{figure}

\subsection{Synthetic data-set}
For the numerical experiments with the synthetic data, we adapt the model of \cite{sinha2017certifying}. In particular, we generate synthetic i.i.d. data $(\bm{x}_{i},y_{i})_{1\leq i\leq n}\sim_{\text{i.i.d.}} Q$ whereby $(\bm{x}_{1},\cdots,\bm{x}_{n})\sim_{\text{i.i.d.}} \mathsf{N}(\bm{0},\bm{I}_{2\times 2})$ with the class labels $y_{i}=\mathrm{sgn}(\|\bm{x}_{i}\|_{2}-\sqrt{2}),i=1,2,\cdots,n$. To
widen the separating margin of the two classes, we remove data points whose norms falls inside the interval $\|\bm{x}_{i}\|_{2}\in (\sqrt{2}/1.3,1.3\sqrt{2})$. To train the mixture of the base kernels, we consider a set of radial Gaussian kernels $K_{i}(\bm{x},\widehat{\bm{x}})=\exp\left(-\gamma_{i}\|\bm{x}-\widehat{\bm{x}}\|_{2}^{2} \right)$, with the bandwidth parameters $\gamma_{i}\in \{0.5,1,1.5,\cdots,9.5,10\}$.

 In Figure \ref{Fig:Comparison_CV}(a), we illustrate the decision boundaries of the kernel SVM with the radius $r=0.01$ (dash-dot line) and $r=100$ (dash line). Clearly, for a larger radius, the decision boundary is pushed outward. This observation is in line with the one that is made in \cite{sinha2017certifying}. Namely, since $70\%$ of data is accumulated in the inner region $\|\bm{x}_{i}\|_{2}\leq \sqrt{2}/1.3$, to incur the maximum classification error, an adversery pushes the features in the region $\|\bm{x}_{i}\|_{2}\leq \sqrt{2}/1.3$ outwards close to the boundaries of $\|\bm{x}\|_{2}=1.3/\sqrt{2}$. For a reference, in Figure \ref{Fig:Comparison_CV}(b), we also show the decision boundary of the kernel SVM, when the kernels are mixed uniformly, \textit{i.e.}, $\bm{\omega}=\bm{1}/m$.  In Figures \ref{Fig:Comparison_CV}(a) and \ref{Fig:Comparison_CV}(b), the level sets of the deicision function
 \begin{align}
 \label{Eq:decision}
 \bm{x}\mapsto \Psi(\bm{x})\df \mathrm{sgn}\left(\omega_{0}+\sum_{i=1}^{n-1}\omega_{i}y_{i}K(\bm{x},\bm{x}_{i})\right),
 \end{align}
 for $r=0.01$ and $r=100$ are shown, respectively, where $\mathrm{sgn}(\cdot)$ is the sign function. Clearly, increasing the radius of the distribution ball $\mathcal{P}$ changes the decision boundaries of the resulting predictor.

 \subsection{MNIST data-set} We apply our robust optimization method to the adversarial perturbations of the out of sample (test) data of a model trained on the MNIST data-set.\footnote{The Jupyter notebook of our Python 3 codes for this experiment is available in the Github repository: \url{https://github.com/mbadieik/Adversarial-MKL}} All computations of this part were performed on a DGX Station from NVIDIA running Linux operating system with an Intel Xeon E5-2698 v4 2.2 GHz (20-Core) CPU and two of four total Tesla V100 GPUs (32 GB memory for each GPU). We present our results for the robust classification of images from MNIST databases \cite{lecun1998gradient}. We use \textsc{Pytorch} in Python 3.5. We first train a convolutional neural network (CNN) to extract features of input images on unmodified training data-set, and use the extracted features for classification with kernel SVM. The trained CNN has two convolutional layers followed by the max pooling layer. The archtiecture of our CNN is as follows:
\begin{itemize}[leftmargin=*]
	\item  Convolutional: Output Channel: 32, Kernel size: $3\times 3$,  Activation: ReLU,  Max Pooling: $2\times 2$,
	
	\item Convolutional: Output Channel 16, Kernel size: $3\times 3$,  Activation: ReLU,  Max Pooling: $2\times 2$,

    \item Fully connected: Input: $5\times 5\times 32$,  Output:100, Activation: ReLU,
    
    \item Kaiming Initialization \cite{he2015delving},
    
    \item  Fully connected: Input: $100$,  Output:10, Activation: ReLU.
\end{itemize}
 To train the CNN, we split the data-set into $60,000$ training data and $10,000$ test data points. We use a binary cross entropy loss function in conjunction with the SGD with the momentum $0.9$ and the learning rate $0.01$ for the optimization of the loss.

 \subsection{Adversarial models} In the sequel, we briefly review the adversarial models we consider in this paper. The adversarial models in the sequel are primarily devised to test the robustness of the end-to-end deep neural networks. Nevertheless, since the features for the downstream kernel SVM are extracted from the penultimate layer of a CNN--- before the soft-max output layer--- these adversarial models are also useful tools to systematically perturb the extracted features in order to examine the robustness of the learned kernel mixture model using Algorithms \ref{Algorithm:1} and \ref{Algorithm:2}.

\begin{figure}[t!]
	\begin{center}
    	\subfigure{
			\includegraphics[trim={.2cm .2cm .2cm  .2cm},width=.23\linewidth]{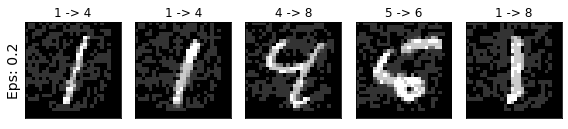} \hspace{3mm}
			\includegraphics[trim={.2cm .2cm .2cm  .2cm},width=.23\linewidth]{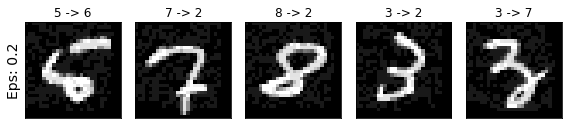}\hspace{3mm}
			\includegraphics[trim={.2cm .2cm .2cm  .2cm},width=.23\linewidth]{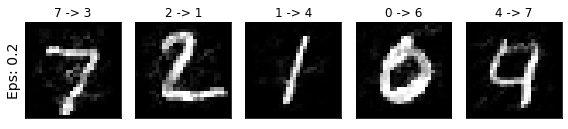}\hspace{3mm}
			\includegraphics[trim={.2cm .2cm .2cm  .2cm},width=.23\linewidth]{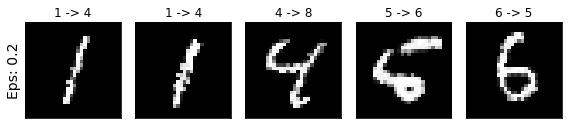}}\\	
		\subfigure{
			\includegraphics[trim={.2cm .2cm .2cm  .2cm},width=.23\linewidth]{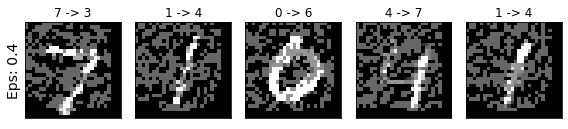} \hspace{3mm}
			\includegraphics[trim={.2cm .2cm .2cm  .2cm},width=.23\linewidth]{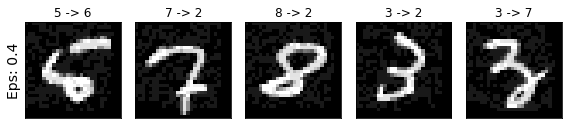}\hspace{3mm}
			\includegraphics[trim={.2cm .2cm .2cm  .2cm},width=.23\linewidth]{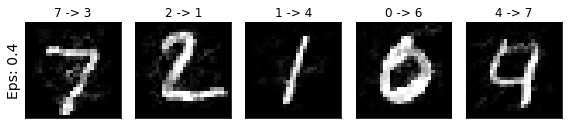}\hspace{3mm}
			\includegraphics[trim={.2cm .2cm .2cm  .2cm},width=.23\linewidth]{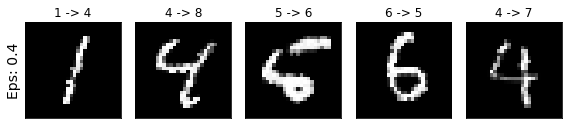}}\\
		\subfigure{
			\includegraphics[trim={.2cm .2cm .2cm  .2cm},width=.23\linewidth]{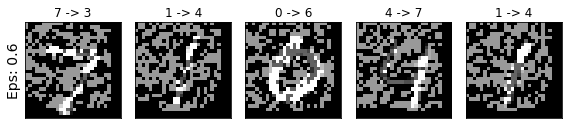} \hspace{3mm}
			\includegraphics[trim={.2cm .2cm .2cm  .2cm},width=.23\linewidth]{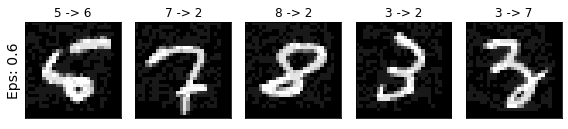}\hspace{3mm}
			\includegraphics[trim={.2cm .2cm .2cm  .2cm},width=.23\linewidth]{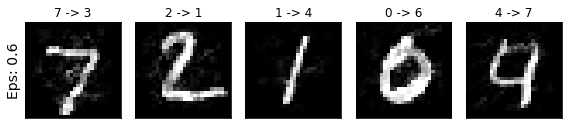}\hspace{3mm}
			\includegraphics[trim={.2cm .2cm .2cm  .2cm},width=.23\linewidth]{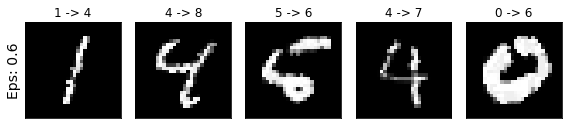}}\\
 		\subfigure{
			\includegraphics[trim={.2cm .2cm .2cm  .2cm},width=.23\linewidth]{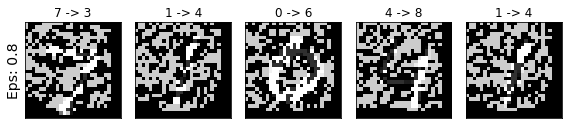} \hspace{3mm}
			\includegraphics[trim={.2cm .2cm .2cm  .2cm},width=.23\linewidth]{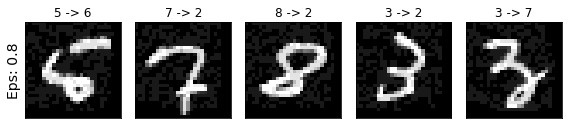}\hspace{3mm}
			\includegraphics[trim={.2cm .2cm .2cm  .2cm},width=.23\linewidth]{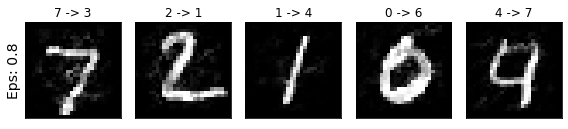}\hspace{3mm}
			\includegraphics[trim={.2cm .2cm .2cm  .2cm},width=.23\linewidth]{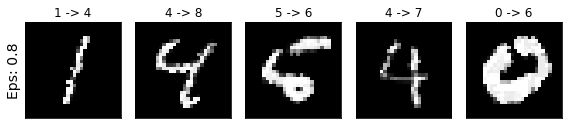}}
		\caption{\footnotesize{Example of adversarial inputs resulting in misclassification on MNIST dataset for different perturbation levels $\epsilon \in \{0.2,0.4,0.6,0.8\}$. Panel (a): FGSM  \cite{goodfellow2014explaining}, Panel (b): PGM  \cite{madry2017towards}, Panel (c): DDN  \cite{rony2019decoupling}, Panel (d): C\&W \cite{carlini2017towards}.}}
		\label{Fig:Comparison_CV} 
	\end{center}
\end{figure}

\subsubsection{Fast Gradient Signed Method \text{(FGSM)}  \cite{goodfellow2014explaining}}  This method was proposed to improve the robustness of a neural network model against input purturbations. Specifically, given a loss function $\ell((\bm{x},y);\bm{\theta})$ for deep neural networks with the weights $\bm{\theta}$, the $i$-th out-of-sample feature vector is perturbed by the additive noise $\tilde{\bm{x}}_{i}=\bm{x}_{i}+\Delta\bm{x}_{i}(\bm{\theta})$, where $\Delta\bm{x}_{i}(\bm{\theta})$ is the maximal perturbation computed via the following optimization problem
\begin{subequations}
 \begin{align}
 \Delta\bm{x}_{i}(\bm{\theta})&= \arg \min_{\bm{n}:\|\bm{n} \|_{\infty}\leq \epsilon}\ell(\bm{x}_{i}+\bm{n},y_{i};\bm{\theta})\\
 &=-\epsilon\cdot \mathrm{sgn}\Big(\nabla_{\bm{x}} \ell(\bm{x}_{i},y_{i};\bm{\theta})^{T}\cdot \bm{n}\Big). 
 \end{align}
\end{subequations}
where $\mathrm{sgn}(\cdot)$ is the sign function, and the parameter $\epsilon>0$ determines the amount of the adverserial perturbations. 
  
 \subsubsection{Projected Gradient Method (PGM) \cite{madry2017towards}} This method is an iterative approach to compute the adversarial peturbations. The PGM augments the stochastic gradient steps for the parameter $\bm{\theta}$ with the projected gradient ascent over $\bm{x}\mapsto \ell(\bm{x},y;\bm{\theta})$, where
 \begin{subequations}
 \begin{align}
 	\Delta\bm{x}^{t+1}_{i}(\bm{\theta})&=\arg\max_{\bm{n}:\|\bm{n}\|_{p}\leq \epsilon}\nabla_{\bm{x}} \ell(\bm{x}_{i},y_{i};\bm{\theta})^{T}\cdot \bm{n}\\
 	\bm{x}_{i}^{t+1}&=\mathscr{P}_{\ball_{\epsilon}^{p}(\bm{x}^{t}_{i})}\left(\bm{x}_{i}^{t}+\alpha_{t} 	\Delta\bm{x}^{t+1}_{i}(\bm{\theta}) \right), \quad t=1,2,\cdots,T_{\text{adv}},
 \end{align}
\end{subequations}
where $\eta_{t}$ is the step-size (learning rate), $\mathscr{P}(\cdot)$ is the Euclidean projection, and $\ball_{\epsilon}^{p}(\bm{x}^{t}_{i})\df \{\bm{x}\in \real^{n}: \|\bm{x}-\bm{x}^{t}_{i}\|_{p}\leq \epsilon \}$ is the $p$-norm Euclidean ball of the radius $\epsilon$ centered at $\bm{x}_{i}^{t}$. 

\subsubsection{Decoupled Direction and Norm Perturbation (DDN) \cite{rony2019decoupling}} This method induces misclassification with low $L_{2}$ norm. This advarserial model optimizes the cross entropy loss, and instead of penalizing the norm in each iteration, projects the perturbation onto a $L_{2}$-sphere centered at the original image. In particular, the following iterative method is used
\begin{subequations}
\begin{align}
\tilde{\bm{x}}^{t}&=\bm{x}+\eta^{t}\dfrac{\bm{\delta}^{t}}{\|\bm{\delta}^{t}\|_{2}},\\
\bm{\delta}^{t}&=\bm{\delta}^{t-1}+\alpha m \dfrac{\nabla_{\bm{x}}\ell( \bm{x}^{t-1},y;\bm{\theta}) }{\|\nabla_{\bm{x}}\ell( \bm{x}^{t-1},y;\bm{\theta})\|_{2}},
\end{align}
\end{subequations}
where $\alpha>0$ and $\eta^{t}>0$ are fixed and time-varying step-sizes, and $m$ is a parameter taking the value of $m=-1$ for a targeted attack and $m=1$ otherwise. 

\subsubsection{Carlini and Wagner (C\&W) Method \cite{carlini2017towards}} The C\&W $L_{2}$ model minimizes two criteria at the same time: the perturbation that makes the sample adversarial, and the $L_{2}$ norm of the perturbation. Instead of using a box-constrained optimization method, they propose changing variables using the $\tanh$ function, and instead of optimizing the cross-entropy of the adversarial example, they use a difference between logits. For a targeted attack aiming to obtain class $j$, with $Z(\cdot)$ denoting the model output before the softmax activation (logits), it optimizes:
\begin{align}
\min_{\delta\in \real} \|\bm{x}-\tilde{\bm{x}}(\delta)\|_{2}^{2}+Cf(\tilde{\bm{x}}(\delta)),
\end{align}
where
\begin{subequations}
\begin{align}
	\label{Eq:F1}
f(\tilde{\bm{x}}(\delta))&\df \max\Big\{ \max_{i\not= j}\{Z(\tilde{\bm{x}}(\delta))_{i}\}-Z(\tilde{\bm{x}}(\delta))_{j},-\epsilon\Big\},\\
\label{Eq:F2}
\tilde{\bm{x}}(\delta)&\df \dfrac{1}{2}\left(\tanh(\arctan(\bm{x})+\delta)+1\right).
\end{align}
\end{subequations} 
In Equation \eqref{Eq:F1}, $\epsilon$ is the confidence parameter. The larger $\epsilon$, the adversarial sample will be misclassified with higher confidence. Moreover, $Z(\tilde{\bm{x}})_{i}$ denotes the logit corresponding to the $i$th class label. To use this model in the untargeted settings, the definition of $f$ is modified to 
\begin{align}
	f(\tilde{\bm{x}}(\delta))=\max\Big(Z(\tilde{\bm{x}}(\delta))_{y}-\max_{i\not=y}\{Z(\tilde{\bm{x}}(\delta))_{i}  \},-\epsilon \Big).
\end{align}

\subsection{Alternative MKL methodologies}

We compare our method with the following traditional MKL learning paradigms for classfication of the standard test datasets:

\begin{itemize}[leftmargin=*]
\item \textsc{AverageMKL}: simple average of base kernels  \cite{lauriola2020mklpy}, 

\item \textsc{EasyMKL}: fast and memory efficient margin-based combination \cite{aiolli2015easymkl}, 

\item \textsc{GRAM}: radius/margin ratio optimization  \cite{lauriola2017radius}, 

\item \textsc{MEMO}: margin maximization and complexity minimization \cite{lauriola2018minimum}, 

\item  \textsc{RMKL}: margin and radius based multiple kernel learning \cite{do2009margin}, 

\item \textsc{FHeuristic}: heuristic based on kernels alignment  \cite{qiu2008framework}, 

\item \textsc{CKA}: centered kernel alignment \cite{cortes2010two},

\item \textsc{PWMK}: heuristic based on individual kernels performance \cite{tanabe2008simple}.
\end{itemize}

We implement these methods by leveraging the \textsc{MKLpy} library \cite{lauriola2020mklpy}.\footnote{\url{https://pypi.org/project/MKLpy/}} For \textsc{GRAM}, \textsc{MEMO}, and \textsc{RMKL}, we select the step-size (learning rate) of $\eta=0.1$, and the number of iterations $T=500$. For \textsc{EasyMKL},  the regularization parameter of $\lambda=0.1$ is adapted in \textsc{MKLpy} library which we also use in our experiments.

\subsection{Numerical results}  In Table \ref{Table:long_table}, we tabulate the test error of different MKL paradigms for different parameteres $\epsilon \in \{0,0.1,\cdots,0.9\}$ of the adverserial models. For PGM model, we use the step-size $\alpha_{t}=2/255$ and $T_{\text{adv}}=40$ iterations. For Algorithm \ref{Algorithm:1}, we use the step-sizes of $\beta=0.7$ for the primal and $\eta=0.3$ for the dual vectors. We also use $T=200$ iterations without applying any stopping criterion. In addition, our simulations are reported for the batch size of $B=200$. To train the kernel in kernel SVMs, we only use a fraction of training data-set ($1000$ training samples out of $60,000$), while the test is performed on the entire test data-set of $10000$ samples. The reduction of the size of training data-set is due to the fact that certain MKL methods, such as CKA or GRAM, scale poorly to large training data-sets, and to include these methods in our comparision, we had to use a smaller number of training data samples. In our experiments with Algorithm \ref{Algorithm:1}, we leverage a distributional ball with the radius $R=0.1$ which has shown a good performance on the MNIST data-set.

From Table \ref{Table:long_table}, we osberve that increasing $\epsilon$ degrades the accuracy of CNN significantly on FGSM, and PGM model, while it increases the accuracy of C\&W model. The accuracy of DDN is independent of the parameter $\epsilon$, and remiains the same throughout the experiments. For $\epsilon=0$, FGSM and PGM return the unperturbed test data, and clearly the accuracy of CNN is higher than that of kernel methods (\%88 versus \%82). The fact that a deep neural network outperforms the kernel SVMs is natural since the size of the training data for training deep neural network is large $(60,000)$, and end-to-end classification methods outperform alternative machine learning paradigms. Although for $\epsilon>0$, the extracted features feeded to the downstream classifier is similarly perturbed in CNN and kernel SVMs, there is a significant performace degradation of the CNN model under PGM and FGSM attack models compared to that of the kernel SVM. This is due to the fact that these adversarial models are devised to particularly attack end-to-end DNNs and as such they are less effective on kernel SVMs.  From Table \ref{Table:long_table}, we also observe that our MKL training achieves a better accuracy on the perturbed test data-set for all range of $\epsilon\in \{0.1,\cdots,0.9\}$, thus validating our MKL approach.

\begin{center}
	\footnotesize{\centering
\hspace{-20mm}		\begin{table}[!t]
			\scriptsize{ \begin{tabular}{||c|| c| c| c|c|c|c|}
					\rowcolor{Gray}
					\hline
						
				\rowcolor{Gray}
				
				\hline
				CNN & $\epsilon$=0 & $\epsilon$=0.2 & $\epsilon$=0.4 &$\epsilon$=0.6 &$\epsilon$=0.8 &$\epsilon$=1  \\ 
				\hline\hline
				FGSM   					  &0.8875  		      &0.3872         &0.0789      &0.0434   			&0.0460                 		  &0.0575   \\ 
				\hline
				PGM-$\ell_{\infty}$  &0.8875			 &0.7579         &0.7579       &0.7579  		    &0.7579                 		   &0.7579     \\
				\hline
				DDN  					   &0.0344  		   &0.0344         &0.0344      &0.0344    			 &0.0344                 		   &0.0344      \\
				\hline
				C\&W 					  &0.6109			  &0.6204          &0.6297      &0.6395		         &0.6483                 			&0.8875    \\
				\hline
			    
				\hline
				\rowcolor{Gray}

				\textsc{AverageMKL} \cite{lauriola2020mklpy} & $\epsilon$=0 & $\epsilon$=0.2 & $\epsilon$=0.4 &$\epsilon$=0.6 &$\epsilon$=0.8 &$\epsilon$=1  \\
				\hline\hline
				FGSM   					  &0.8204  		  &0.8213         &0.8734         &0.8935	       &0.8969              &0.8996   \\ 
				\hline
				PGM-$\ell_{\infty}$  &0.8204		 &0.8163         &0.8163         &0.8163		   &0.8163               &0.8163    \\
				\hline
				DDN  					   &0.8096  	   &0.8096         &0.8096        &0.8096          &0.8096               &0.8096      \\
				\hline
				C\&W 					  &0.8171		   &0.8172          &0.8173         &0.8174           &0.8188               &0.8204    \\
				\hline
				
				\rowcolor{Gray}
				\textsc{EasyMKL} \cite{aiolli2015easymkl} & $\epsilon$=0 & $\epsilon$=0.2 & $\epsilon$=0.4 &$\epsilon$=0.6 &$\epsilon$=0.8 &$\epsilon$=1 \\
				\hline\hline
				FGSM   					  &0.8209  		  &0.8201         &0.8746         &0.8924         &0.8967                &0.8994   \\ 
				\hline
				PGM-$\ell_{\infty}$  &0.8209		 &0.8168         &0.8168         &0.8168          &0.8168                &0.8168     \\
				\hline
				DDN  					   &0.8106 		 	&0.8106         &0.8106         &0.8168          &0.8106                &0.8106      \\
				\hline
				C\&W 					  &0.8181			&0.8179          &0.8179         &0.8181           &0.8194                &0.8209    \\
				\hline
				
				\hline
				\rowcolor{Gray}
				
				\hline
				\textsc{GRAM} \cite{lauriola2017radius} & $\epsilon$=0 & $\epsilon$=0.2 & $\epsilon$=0.4 &$\epsilon$=0.6 &$\epsilon$=0.8 &$\epsilon$=1  \\
				\hline\hline
				FGSM   					  &0.8211   &0.8198         &0.8752       &0.8922          &0.8965                   &0.8994  \\ 
				\hline
				PGM-$\ell_{\infty}$  &0.8211   &0.8169         &0.8169       &0.8169            &0.8169                  &0.8169     \\
				\hline
				DDN  					   &0.8107   &0.8107         &0.8107       &0.8107             &0.8107                  &0.8107      \\
				\hline
				C\&W 					  &0.8182   &0.8181         &0.8180       &0.8181              &0.8196                  &0.8211    \\
				\hline
				
						\rowcolor{Gray}
				\hline
				\textsc{MEMO} \cite{lauriola2018minimum} & $\epsilon$=0 & $\epsilon$=0.2 & $\epsilon$=0.4 &$\epsilon$=0.6 &$\epsilon$=0.8 &$\epsilon$=1 \\
				\hline\hline
				FGSM   					  &0.8204  		  &0.8190         &0.8730      &0.8922       &0.8964                  &0.8993   \\ 
				\hline
				PGM-$\ell_{\infty}$  &0.8204		 &0.8165         &0.8165      &0.8165       &0.8165                   &0.8165     \\
				\hline
				DDN  					   &0.8101    		&0.8101         &0.8101       &0.8101        &0.8101                   &0.8101      \\
				\hline
				C\&W 					  &0.8175	    	&0.8174          &0.8177      &0.8178        &0.8188                   &0.8204    \\
				\hline

				\rowcolor{Gray}
				\hline
				\textsc{RMKL} \cite{do2009margin} & $\epsilon$=0 & $\epsilon$=0.2 & $\epsilon$=0.4 &$\epsilon$=0.6 &$\epsilon$=0.8 &$\epsilon$=1\\ 
				\hline\hline
				FGSM   					  &0.8198   &0.8180         &0.8722         &0.8920           &0.8961                  &0.8990   \\ 
				\hline
				PGM-$\ell_{\infty}$  &0.8198   &0.8160         &0.8160         &0.8160            &0.8160                  &0.8160    \\
				\hline
				DDN  					   &0.8095  &0.8095         &0.8095        &0.8095            &0.8095                 &0.8095     \\
				\hline
				C\&W 					  &0.8168   &0.8170          &0.8172         &0.8172             &0.8184                  &0.8198    \\
				\hline

				\rowcolor{Gray}
				\hline
				\textsc{FHeuristic} \cite{qiu2008framework} & $\epsilon$=0 & $\epsilon$=0.2 & $\epsilon$=0.4 &$\epsilon$=0.6 &$\epsilon$=0.8 &$\epsilon$=1 \\ [0.5ex] 
				\hline\hline
				FGSM   					  &0.8200  		   &0.8182         &0.8721         &0.8920          &0.8962               &0.8990   \\ 
				\hline
				PGM-$\ell_{\infty}$  &0.8200	      &0.8163         &0.8163         &0.8163           &0.8163                &0.8163     \\
				\hline
				DDN  					   &0.8097 			&0.8097         &0.8097        &0.8097           &0.8097               &0.8097      \\
				\hline
				C\&W 					  &0.8170			&0.8170          &0.8172         &0.8173           &0.8186               &0.8200    \\
				\hline
				
			    \rowcolor{Gray}
				\hline
				CKA \cite{cortes2010two}  & $\epsilon$=0 & $\epsilon$=0.2 & $\epsilon$=0.4 &$\epsilon$=0.6 &$\epsilon$=0.8 &$\epsilon$=1 \\ [0.5ex] 
				\hline\hline
				FGSM   					  &0.1743  		&0.1359         &0.1097        &0.1013                  &0.0988                &0.0987   \\ 
				\hline
				PGM-$\ell_{\infty}$  &0.1743	   &0.1626         &0.1626        &0.1626                 &0.1626                 &0.1626     \\
				\hline
				DDN  					   &0.1417      &0.1417         &0.1417         &0.1417                  &0.1417                  &0.1417      \\
				\hline
				C\&W 					  &0.1647		&0.1643         &0.1640        &0.1648                  &0.1653                 &0.1743    \\
				\hline
				
							    \rowcolor{Gray}
				\hline
				PWMK \cite{cortes2010two}  & $\epsilon$=0 & $\epsilon$=0.2 & $\epsilon$=0.4 &$\epsilon$=0.6 &$\epsilon$=0.8 &$\epsilon$=1 \\ [0.5ex] 
				\hline\hline
				FGSM   					  &0.8204 	 &0.8213         &0.8735         &0.8935            &0.8969                 &0.8996   \\ 
				\hline
				PGM-$\ell_{\infty}$  &0.8204	 &0.8163         &0.8163         &0.8163             &0.8163                 &0.8163     \\
				\hline
				DDN  					   &0.8096 	   &0.8096         &0.8096        &0.8096	         &0.8096                 &0.8096      \\
				\hline
				C\&W 					  &0.8171	  &0.8172          &0.8173         &0.8174	           &0.8188                  &0.8204    \\
				\hline

				\rowcolor{Gray}
				\hline
				Algorithm \ref{Algorithm:1} ($R$=0.1)  & $\epsilon$=0 & $\epsilon$=0.2 & $\epsilon$=0.4 &$\epsilon$=0.6 &$\epsilon$=0.8 &$\epsilon$=1 \\ [0.5ex] 
				\hline\hline
				FGSM   					  &0.8218$\pm 0.002$  		  &0.8267$\pm 0.005$         &0.8810$\pm 0.004$         &0.8947$\pm 0.001$	       &0.8976$\pm 0.004$                  &0.9020$\pm 0.008$   \\ 
				\hline
				PGM-$\ell_{\infty}$  &0.8218$\pm 0.003$			 &0.8179$\pm 0.004$         &0.8180$\pm 0.005$         &0.8180$\pm 0.002$	       &0.8180$\pm 0.003$                  &0.9020$\pm 0.005$     \\
				\hline
				DDN  					   &0.8125$\pm 0.010$  		    &0.8117$\pm 0.011$         &0.8124$\pm 0.015$         &0.8126$\pm 0.009$		     &0.8126$\pm 0.010$                   &0.9020$\pm 0.004$      \\
				\hline
				C\&W 					  &0.8191$\pm 0.006$		   &0.8184$\pm 0.010$         &0.8198$\pm 0.005$         &0.8201$\pm 0.007$			&0.8224$\pm 0.009$                   &0.9020$\pm 0.001$    \\
				\hline				
					
			\end{tabular}}\normalsize
			\caption{\footnotesize{Out-of-sample (test) accuracy  under FGSM  \cite{goodfellow2014explaining}, PGM  \cite{madry2017towards}, DDN  \cite{rony2019decoupling}, and C\&W \cite{carlini2017towards} attacks with different levels of adversarial robustness $\epsilon$. The test accuracy of Algorithm \ref{Algorithm:1} is averaged across 5 trials of the algorithm. For PGM-$\ell_{\infty}$, DDN, and C\&W we choose $T_{\mathrm{adv}}=30$. The test accuracies are computed using the one-versus-all rule (digit 0, versus all).}}
			\label{Table:long_table}
		\end{table}
	}
\end{center}

\section{Discussion and Concluding Remarks}
\label{Section:Conclusion}

In this paper, we have proposed a novel multiple kernel learning approach that is certifiably robust against adverserial examples. In particular, we proposed a distributionally robust optimization with respect to KL divergence. We characterized the distributionally robust optimization problem as a minimax vector optimization. We then proposed a biased stochastic primal-dual (SPD) algorithm. To debias the SPD, we used a Gumbel max perturbation technique to estimate the log-sum expression in the objective function. We provided theoretical performance guarantees for the proposed SPD algorithms. In particular, we established non-asymptotic consistency and asymptotic normality of the Monte Carlo sample average approximation associated with the empirical loss function. We also proved a min-max lower bound for the estimation problem associated with the weights of the mixture kernel model. Moreover, we proved a novel distributionally robust generalization bound based on the notion of the Rademacher and Gaussian complexities of a function class. We also derived upper bounds on the Rademacher and Gaussian complexity of function classes that are expressed in terms of random features which are tighter than previous known bounds in the MKL literature. To validate our proposed multiple kernel learning technique, we applied our method to synthetic and MNIST benchmark data-sets. Our numerical results shows that the learned kernels using our method is indeed robust to $\ell_{2}$ PGM perturbations compared to the standard multiple kernel learning approaches in the litearture. 

\appendix

\section{}

\textbf{Notation and definitions}. We denote the vectors by bold small letters, \textit{e.g.} $\bm{x}=(x_{1},\cdots,x_{n})\in \real^{n}$, and matrices by the bold capital letters, \textit{e.g.}, $\bm{M}=[M_{ij}]\in \real^{n\times m}$. The unit sphere in $n$-dimensions centered at the origin is denoted by $\mathrm{S}^{n-1}= \{\bm{x}\in \real^{n}:\sum_{i=1}^{n}x_{i}^{2}=1 \}$. We denote the $n$-by-$n$ identity matrix with $\bm{I}_{n}$, and the vector of all ones with $\bm{1}_{n}=(1,1,\cdots,1)\in \real^{n}$. For a symmetric matrix $\bm{M}=[M_{ij}]\in \real^{n\times n}$, let $\|\bm{M}\|_{2}=\sup_{\bm{x}\in \mathrm{S}^{n-1}}\langle \bm{x},\bm{M}\bm{x} \rangle$ and $\|\bm{M}\|_{F}=\sqrt{\text{Tr}(\bm{M}\bm{M}^{T})}=\sqrt{\sum_{i,j=1}^{n}|M_{ij}|^{2}}$ denotes the spectral and Frobenius norms, respectively. The eigenvalues of the matrix $\bm{M}$ are ordered and denoted by $\lambda_{1}(\bm{M})\geq \cdots \geq \lambda_{n}(\bm{M})$. We alternatively write $\lambda_{\text{min}}(\bm{M})=\lambda_{n}(\bm{M})$ and $\lambda_{\text{max}}(\bm{M})=\lambda_{1}(\bm{M})$ for the minimum and maximum eigenvalues of the matrix $\bm{M}$, respectively.

%

\begin{definition}\textsc{(Orlicz Norm)}
	\label{Def:Orlicz}
	The Young-Orlicz modulus is a convex non-decreasing function $\psi:\real_{+}\rightarrow \real_{+}$ such that $\psi(0)=0$ and $\psi(x)\rightarrow \infty$ when $x\rightarrow \infty$. Accordingly, the Orlicz norm of an integrable random variable $X$ with respect to the modulus $\psi$ is defined as
	\begin{align}
	\|X\|_{\psi}\df \inf \left\{\beta>0:\expect\left[\psi\left({|X|\over \beta}\right)\right]\leq 1\right\}.
	\end{align}
\end{definition}

In the sequel, we consider the Orlicz modulus $\psi_{\nu}(x)\df \exp(x^{\nu})-1$ . Accordingly, the cases of $\|\cdot\|_{\psi_{2}}$ and $\|\cdot\|_{\psi_{1}}$ norms are called the sub-Gaussian and the sub-exponential norms and have the following alternative definitions:

\begin{definition}\textsc{(Sub-Gaussian Norm)}
	The sub-Gaussian norm of a random variable $Z$, denoted by $\|Z\|_{\psi_{2}}$, is defined as
	\begin{align}
	\|Z\|_{\psi_{2}}= \sup_{q\geq 1} q^{-1/2}(\expect|Z|^{q})^{1/q}.
	\end{align}
	For a random vector $\bm{Z}\in \real^{n}$, its sub-Gaussian norm is defined as 
	\begin{align}
	\|\bm{Z}\|_{\psi_{2}}=\sup_{\bm{x}\in \mathrm{S}^{n-1}}\|\langle  \bm{x},\bm{Z}\rangle \|_{\psi_{2}}.
	\end{align}
\end{definition}

\begin{definition}\textsc{(Sub-exponential Norm)}
	\label{Definition:Alternative}
	The sub-exponential norm of a random variable $Z$, denoted by $\|Z\|_{\psi_{1}}$, is defined as follows
	\begin{align}
	\|Z\|_{\psi_{1}}=\sup_{q\geq 1} q^{-1}(\expect[|Z|^{q}])^{1/q}.
	\end{align}
	For a random vector $\bm{Z}\in \real^{n}$, its sub-exponential norm is defined as 
	\begin{align}
	\|\bm{Z}\|_{\psi_{1}}= \sup_{\bm{x}\in \mathrm{S}^{n-1}} \|\langle \bm{Z},\bm{x}\rangle\|_{\psi_{1}}.
	\end{align}
\end{definition}

We use asymptotic notations throughout the paper. We use the standard asymptotic notation for sequences. If $a_{n}$ and $b_{n}$ are positive
sequences, then $a_{n}=\mathcal{O}(b_{n})$ means that $\lim \sup_{n\rightarrow \infty} a_{n}/b_{n}< \infty$, whereas  $a_{n} = \Omega(b_{n})$ means that
$\lim \inf_{n\rightarrow \infty} a_{n}/b_{n} > 0$. Furthermore, $a_{n}=\widetilde{\mathcal{O}}(b_{n})$ implies $a_{n}=\mathcal{O}(b_{n}\text{poly}\log(b_{n}))$. Moreover $a_{n}=o(b_{n})$ means that $\lim_{n\rightarrow \infty}a_{n}/b_{n}=0$ and $a_{n}=\omega(b_{n})$ means that $\lim_{n\rightarrow \infty} a_{n}/b_{n}=\infty$. Lastly, we have $a_{n}=\Theta(b_{n})$ if $a_{n}=\mathcal{O}(b_{n})$ and $a_{n}=\Omega(b_{n})$. 

\subsection{Proof of Lemma \ref{Lemma:Saddle Point Characterization of Functional Optimization}}
\label{Appendix:Proof_of_Lemma_Saddle_Point_Characterization}

The proof is a small adaptation of the robust optimization methods of \cite{hu2013kullback} for the functional optimization problem in Eq. \eqref{Eq:equivalent_optimization}. We present the proof for completeness. First, consider the Lagrangian form of the functional optimization problem in Eq. \eqref{Eq:equivalent_optimization},
\begin{subequations}
	\label{Eq:Lagrangian_functional}
\begin{align}
&\max_{L\in \widehat{\mathbb{L}}(Q)} \expect_{Q^{\otimes 2}}\left[ Ly\tilde{y}\int_{\Xi}\varphi(\bm{x};\bm{\xi})\varphi(\tilde{\bm{x}};\bm{\xi})\mu_{\bm{\omega}}(\mathrm{d}\bm{\xi})-\lambda\left( L\log L- r/2 \right) \right]\\
&\text{s.t.}: \expect_{Q^{\otimes 2}}[L]=1,
\end{align}
\end{subequations}
where $\widehat{\mathbb{L}}(Q)\df \{L\in \mathbb{L}(Q):L\geq 0,\ \text{a.s.} \}$. Now, consider the following functionals 
\begin{subequations}
\begin{align}
F[L(\bm{x},y)]&\df \expect_{Q^{\otimes 2}}\left[Ly\tilde{y}\int_{\Xi}\varphi(\bm{x};\bm{\xi})\varphi(\tilde{\bm{x}};\bm{\xi})\mu_{\bm{\omega}}(\mathrm{d}\bm{\xi})-\lambda(L\log L-(r/2))\right],\\
G[L(\bm{x},y)]&\df \expect_{Q^{\otimes 2}}[L]-1.
\end{align}
\end{subequations}
The functionals $F(L)$ and $G(L)$ are convex and linear in $L$, respectively. Therefore, we can calculate the directional derivative of $F[L(\bm{x},\bm{y})]$ in the direction $U(\bm{x},y)$ as below
\begin{align}
\nonumber
D_{U}(F)[L]&\df \lim_{t\rightarrow 0} \dfrac{F[L(\bm{x},y)+tU(\bm{x},y)]-F[L(\bm{x},y)]}{t}\\ \nonumber
&=\lim_{t\rightarrow 0}\Bigg(\dfrac{\expect_{Q^{\otimes 2}}[(L+tU)y\tilde{y}\int_{\Xi}\varphi(\bm{x};\bm{\xi})\varphi(\tilde{\bm{x}};\bm{\xi})\mu_{\bm{\omega}}(\mathrm{d}\bm{\xi})-\alpha (L+tU)\log(L)]}{t}\\ \nonumber
&\hspace{4mm}-\dfrac{\expect_{Q^{\otimes 2}}[Ly\tilde{y}\int_{\Xi}\varphi(\bm{x};\bm{\xi})\varphi(\tilde{\bm{x}};\bm{\xi})\mu_{\bm{\omega}}(\mathrm{d}\bm{\xi})-\alpha L \log (L)]}{t} \Bigg)\\ \nonumber
&=\expect_{Q^{\otimes 2}}\left[Uy\tilde{y}\int_{\Xi}\varphi(\bm{x};\bm{\xi})\varphi(\tilde{\bm{x}};\bm{\xi})\mu_{\bm{\omega}}(\mathrm{d}\bm{\xi})\right]\\
&\hspace{4mm}-\lambda \lim_{t\rightarrow 0}\expect_{Q^{\otimes 2}}\left[\dfrac{(L+tU)\log(L+tU)-L\log L}{t} \right].
\end{align}
Note that the function $y\log y$ is convex in $y$ in $\real_{+}$. Therefore, for any $y$ and direction $u$, the function $[((y+tu)\log(y+tu)-y\log y)/t]$ is monotone in $t$. Therefore, by the monotone convergence theorem \cite{endou2008lebesgue}, we can interchange the order of the expectation and the limit. We then obtain
\begin{align}
\nonumber
D_{U}(F)[L]&=\expect_{Q^{\otimes 2}}\left[Uy\tilde{y}\int_{\Xi}\varphi(\bm{x};\bm{\xi})\varphi(\tilde{\bm{x}};\bm{\xi})\mu_{\bm{\omega}}(\mathrm{d}\bm{\xi})\right]\\ \nonumber
&\hspace{4mm}-\lambda \expect_{Q^{\otimes 2}}\left[\lim_{t\rightarrow 0} \dfrac{(L+tU)\log(L+tU)-L\log L }{t} \right]\\
&=\expect_{Q^{\otimes 2}}\left[Uy\tilde{y}\int_{\Xi}\varphi(\bm{x};\bm{\xi})\varphi(\tilde{\bm{x}};\bm{\xi})\mu_{\bm{\omega}}(\mathrm{d}\bm{\xi})\right]-\lambda \expect_{Q^{\otimes 2}}[(\log(L)+1)U].
\end{align}
Similarly, the directional derivative of $G[L]$ is given by
\begin{align}
D_{U}(G)[L]=\expect_{Q^{\otimes 2}}[U].
\end{align}
We now consider the Lagrangian form of the functional optimization problem in Eq. \eqref{Eq:Lagrangian_functional} if  
\begin{align}
\nonumber
\mathcal{L}[L](\lambda)=&\expect_{Q^{\otimes 2}}\left[Ly\tilde{y}\int_{\Xi}\varphi(\bm{x};\bm{\xi})\varphi(\tilde{\bm{x}};\bm{\xi})\mu_{\bm{\omega}}(\mathrm{d}\bm{\xi})-\lambda(L\log(L)-(r/2))\right]\\
&+\alpha (\expect_{Q^{\otimes 2}}[L]-1).
\end{align} 
Due to  the result of Bonnas and Shapiro \cite[Proposition 3.3.]{bonnans2013perturbation}, $L^{\ast}$ is an optimal solution of Problem \eqref{Eq:Lagrangian_functional} if $L^{\ast}\in \widehat{\mathbb{L}}(Q)$, $\expect_{Q^{\otimes 2}}[L^{\ast}]=1$, and 
\begin{align}
L^{\ast}\in \mathrm{arg}\max_{L\in \widehat{\mathbb{L}}(Q)} \ell(L,\lambda^{\ast}).
\end{align}
This is an unconstrained optimization problem whose directional derivative is given by
\begin{align}
D_{U}(\mathcal{L})[L](\lambda)=\expect_{Q^{\otimes 2}}\left[Uy\tilde{y}\int_{\Xi}\varphi(\bm{x};\bm{\xi})\varphi(\tilde{\bm{x}};\bm{\xi})\mu_{\bm{\omega}}(\mathrm{d}\bm{\xi})-\lambda(\log(L)+1)U+\alpha U\right].
\end{align}
Therefore, the optimal solution has the following form
\begin{align}
\label{Eq:Plugging_into}
L^{\ast}(\alpha)=\exp\left({{\alpha-\lambda\over \lambda}}\right)\cdot \exp\left({{1\over \lambda} y\tilde{y}\int_{\Xi}\varphi(\bm{x};\bm{\xi})\varphi(\tilde{\bm{x}};\bm{\xi})\mu_{\bm{\omega}}(\mathrm{d}\bm{\xi}) }\right).
\end{align}
Now, suppose that the Lagrange multiplier belongs to the following set 
\begin{align}
\lambda \in \Lambda\df \left\{s\in \real: s>0,\expect_{Q^{\otimes 2}}\Bigg[\exp\Bigg(\dfrac{y\tilde{y}}{{s}}\int_{\Xi}\varphi(\bm{x};\bm{\xi})\varphi(\tilde{\bm{x}};\bm{\xi})\mu_{\bm{\omega}}(\mathrm{d}\bm{\xi})\Bigg)\Bigg]<+\infty\right\}.
\end{align}
Then, by choosing the Lagrange multiplier 
\begin{align}
\alpha^{\ast}=-\lambda \log \expect_{Q^{\otimes 2}}\Bigg[\exp\Bigg(\dfrac{y\tilde{y}}{{\lambda}}\int_{\Xi}\varphi(\bm{x};\bm{\xi})\varphi(\tilde{\bm{x}};\bm{\xi})\mu_{\bm{\omega}}(\mathrm{d}\bm{\xi})\Bigg)\Bigg]+\lambda,
\end{align} 
the feasibility constraint $\expect_{Q^{\otimes 2}}[L^{\ast}]=1$ is satisfied. Plugging $\alpha^{\ast}$ into Eq. \eqref{Eq:Plugging_into} yields
\begin{align}
L^{\ast}=L^{\ast}(\lambda^{\ast})=\dfrac{\exp\left({1\over \lambda}y\tilde{y}\int_{\Xi}\varphi(\bm{x};\bm{\xi})\varphi(\tilde{\bm{x}};\bm{\xi})\mu_{\bm{\omega}}(\mathrm{d}\bm{\xi}) \right)}{\expect_{Q^{\otimes 2}}[\exp\left({1\over \lambda}y\tilde{y}\int_{\Xi}\varphi(\bm{x};\bm{\xi})\varphi(\tilde{\bm{x}};\bm{\xi})\mu_{\bm{\omega}}(\mathrm{d}\bm{\xi})\right)]}.
\end{align}
The pair $(\lambda^{\ast},L^{\ast})$ satisfies Bonnas and Shapiro optimality criterion. Now, plugging $L^{\ast}$ into yields the following value function
\begin{align}
\arg \max_{L\in \widehat{\mathbb{L}}^{Q}: \expect_{Q^{\otimes 2}}[L]=1} F[L(\bm{x},y)]=\lambda \log \expect_{Q^{\otimes 2}}\left[\exp\left(y\tilde{y}\int_{\Xi}\varphi(\bm{x};\bm{\xi})\varphi(\tilde{\bm{x}};\bm{\xi})\mu_{\bm{\omega}}(\mathrm{d}\bm{\xi})\right)\right]+{\lambda r\over 2}. 
\end{align}
To complete the proof, we note that due Assumption \ref{Assumption:1}, we have that for any $s>0$,
\begin{align}
\nonumber
\expect_{Q^{\otimes 2}}\left[\exp\Bigg(\dfrac{y\tilde{y}}{{s}}\int_{\Xi}\varphi(\bm{x};\bm{\xi})\varphi(\tilde{\bm{x}};\bm{\xi})\mu_{\bm{\omega}}(\mathrm{d}\bm{\xi})\Bigg)\right]&=\expect_{Q^{\otimes 2}}\left[\exp\Bigg(\dfrac{y\tilde{y}}{{ms}}\sum_{i=1}^{m}\omega_{i}K_{i}(\bm{x};\tilde{\bm{x}})\Bigg)\right]\\ \nonumber
&\leq \expect_{Q^{\otimes 2}}\left[\exp\Bigg(\Bigg|\dfrac{y\tilde{y}}{{ms}}\sum_{i=1}^{m}\omega_{i}K_{i}(\bm{x};\tilde{\bm{x}})\Bigg|\Bigg)\right]\\
&\leq \exp(B/s).
\end{align} 
Thus, as long as $s>0$, we have that  
\begin{align}
\expect_{Q^{\otimes 2}}\left[\exp\Bigg(\dfrac{y\tilde{y}}{{s}}\int_{\Xi}\varphi(\bm{x};\bm{\xi})\varphi(\tilde{\bm{x}};\bm{\xi})\mu_{\bm{\omega}}(\mathrm{d}\bm{\xi})\Bigg)\right]<+\infty.
\end{align}
Hence, $\Lambda=\real_{+}\backslash\{0\}$.

\subsection{Proof of Lemma \ref{Lemma:Asymptotic_Normality}}
\label{Appendix:Asymptotic_Normality}

The proof is a minor adaptation of \cite[Thm. 2]{fan2018dnn}. We present the proof for completeness.  We first establish the asymptotic normality of the following $U$-statistic
\begin{align}
\label{Eq:}
U_{\lambda}(\psi_{N},\bm{z}_{1:n})\df \dfrac{2}{n(n-1)}\sum_{1\leq i<j\leq n}\exp\left(\dfrac{1}{\lambda}\psi_{N}(\bm{z}_{i},\bm{z}_{j})\right).
\end{align}
Let $\Psi_{N}(\bm{z}_{i},\bm{z}_{j})\df \exp\left(\dfrac{1}{\lambda}\psi_{N}(\bm{z}_{i},\bm{z}_{j})\right)$. Consider the following projection functions:
\begin{align}
\Psi_{N}^{1}(\bm{z}_{1})&=\expect\left[\Psi_{N}(\bm{z}_{1},\bm{z}_{2})|\bm{z}_{1}\right],\quad \widetilde{\Psi}_{N}^{1}({\bm{z}}_{1})=\Psi_{N}^{1}(\bm{z}_{1})-\expect[\Psi_{1}^{N}]\\
\Psi_{N}^{2}(\bm{z}_{1},\bm{z}_{2})&=\expect\left[\Psi_{N}(\bm{z}_{1},\bm{z}_{2})|\bm{z}_{1},\bm{z}_{2}\right], \quad \widetilde{\Psi}^{2}_{N}(\bm{z}_{1},\bm{z}_{2})=\Psi^{2}_{N}(\bm{z}_{1},\bm{z}_{2})-\expect[\Psi_{2}^{N}].
\end{align}
The Hoeffding's canonical terms are defined as follows
\begin{align}
g_{N}^{1}(\bm{z}_{1})&=\widetilde{\Psi}_{N}^{1}({\bm{z}}_{1})\\
g_{N}^{2}(\bm{z}_{1},\bm{z}_{2})&=\widetilde{\Psi}_{N}^{2}({\bm{z}}_{1},\bm{z}_{2})-g_{N}^{1}(\bm{z}_{1})-g_{N}^{2}(\bm{z}_{2}).
\end{align}
The kernel $\Psi$ can be written as the sum of canonical terms
\begin{align}
\label{Eq:Canonical}
\Psi_{N}(\bm{z}_{1},\bm{z}_{2})-\expect [\Psi_{N}(\bm{z}_{1},\bm{z}_{2})]=g_{N}^{1}(\bm{z}_{1})+g_{N}^{1}(\bm{z}_{2})+g_{N}^{2}(\bm{z}_{1},\bm{z}_{2}).
\end{align}
All the canonical terms in Equation \eqref{Eq:Canonical} are un-correlated. Hence, we
have:
\begin{align}
\mathrm{Var}[\Psi_{N}(\bm{z}_{1},\bm{z}_{2})]=2\expect[g_{N}^{1}(\bm{z}_{1})]+\expect[g_{N}^{2}(\bm{z}_{1},\bm{z}_{2})].
\end{align}
The $U$-statistics can also be written as follows
\begin{align}
\nonumber
U_{\lambda}(\psi_{N},\bm{z}_{1:n})-\expect[U_{\lambda}]&=\dfrac{2}{n(n-2)}\sum_{1\leq i<j\leq n}(\widetilde{\Psi}_{N}(\bm{z}_{1},\bm{z}_{2})-\expect\Psi_{N})\\ \label{Eq:leading_term}
&=\dfrac{2}{n(n-1)}\left((n-1)\sum_{i=1}^{n}g_{N}^{1}(\bm{z}_{i})+\sum_{1\leq i<j\leq n}g_{N}^{2}(\bm{z}_{i},\bm{z}_{j})\right).
\end{align}
We define the leading term of Eq. \eqref{Eq:leading_term} as the H\'{a}jek projection
\begin{align}
\widehat{U}_{\lambda}(\psi_{N},\bm{z}_{1:n})\df \dfrac{2}{n}\sum_{i=1}^{n}g_{N}^{1}(\bm{z}_{i}).
\end{align}
Note that $\mathrm{Var}(\widehat{U}_{\lambda}(\psi_{N},\bm{z}_{1:n}))={4\over n}\mathrm{Var}(\Psi_{N}(\bm{z}_{1}))$. Since the H\'{a}jek projection is the average of the independent and identically distributed terms, by the Lindeberg-L\`{e}vy Central Limit Theorem (CLT), we obtain
\begin{align}
\dfrac{\sqrt{n}\widehat{U}_{\lambda}(\psi_{N},\bm{z}_{1:n})}{2\mathrm{Var}^{1\over 2}(\Psi_{N}(\bm{z}_{1}))}\stackrel{d}{\leadsto}\mathsf{N}(0,1).
\end{align}
In the sequel, we prove that the reminder term converges to zero in probability, \textit{i.e.},
\begin{align}
\lim\sup_{n\rightarrow \infty}\prob\left(\left|\dfrac{U_{\lambda}-\expect[U_{\lambda}] -\widehat{U}_{\lambda}}{4\mathrm{Var}(\Psi_{N}(\bm{z}_{1}))}\right|\geq \delta\right)=0,\quad \forall\delta >0.
\end{align}
Notice that
\begin{align*}
\expect\left[(U_{\lambda}-\expect[U_{\lambda}]-\widehat{U}_{\lambda})^{2}\right]&=\dfrac{2}{n(n-1)}\expect[(g^{2}_{N}(\bm{z}_{1},\bm{z}_{2}))^{2}]\\
&\leq \dfrac{4}{n^{2}} \expect[(g^{2}_{N}(\bm{z}_{1},\bm{z}_{2}))^{2}]\\
&\leq \dfrac{4}{n^{2}}\mathrm{Var}[\Psi_{N}].
\end{align*}
Now, $\mathrm{Var}[\Psi_{N}]$ is bounded from above for $\lambda>0$, since $\psi_{N}(\bm{z}_{i},\bm{z}_{j})={1\over N}\langle y_{i}\varphi(\bm{x}_{i};\bm{\xi}) ,y_{j}\varphi(\bm{x}_{j};\bm{\xi}) \rangle$ is bounded. Therefore,
\begin{align}
\dfrac{\expect\left[(U_{\lambda}-\expect[U_{\lambda}]-\widehat{U}_{\lambda})^{2}\right]}{\mathrm{Var}[\Psi_{N}]}\leq \dfrac{{4\over n^{2}\mathrm{Var}[\Psi_{N}]}}{{4\over n}\mathrm{Var}[\Psi_{N}]}=\dfrac{1}{n}\rightarrow 0,\quad \mathrm{as}\ n\rightarrow \infty.
\end{align}
Applying the standard $\delta$-method to $\lambda\log U_{\lambda}$ completes the proof.

\subsection{Proof of Theorem \ref{Theorem:Stochastic_Equicontinuity}}
\label{Appendix:Proof_of_Stochastic_Equicontinuity}

We leverage the methods of the empircal process theory \cite{pollard1990empirical} to establish the stochastic equicontinuity of the underlying $U$-statistics. Let $(\mathcal{F},\|\cdot\|)$ be a subset of a normed space of real functions $f:\mathcal{X}\rightarrow \real$ on some
set $\mathcal{X}$. Typically, the underlying normed space is the $L_{r}(P)$-space associated with the probability measures $P\in \mathcal{M}(\mathcal{X})$, equipped with the following norms:
\begin{subequations}
\begin{align}
	\|f\|_{L_{r}(P)}&\df \left(\int_{\mathcal{X}}f^{r}(x)P(\mathrm{d}x)\right)^{1\over r},\\
  \|f\|_{L_{\infty}(P)}&\df \sup_{x\in \mathrm{supp}(P)}|f(x)|.
\end{align}
\end{subequations}
For a function $f\in \mathcal{F}$ from the function class $\mathcal{F}$, we define 
\begin{subequations}
\begin{align}
\mathbb{P}f &\df \expect_{P}[f(X)], \\
\mathbb{P}_{n}f &\df \dfrac{1}{n}\sum_{i=1}^{n}f(X_{i}), \quad X_{1},\cdots,X_{n}\sim_{\text{i.i.d.}} P.
\end{align}
\end{subequations}
Furthermore, $\mathbb{G}_{n}f \df \sqrt{n}(\mathbb{P}_{n}-\mathbb{P})f$. Then $(\mathbb{G}_{n}f)_{f\in \mathcal{F}}$ is an empirical process indexed by $\mathcal{F}$ endowed with the norm $\|\mathbb{G}_{n}f \|_{\mathcal{F}}=\sup_{f\in \mathcal{F}}|\mathbb{G}_{n}f|$.

\begin{definition}\textsc{($\varepsilon$-Bracket)} Given two functions $l(\cdot)$ and $u(\cdot)$, the bracket $[l,u]$ is the set of functions $f\in \mathcal{F}$ with $l(x)\leq f(x)\leq u(x)$ for all $x\in \mathcal{X}$. An $\epsilon$-bracket is a bracket with $\|l-u\|_{L_{r}(P)}\leq \epsilon$.
\end{definition}

\begin{definition}\textsc{(Bracketing, Covering, and Packing Numbers)}  The \textit{bracketing} number, denoted by $\mathcal{N}_{[]}(\varepsilon, \mathcal{F},L_{r}(P))$, is the minimum number of $\varepsilon$-brackets required to cover $\mathcal{F}$. The \textit{covering} number $\mathcal{N}(\varepsilon,\mathcal{F},L_{r}(P))$ is the minimum number of balls $\ball^r_{\varepsilon}(f_{i})\df \{f\in \mathcal{F}:\|f-f_{i}\|_{L_{r}(P)}\leq \varepsilon\}$ of radius $\varepsilon$, $f_{i}\in \mathcal{F}$ needed to cover the set $\mathcal{F}$. The \textit{packing} number $\mathcal{M}(\varepsilon,\mathcal{F},L_{r}(P))$ is the largest number such that there exist functions $f_{1},\cdots,f_{\mathcal{M}}\in \mathcal{F}$ satisfying $\|f_{i}-f_{j}\|_{L_{r}(P)}\geq \varepsilon, i\not=j$.
\end{definition}

\begin{remark}
It is well-known that for an arbitrary function class $\mathcal{F}$, the $\epsilon$-bracketing number is an upper bound for the $\epsilon$-covering number, \textit{i.e.},
\begin{align}
\label{Eq:Inequality_between_bracketting}
\mathcal{N}(\varepsilon,\mathcal{F},L_{r}(P))\leq \mathcal{N}_{[]}(2\varepsilon,\mathcal{F},L_{r}(P)).
\end{align} 
Indeed, if $f\in \mathcal{F}$ is the $2\epsilon$-bracket $[l,u]$, then it is also in the metric ball $\ball_{r}^{\epsilon}((l+u)/2)\df \{f\in \mathcal{F}:\|f-(l+u)/2\|_{L_{r}(P)}\leq \epsilon\}$. In addition, the following inequalities hold between the covering and packing numbers
\begin{align}
\mathcal{M}(2\varepsilon,\mathcal{F},L_{r}(P)) \leq \mathcal{N}(\varepsilon,\mathcal{F},L_{r}(P))\leq 	\mathcal{M}(\varepsilon,\mathcal{F},L_{r}(P)).
\end{align}
\end{remark}

The following result, due to Khosravi, \textit{et al.} \cite{{khosravi2019non}} characterizes a concentration result for the $U$-statistics using the empirical process theory:

\begin{lemma}\textsc{(Stocahstic Equicontinuity with the Bracketting Integral, \cite[Lemma 15]{khosravi2019non})}
\label{Lemma:Function_Class}
	Consider a function space $\mathcal{F}$ of symmetric functions from some data space $\mathcal{Z}^{s}$ to $\real$, and consider a $U$-statistc of the order $s$, with the kernel $f$ over $n$ samples:
	 \begin{align}
	 	\label{Eq:U_stat_order_s}
	 	U_{s}(f,\bm{z}_{1:n})=\dfrac{1}{{n \choose s}}\sum_{1\leq i_{1}\leq i_{2}\leq\cdots\leq i_{s}\leq n}f(\bm{z}_{i_{1}},\cdots,\bm{z}_{i_{s}}).
	 \end{align}
Suppose $\sup_{f\in \mathcal{F}}\|f\|_{L_{2}(P)}\leq \eta$, $\sup_{f\in \mathcal{F}}\|f\|_{L_{\infty}(P)}\leq G$, and let $\kappa_{0}=n/s$. Then for $\kappa_{0}\geq  {G^{2}\over \log N_{[]}(1/2,\mathcal{F},\|\cdot\|_{P,2})}$ with the probability of at least $1-\delta$ we  have
\begin{align}
\label{Eq:Khosravi}
&\sup_{f\in \mathcal{F}}\big|U_{s}(f,\bm{Z}_{1:n})-\expect[f(\bm{Z}_{1:n})]\big|\\ \nonumber
&=\mathcal{O}\left(\inf_{\rho>0}\dfrac{1}{\sqrt{\kappa_{0}}}\int_{\rho}^{2\eta}\sqrt{\log(\mathcal{N}_{[]}(\epsilon,\mathcal{F},L_{2}(P)))}\mathrm{d}\epsilon+\eta \sqrt{\dfrac{\log(1/\delta)+\log\log(\eta/\rho)}{\kappa}}+\rho \right).
\end{align}
\hfill $\square$
\end{lemma}
To obtain a bound based on Lemma \ref{Lemma:Function_Class}, the bracketing number $\mathcal{N}_{[]}(\epsilon,\mathcal{G}_{mN},L_{2}(P))$ as well as the constants $\eta$ and $G$ must be calculated, where $\mathcal{G}_{mN}$ is the function class defined in Eq. \eqref{Eq:GmN}. Nevertheless, the computation of the $\epsilon$-bracketing number $\mathcal{N}_{[]}(\epsilon,\mathcal{F},L_{2}(P))$ is not readily amenable to the underlying composite class $\mathcal{G}_{mN}$ in Theorem \ref{Theorem:Stochastic_Equicontinuity}. To alleviate this issue, we establish the following alternative bound based on the Dudley's metric entropy integral. The proof is  postponed to Appendix \ref{Appendix:Metric_Entropy_Integral}:
\begin{lemma}\textsc{(Stocahstic Equicontinuity with the Metric Entropy Integral)}
	\label{Lemma:Function_Class_1}
	Consider a function space $\mathcal{F}$ of symmetric functions from some data space $\mathcal{Z}^{s}$ to $\real$, and consider a $U$-statistc of the order $s$, with the kernel $f$ over $n$ samples defined in Equation \ref{Eq:U_stat_order_s} of Lemma \ref{Lemma:Function_Class}.  Suppose  and let $\kappa=\lceil n/s\rceil=\lceil \kappa_{0} \rceil$. Furthermore, suppose
	\begin{align}
		\delta\df \sup_{f,g\in \mathcal{F}}\|f-g\|_{L_{2}(P)}<+\infty.
	\end{align}	
	Then, with the probability of at least $1-\varrho$ we have
	\begin{align}	
		\label{Eq:Masoud_1}
\sup_{f\in \mathcal{F}}\big|U_{s}(f,\bm{Z}_{1:n})-\expect[f(\bm{Z}_{1:n})]\big|\leq	\dfrac{1}{\varrho} \dfrac{4}{\sqrt{\kappa}}\int_{0}^{\delta/2}\sqrt{\log \mathcal{N}(\epsilon,\mathcal{F},L_{2}(P))}\mathrm{d}\epsilon.
	\end{align}
\end{lemma}

The proof of Lemma \ref{Lemma:Function_Class_1} relies on the standard chaining argument similar to \cite[Lemma 15]{khosravi2019non}. However, the proof is different from \cite[Lemma 15]{khosravi2019non} in that we leverage a symmetrization approach to attain the metric entropy integral term. In addition,  in light of Inequality \eqref{Eq:Inequality_between_bracketting} in Remark \ref{Remark:2}, the bound in Eq. \eqref{Eq:Masoud_1} of Lemma \ref{Lemma:Function_Class_1} can be \textit{potentially} tighter than that of Eq. \eqref{Eq:Khosravi} in Lemma \ref{Lemma:Function_Class}.

\begin{definition}\textsc{(Growth function, VC dimension, Shattering,\cite{vapnik2015uniform})} Let $\mathcal{F}$ denote a class of functions from $\mathcal{X}$ to $\{0, 1\}$ (the hypotheses, or the classification rules). For any non-negative integer $m$, we define the growth function of $\mathcal{H}$ as follows
	\begin{align}
		\Pi_{\mathcal{F}}(n) \df \max_{x_{1},\cdots,x_{n}\in \mathcal{X}} |{(f(x_{1}),\cdots,f(x_{m})): f \in \mathcal{F}}|.
	\end{align}
	If $|{(f(x_{1}),\cdots,f(x_{n})): f \in \mathcal{F}}|= 2^n$, we say $\mathcal{F}$ shatters the set $\{x_{1},\cdots,x_{n}\}$. The Vapnik-Chervonenkis dimension of $\mathcal{F}$, denoted by $\mathrm{dim}_{\mathrm{VC}}(\mathcal{F})$, is the size of the largest shattered set, \textit{i.e.}, the largest $n$ such that $\Pi_{\mathcal{F}}(n)=2^{n}$. If there is no largest $n$, we define $\mathrm{dim}_{\mathrm{VC}}(\mathcal{F}) = \infty$.
\end{definition}

We also define the following generalization of VC dimension: 
\begin{definition}\textsc{(Pollard's Pseudo-dimension, \cite{pollard1990empirical})} 
	Let $\mathcal{F}$ denote a class of real-valued functions mapping $\mathcal{X}$ to $\real_{+}$. Consider the function class $\mathcal{G}$ defined as follows
	\begin{align}
		\mathcal{G}\df \{g:\mathcal{X}\times \real\rightarrow \{0,1\}:x\mapsto \mathbb{I}_{\{f(x)\geq c\}}, f\in \mathcal{F},c\in \real\}.
	\end{align}
	Then, the Pollard's pseudo-dimension is defined as follows
	\begin{align}
		\mathrm{dim}_{P}(\mathcal{F})\df \mathrm{dim}_{\mathrm{VC}}(\mathcal{G}).
	\end{align}
\end{definition}

\begin{remark}
\label{Remark:1}
For any function class $\mathcal{F}$, let $F(\bm{x})\df \sup_{f\in \mathcal{F}}|f(\bm{x})|$ denotes the envelop function. Suppose $\|F\|_{L_{r}(P)}\leq \infty$. Then, there is a universal constant $K$ such that for any $\epsilon>0$ and $r\geq 1$, the following inequality holds
\begin{align}
\label{Eq:Dudley}
\sup_{P\in \mathcal{M}(\mathcal{X})}\mathcal{N}(\epsilon\|F\|_{L_{r}(P)},\mathcal{F},L_{r}(P))\leq \left(\dfrac{K\log (K/\epsilon^{r})}{\epsilon^{r}} \right)^{\mathrm{dim}_{\mathrm{VC}}(\mathcal{F})}\leq \left({K'\over \epsilon}\right)^{r(\mathrm{dim}_{\mathrm{VC}}(\mathcal{F}) -1)+\delta},
\end{align}
for all  $\delta>0$, where $K=3e^{2}/(e-1)\approx 12.9008$.  Alternatively,
\begin{align}
\label{Eq:Hussler}
\sup_{P\in \mathcal{M}(\mathcal{X})}\mathcal{N}(\epsilon\|F\|_{L_{r}(P)},\mathcal{F},L_{r}(P))\leq \tilde{K}\mathrm{dim}_{\mathrm{VC}}(\mathcal{F}) (4e)^{\mathrm{dim}_{\mathrm{VC}}(\mathcal{F})}\left({1\over \epsilon} \right)^{r (\mathrm{dim}_{\mathrm{VC}}(\mathcal{F}) -1)},
\end{align}
for some universal constant $\tilde{K}$. Inequality \ref{Eq:Dudley} is due to Dudley \cite{dudley1978central}, and Inequality \ref{Eq:Hussler} is due to Haussler \cite{haussler1995sphere}.
\end{remark}

\begin{remark}
\label{Remark:2}
For any function class $\mathcal{F}$, clearly we have $\mathrm{dim}_{\mathrm{VC}}(\mathcal{F}) \leq \mathrm{dim}_{P}(\mathcal{F})$. 
\end{remark}

The following lemma is proved in Appendix \ref{Appendix:Proof_of_Lemma_Bartlett}:
\begin{lemma}\textsc{($\epsilon$-Covering Number of the Composite Class)} 
\label{Lemma:Bartlett}	
Consider the function class $\mathcal{F}$ of real valued functions on $\mathcal{X}$ , and let $\psi:\real\rightarrow \real_{+}$ denotes a $L_{\psi}$-Lipschitz function satisfying $|\psi(s)-\psi(\tilde{s})|\leq L_{\psi}|s-\tilde{s}|, \quad \forall s\in \mathcal{S}$ where $\mathcal{S}\subset \real$ is a compact domain. Let 
\begin{align}
\nonumber
\mathcal{F}'\df \Bigg\{g:(\mathcal{X}\times \mathcal{Y})^{2}\rightarrow \real&:(\bm{x},y),(\tilde{\bm{x}},\tilde{y})\mapsto \psi\left(y\tilde{y}{1\over N}\sum_{k=1}^{N}\sum_{i=1}^{m}\omega_{i} f_{i}^{k}(\bm{x})f^{k}_{i}(\tilde{\bm{x}})\right)
\\ & : f_{i}^{k}\in \mathcal{F}, \forall i\in [m],\forall k\in [N],\forall\bm{\omega}\in \mathrm{S}^{+}_{m}\Bigg\}, 
\end{align}
denotes the composite class on $\mathcal{X}\times \mathcal{Y}$ equipped with the measure $P\in \mathcal{M}(\mathcal{X}\times \mathcal{Y})$ with the marginal $Q\in \mathcal{M}(\mathcal{X})$. Then, the $\epsilon$-covering entropy of the composite class $\mathcal{F}'$ is bounded from above as follows
\begin{align}
\label{Lemma:Inequality}
\log \mathcal{N}(\epsilon,\mathcal{F}',L_{r}(P))\leq Nm\log \mathcal{N}\left(\dfrac{(\epsilon)^{1\over r}}{2mL_{\psi}\|F\|_{L_{2r}(Q)}},\mathcal{F},L_{2r}(Q)\right),
\end{align}
where $F$ is the envelop function associated with the function class $\mathcal{F}$. \hfill $\square$
\end{lemma}

Now, let $\psi(\xi)=\exp({1\over \lambda}\xi)$ for $\lambda>0$. Due to Assumption \ref{Assumption:1}, we have that 
\begin{align*}
\dfrac{1}{N}\left|\langle y\bm{\varphi}_{\bm{\omega}}(\bm{x}), \tilde{y}\bm{\varphi}_{\bm{\omega}}(\tilde{\bm{x}}) \rangle\right|&=\dfrac{1}{N}\left|\sum_{i=1}^{m}\sum_{k=1}^{N}\omega_{i}\varphi(\bm{x};\bm{\xi}_{i}^{k})\varphi(\tilde{\bm{x}};\bm{\xi}_{i}^{k})  \right| \\
&\leq L^{2}, \quad \forall (\bm{x},y),(\tilde{\bm{x}},\tilde{y})\in \mathcal{X}\times \mathcal{Y}.
\end{align*}
Thus, the function $\psi(\xi)$ with $\xi={1\over N}\langle y\bm{\varphi}_{\bm{\omega}}(\bm{x}), \tilde{y}\bm{\varphi}_{\bm{\omega}}(\tilde{\bm{x}}) \rangle$ is Lipschitz with the constant $L_{\psi}=\exp(L^{2}/\lambda)/\lambda$.
Therefore, using Inequality \ref{Lemma:Inequality} of Lemma \ref{Lemma:Bartlett}, we obtain
\begin{align}
\label{Eq:Simple_Bound}
\mathcal{N}(\epsilon,\mathcal{G}_{N,m},L_{r}(P))\leq \left(\mathcal{N}\left({\lambda (\epsilon)^{1\over r}\over 2m\exp(L^{2}/\lambda) \|F\|_{L_{2r}(Q)}},\mathcal{F},L_{2r}(Q)\right)\right)^{Nm},
\end{align}
where $\mathcal{F}$ is the function class defined in Theorem \ref{Theorem:Stochastic_Equicontinuity}. To compute a bound on the covering number $\mathcal{N}\left(\epsilon',\mathcal{F},L_{2r}(Q)\right)$, we leverage the fact that $\varphi(\bm{x};\bm{\xi}_{i}^{k})=\varphi_{0}(\langle\bm{x},\bm{\xi}_{i}^{k}\rangle+b_{i}^{k})$, where $b_{i}^{k}\in  \real$ and $\bm{\xi}_{i}^{k}\in \real^{d}$ is bounded from above by $L$ (cf. Assumption \ref{Assumption:1}). Therefore, the constant function $F(\bm{x})=L$ for all $\bm{x}\in \mathcal{X}$ is an envelop for the function class $\mathcal{F}$ with $\|F\|_{L_{2r}(Q)}=L$. We can then simplify the upper bound in Eq. \eqref{Eq:Simple_Bound} as follows
\begin{align}
\mathcal{N}(\epsilon,\mathcal{G}_{N,m},L_{r}(P))\leq \left(\mathcal{N}\left(\epsilon',\mathcal{F},L_{2r}(Q)\right)\right)^{Nm},
\end{align} 
where $\epsilon'\df {\lambda (\epsilon)^{1\over r}\over 2m\exp(L^{2}/\lambda)L}$. Moreover, since $x\mapsto\varphi_{0}(x)$ is 1-Lipschitz, we have the following inequality
\begin{align}
\mathcal{N}\left(\epsilon',\mathcal{F},L_{2r}(Q)\right)\leq \mathcal{N}\left(\epsilon',\mathcal{\mathcal{L}},L_{2r}(Q)\right).
\end{align}
where $\mathcal{L}$ is the class of linear functions
\begin{align}
\mathcal{L}\df \{ f:\real^{d}\rightarrow \real:\bm{x}\mapsto\langle\bm{x},\bm{\xi}\rangle+b: \bm{\xi}\in \real^{d},b\in \real \}.
\end{align}
To compute a bound on $\mathcal{N}\left(\epsilon',\mathcal{L}, L_{2}(P)\right)$, we leverage Remarks \ref{Remark:1} to obtain
\begin{align}
\mathcal{N}(\epsilon',\mathcal{L},L_{2}(P))\leq \left(\dfrac{K\log (K/\epsilon'^{2})}{\epsilon'^{2}} \right)^{\mathrm{dim}_{\mathrm{VC}}(\mathcal{F}')}.
\end{align}
Moreover, $\mathrm{dim}_{\mathrm{VC}}(\mathcal{L})=\mathrm{dim}_{\mathrm{VC}}(\mathrm{sgn}(\mathcal{L}))$, where $\mathrm{sgn}(\mathcal{L})=\{ \mathrm{sgn}(f):f\in \mathcal{L}\}$ and $\mathrm{sgn}(\cdot)$ is the sign function. Therefore, we can alternatively focus on the following function class
\begin{align}
\mathrm{sgn}(\mathcal{L})= \{ f:\real^{d}\rightarrow \real:\bm{x}\mapsto[\langle\bm{x},\bm{\xi}\rangle+b]_{+}: \bm{\xi}\in \real^{d},b\in \real \},
\end{align}
where $[u]_{+}\df \max (u, 0)$. 
We now invoke the following result from the convex geometry:
\begin{lemma}\textsc{(Radon's Theorem, \cite{radon1921mengen})} 
	\label{Lemma:Radons}	
	A set $\mathcal{A}\subset \real^{m}$ of $m+2$ points can be partitioned into two disjoint sets $\mathcal{A}_{1}$ and $\mathcal{A}_{2}$, such that
	\begin{align}
		\mathrm{Convex}(\mathcal{A}_{1})\cap \mathrm{Convex}(\mathcal{A}_{2})\not=\emptyset,
	\end{align}
	where $\mathrm{Convex}(\mathcal{A}_{i})$ denotes the convex hull of the set $\mathcal{A}$ defined as follows
	\begin{align}
		\mathrm{Convex}(\mathcal{A}_{i})\df \left\{\sum_{k=1}^{|\mathcal{A}_{i}|}\lambda_{k}x_{k}: \sum_{k=1}^{|\mathcal{A}_{i}|}\lambda_{k}=1, \lambda_{k}\geq 0, k=1,2,\cdots,|\mathcal{A}_{i}| \right\}, \quad i=1,2.
	\end{align}
	\hfill $\square$
\end{lemma}
We will show that $\mathrm{dim}_{P}(\mathrm{sgn}(\mathcal{L}))\leq d+ 2$. The argument is by contradiction. Suppose $\mathrm{dim}_{P}(\mathrm{sgn}(\mathcal{F}))>d+ 2$. It must be that there exists
a shattered set
\begin{align}
	\{(\bm{x}_{1},b_{1},c_{1}),(\bm{x}_{2},b_{2},c_{2}),\cdots,(\bm{x}_{(d+3)}, b_{(d+3)},c_{(d+3)})\}\subset \real^{d}\times \real\times \real.
\end{align}
such that, for all $\bm{e}\in \{0,1\}^{d+3}$, there exists a vector $\bm{\xi}_{e}\in \real^{d}$ satisfying
\begin{align}
	[\langle\bm{\xi}_{e},\bm{x}_{i}\rangle+b_{i}]_{+}\geq c_{i}, \text{iff}\ e_{i}=1, \quad \forall 1\leq i\leq d+3.
\end{align}

Observe that we must have $c_{i}\in \real_{+}$ for all $i=1,2,\cdots,d$, since if $c_{i}\leq 0$, then no such shattered set can be demonstrated. But if $c_{i}\in \real_{+}$, for all $\bm{\xi}_{e}\in \real^{d}$, then
\begin{subequations}
	\begin{align}
		&[\langle\bm{\xi}_{e},\bm{x}_{i}\rangle+b_{i}]_{+}\geq c_{i}    \Rightarrow \langle\bm{\xi}_{e},\bm{x}_{i}\rangle \geq c_{i}-b_{i}\\
		&[\langle\bm{\xi}_{e},\bm{x}_{i}\rangle+b_{i}]_{+}<c_{i}    \Rightarrow \langle\bm{\xi}_{e},\bm{x}_{i}\rangle\leq c_{i}-b_{i}.
	\end{align}
\end{subequations}
For each $1\leq i\leq d+ 3$, define the vector $\bm{s}_{i}=(s_{i}^{1},\cdots,s_{i}^{d+1})\in \real^{d+1}$ component-wise according to
\begin{align}
	s_{i}^{j}\df \left\{
	\begin{array}{ll}
		x_{i}^{j} & \mbox{if}\ j<d+1 \\
		c_{i}-b_{i} & \mbox{if}\ j=d+1
	\end{array}\right.
\end{align}

Let $\mathcal{A}=\{z_{1}, z_{2},\cdots, z_{(d+3)}\}\subset \real^{d+1}$, and let $\mathcal{A}_{1}$ and $\mathcal{A}_{2}$ be subsets of $\mathcal{A}$ satisfying the conditions of Radon's lemma (cf. Lemma \ref{Lemma:Radons}). Define a vector $\bm{e} \in \{0,1\}^{d+3}$ component-wise according to
\begin{align}
	e_{i}=\mathbb{I}_{\{z_{i}\in \mathcal{A}_{1}\}}.
\end{align}
Then, for the vector $\bm{\xi}_{e}$, we have,
\begin{subequations}
	\label{Eq:Must_Satisfy}
	\begin{align}
		\label{Eq:Must_Satisfy1}
		\sum_{i=1}^{d}\xi_{e,j}z_{j}&\geq  z_{d+1},\quad \forall x\in \mathcal{A}_{1} \\
		\label{Eq:Must_Satisfy2}
		\sum_{i=1}^{d}\xi_{e,j}z_{j}&< z_{d+1},\quad \forall x\in \mathcal{A}_{2}.
	\end{align}
\end{subequations}
Now, let $\bm{x}_{0}\in \real^{d+1}$ be a point contained in both the convex hull of $\mathrm{Conv}(\mathcal{A}_{1})$ and the convex hull of $\mathrm{Conv}(\mathcal{A}_{2})$. Such a point must exist by Radon’s lemma (cf. Lemma \ref{Lemma:Radons}). Since $\bm{x}_{0}$ satisfies both inequalities in Eqs. \eqref{Eq:Must_Satisfy1}-\eqref{Eq:Must_Satisfy2} simultaneously, this yields a contradiction. Therefore, $\mathrm{dim}_{P}(\mathrm{sgn}(\mathcal{L}))< d+2$, and due to Remark \ref{Remark:2}, we obtain $\mathrm{dim}_{\mathrm{VC}}(\mathrm{sgn}(\mathcal{L}))< d+2$.  This bound is essentially tight in the sense that $\mathrm{dim}_{\mathrm{VC}}(\mathrm{sng}(\mathcal{L}))\geq d+1$. To see this, for any given the labels $y_{1},\cdots,y_{d+1}\in \{\pm\}^{d+1}$, consider the function $f\in \mathcal{F}''$ with $b=y_{d+1}$ and $\bm{\xi}=(\xi_{1},\cdots,\xi_{d})$ with $\xi_{i}=y_{i}-y_{d+1}$. Then, $f(\bm{0})=[\langle\bm{\xi},\bm{0}\rangle+y_{d+1} ]_{+}=y_{d+1}$, and $f(\bm{e}_{i})=[\langle \bm{\xi},\bm{e}_{i} \rangle+y_{d+1}]_{+}=[\xi_{i}+y_{d+1}]_{+}=[y_{i}]_{+}=y_{i}$, where $\bm{e}_{i}$ is the $i$-th basis vector. Therefore, we shatter the points $\bm{e}_{1},\cdots,\bm{e}_{d},\bm{0}$.

Combining all the previous steps yields
\begin{align}
\nonumber
\mathcal{N}(\epsilon,\mathcal{G}_{n,N},L_{2}(P))&\leq \left( \mathcal{N}\left(\epsilon',\mathcal{F}, L_{2}(P)\right)\right)^{Nm}\\ \nonumber
\nonumber
&=\left(\mathcal{N}(\epsilon',\mathcal{L},L_{2}(P))\right)^{Nm} \\ \nonumber 
&\leq \left(\dfrac{K\log (K/\epsilon'^{2})}{\epsilon'^{2}} \right)^{Nm\mathrm{dim}_{\mathrm{VC}}(\mathcal{L})}\\ \nonumber
&\leq \left(\dfrac{K\log (K/\epsilon'^{2})}{\epsilon'^{2}} \right)^{Nm\mathrm{dim}_{\mathrm{VC}}(\mathrm{sgn}(\mathcal{L}))}\\ \label{Eq:VC_Bound}
&=\left(\dfrac{K\log (K/\epsilon'^{2})}{\epsilon'^{2}} \right)^{Nm(d+2)},
\end{align}
where $\epsilon'\df {\lambda (\epsilon)^{1\over r}\over 2m\exp(L^{2}/\lambda)L}$. Plugging the upper bound in Eq. \eqref{Eq:VC_Bound} into Inequality \eqref{Eq:Masoud_1} of Lemma \ref{Lemma:Function_Class_1} completes the proof. $\hfill$ $\blacksquare$

\subsection{Proof of Lemma \ref{Thm:Consistency with respect to the Sampling Data}}
\label{Appendix:Proof_of_Lemma_Consistency}

In this section, we establish the proofs of the upper bound in Theorem \ref{Thm:Consistency with respect to the Sampling Data} regarding the consistency of the finite sample approximations in Section \ref{Section:Primal-Dual Method}. The proof is the adaptation of the one given in \cite{sinha2016learning}. However, our proof crucially uses the Gumbel max perturbation technique.  We first recall the following definition of the population and finite sample estimates
\begin{subequations}
\begin{align}
H(\lambda;\bm{\omega})&=\lambda \log \expect_{Q^{\otimes 2}}\left[\exp\left(-\dfrac{y\tilde{y}}{\lambda}\expect_{\mu_{\bm{\omega}}}[\varphi(\bm{x}_{i};\bm{\xi})\varphi(\bm{x}_{j};\bm{\xi}) ]\right)\right]+\dfrac{\lambda r}{2},\\
H_{n,N}(\lambda;\bm{\omega})&=\lambda \log \dfrac{2}{n(n-1)}\sum_{1\leq i<j\leq n}  \exp\left(-\dfrac{y_{i}y_{j}}{\lambda}\expect_{\widehat{\mu}^{N}_{\bm{\omega}}}[\varphi(\bm{x}_{i};\bm{\xi})\varphi(\bm{x}_{j};\bm{\xi})] \right)  +\dfrac{\lambda r}{2}.
\end{align}
\end{subequations}
Now, define 
\begin{subequations}
\begin{align}
T_{\bm{\omega}}(P)&\df \expect_{P^{\otimes 2}}\Big[y
\tilde{y}\expect_{\mu_{\bm{\omega}}}[\varphi(\bm{x};\bm{\xi})\varphi(\tilde{\bm{x}};\bm{\xi})]\Big]\\
\widehat{T}_{\bm{\omega}}(P)&\df \dfrac{1}{N}\expect_{P^{\otimes 2}}\Big[y\tilde{y}\langle \bm{\varphi}_{\bm{\omega}}(\bm{x}), \bm{\varphi}_{\bm{\omega}}(\tilde{\bm{x}}) \rangle \Big].
\end{align}
\end{subequations}
From Lemma \ref{Lemma:Saddle Point Characterization of Functional Optimization}, we recall that
\begin{subequations}
\begin{align}
\label{Eq:State_1}
\sup_{\bm{\omega}\in \mathrm{S}_{+}^{m}}\min_{\lambda>0}H(\lambda;\bm{\omega})&= \sup_{\bm{\omega}\in \mathrm{S}_{+}^{m}}\inf_{P\in \mathcal{P}}T_{\bm{\omega}}(P)\\
\label{Eq:State_2}
\sup_{\bm{\omega}\in \mathrm{S}_{+}^{m}}\min_{\lambda>0}H_{n,N}(\lambda;\bm{\omega})&=\sup_{\bm{\omega}\in \mathrm{S}_{+}^{m}}\inf_{\widehat{P}^{n}\in \widehat{\mathcal{P}}^{n}}\widehat{T}_{\bm{\omega}}(\widehat{P}^{n}),
\end{align}
\end{subequations}
where $\mathcal{P}$ and $\widehat{\mathcal{P}}^{n}$ are defined as follows
\begin{subequations}
\begin{align}
\mathcal{P}&\df \{ P\in \mathcal{M}(\mathcal{X}\times \mathcal{Y}): D_{\mathrm{KL}}(P||Q)\leq r, P\ll Q \},\\
\widehat{\mathcal{P}}^{n}&\df \{ P\in \mathcal{M}(\mathcal{X}\times \mathcal{Y}): D_{\mathrm{KL}}(P||\widehat{Q}^{n})\leq r, P\ll \widehat{Q}^{n} \}.
\end{align}
\end{subequations}
Furthermore, lets define
\begin{align}
\left(\bm{\omega}_{\ast},\widehat{P}_{\ast}^{n}\right)=\arg\sup_{\bm{\omega}\in \mathrm{S}_{+}^{m}}\inf_{\widehat{P}^{n}\in \widehat{\mathcal{P}}^{n}}{1\over N}\expect_{\widehat{P}^{n,\otimes 2}}\Big[y\tilde{y}\langle \bm{\varphi}_{\bm{\omega}}(\bm{x}), \bm{\varphi}_{\bm{\omega}}(\tilde{\bm{x}}) \rangle \Big].
\end{align} 
Fix $\bm{\omega}\in \mathrm{S}_{+}^{m}$. We wish to upper bound the following difference term
\begin{align}
\left|\sup_{\bm{\omega}\in \mathrm{S}_{+}^{m}}\inf_{P\in \mathcal{P}}T_{\bm{\omega}}(P)-T_{\bm{\omega}_{\ast}}(\widehat{P}_{\ast}^{n})\right|\leq \mathsf{S}_{1}+\mathsf{S}_{2}+\mathsf{S}_{3},
\end{align}
where $\mathsf{S}_{1},\mathsf{S}_{2}$, and $\mathsf{S}_{3}$ are defined as below
\begin{align*}
\mathsf{S}_{1}&\df \left|\sup_{\bm{\omega}\in \mathrm{S}_{+}^{m}}\inf_{P\in \mathcal{P}}T_{\bm{\omega}}(P)-\sup_{\bm{\omega}\in \mathrm{S}_{+}^{m}}\inf_{P\in \mathcal{P}}\widehat{T}_{\bm{\omega}}(P)\right|,\\
\mathsf{S}_{2}&\df \left|\sup_{\bm{\omega}\in \mathrm{S}_{+}^{m}} \inf_{P\in \mathcal{P}}\widehat{T}_{\bm{\omega}}(P)- \widehat{T}_{\bm{\omega}_{\ast}}(\widehat{P}_{\ast}^{n})\right|,\\
\mathsf{S}_{3}&\df \left| \widehat{T}_{\bm{\omega}_{\ast}}(\widehat{P}_{\ast}^{n})-T_{\bm{\omega}_{\ast}}(\widehat{P}_{\ast}^{n})\right|.
\end{align*}
We now consider the following upper bound on each term:

\textbf{Upper Bound on $\mathsf{S}_{1}$ and $\mathsf{S}_{3}$:}

To compute upper bounds on $\mathsf{S}_{1}$ and $\mathsf{S}_{3}$, we leverage the following inequalities
\begin{subequations}
\begin{align}
\label{Eq:S1}
\mathsf{S}_{1}&\df \left|\sup_{\bm{\omega}\in \mathrm{S}_{+}^{m}}\inf_{P\in \mathcal{P}}T_{\bm{\omega}}(P)-\sup_{\bm{\omega}\in \mathrm{S}_{+}^{m}}\inf_{P\in \mathcal{P}}\widehat{T}_{\bm{\omega}}(P)\right|\leq \sup_{\bm{\omega}\in \mathrm{S}_{+}^{m}} \sup_{P\in \mathcal{P}}\left|T_{\bm{\omega}}(P)-\widehat{T}_{\bm{\omega}}(P)\right|,\\
\label{Eq:S3}
\mathsf{S}_{3}&\df \left| \widehat{T}_{\bm{\omega}}(\widehat{P}_{\ast}^{n})-T_{\bm{\omega}}(\widehat{P}_{\ast}^{n})\right|\leq \sup_{\bm{\omega}\in \mathrm{S}_{+}^{m}} \sup_{P\in \widehat{\mathcal{P}}^{n}}\left|T_{\bm{\omega}}(P)-\widehat{T}_{\bm{\omega}}(P)\right|.
\end{align}
\end{subequations}
To establish a uniform concentration bound, we use the fact that $\bm{\omega}\in \mathrm{S}_{m}^{+}$ and the simplex $\mathrm{S}_{m}^{+}$ is compact with the diameter $\mathrm{diam}(\mathrm{S}_{m}^{+})=\max_{\bm{\omega}_{0},\bm{\omega}_{1}\in \mathrm{S}_{m}^{+}}\|\bm{\omega}_{0}-\bm{\omega}_{1}\|_{2}\leq 2$. Therefore, we can find an $\varepsilon$-net that covers $\mathrm{S}_{m}^{+}$, using at most $M=(4\mathrm{diam}(\mathrm{S}_{+}^{m})/\varepsilon)^{m}=(8/\varepsilon)^{m}$ balls of the radius $\varepsilon>0$. Lets $\{\bm{\omega}_{i}\}_{i=1}^{M}$ denotes the center (anchor) of these balls. Now, define the error term
\begin{align}
\label{Eq:Error_Term}
E_{N}(\bm{z},\tilde{\bm{z}})\df \expect_{P^{\otimes 2}}\Big[{y\tilde{y}}\expect_{\mu_{\bm{\omega}_{i}}}[\varphi(\bm{x};\bm{\xi})\varphi(\tilde{\bm{x}};\bm{\xi})]\Big]-\expect_{P^{\otimes 2}}\Big[{y\tilde{y}\over N}\langle \bm{\varphi}_{\bm{\omega}_{i}}(\bm{x}), \bm{\varphi}_{\bm{\omega_{i}}}(\tilde{\bm{x}}) \rangle \Big],
\end{align}
where $\bm{z}\df(y,\bm{x}),\tilde{\bm{z}}\df (\tilde{y},\tilde{\bm{x}})$, and $\bm{z},\tilde{\bm{z}}\in \mathcal{Z}=\mathcal{X}\times \mathcal{Y}$. Invoking the H\"{o}lder inequality $|\langle f,g\rangle|\leq \|f\|_{p}\|g\|_{q},p^{-1}+q^{-1}=1$ for the measure space $L^{p}(\mathcal{Z}^{2},P^{\otimes 2})$ with $f=E_{N}(\bm{z}_{0},\bm{z}_{1})$, $g\equiv 1$ and $p=q=2$, it follows that 
\begin{align}
\left|\int_{\mathcal{Z}\times \mathcal{Z}}E_{N}(\bm{z},\tilde{\bm{z}})\mathrm{d}P^{\otimes 2}(\bm{z},\tilde{\bm{z}})\right|\leq  \left(\int_{\mathcal{Z}\times \mathcal{Z}}\left|E_{N}(\bm{z},\tilde{\bm{z}})\right|^{2}\mathrm{d}P^{\otimes 2}(\bm{z},\tilde{\bm{z}})\right)^{1\over 2}.
\end{align}
Due to the non-negativity of the integrand on the right hand side of the preceding inequality, the supremum can be moved inside the square root, \textit{i.e.},
\begin{align}
\nonumber
\sup_{P\in \mathcal{P}}\left|\int_{\mathcal{Z}\times \mathcal{Z}}E_{N}(\bm{z},\tilde{\bm{z}})\mathrm{d}P^{\otimes 2}(\bm{z},\tilde{\bm{z}})\right|&\leq \sup_{P\in \mathcal{P}} \left(\int_{\mathcal{Z}\times \mathcal{Z}}\left|E_{N}(\bm{z},\tilde{\bm{z}})\right|^{2}\mathrm{d}P^{\otimes 2}(\bm{z},\tilde{\bm{z}})\right)^{1\over 2}\\
&\leq \left(\sup_{P\in \mathcal{P}} \int_{\mathcal{Z}\times \mathcal{Z}}\left|E_{N}(\bm{z},\tilde{\bm{z}})\right|^{2}\mathrm{d}P^{\otimes 2}(\bm{z},\tilde{\bm{z}})\right)^{1\over 2}.
\end{align}
We leverage Lemma \ref{Lemma:Saddle Point Characterization of Functional Optimization} to obtain
\begin{align}
\nonumber
\left(\sup_{P\in \mathcal{P}}\int_{\mathcal{Z}\times \mathcal{Z}}\left|E_{N}(\bm{z},\tilde{\bm{z}})\right|^{2}\mathrm{d}P^{\otimes 2}(\bm{z}_{0},\bm{z}_{1})\right)^{1\over 2}&=\left(\inf_{\lambda>0}\lambda \log \expect_{Q^{\otimes 2}}\left[\exp\left({1\over \lambda}|E_{N}(\bm{z},\tilde{\bm{z}})|^{2}\right)\right]+\dfrac{\lambda r}{2}\right)^{1\over 2}\\
&\leq \left(\lambda_{0} \log \expect_{Q^{\otimes 2}}\left[\exp\left({1\over \lambda_{0}}|E_{N}(\bm{z},\tilde{\bm{z}})|^{2}\right)\right]+\dfrac{\lambda_{0} r}{2}\right)^{1\over 2},
\end{align}
where $\lambda_{0}\in \real_{+}\backslash\{0\}$ is a free parameter of the upper bound which we will specify in the sequel. Using the Chernoff bound, we obtain for $\beta\geq 0$ that
\begin{align}
\nonumber
&\prob\left(\sup_{P\in \mathcal{P}}\left| \int_{\mathcal{Z}\times \mathcal{Z}}E_{N}(\bm{z}_{0},\bm{z}_{1})\mathrm{d}P^{\otimes 2}(\bm{z}_{0},\bm{z}_{1})\right|\geq \delta\right)\\ \nonumber
&\leq \prob\left(\lambda_{0} \log \expect_{Q^{\otimes 2}}\left[\exp\left({1\over \lambda_{0}}|E_{N}(\bm{z}_{0},\bm{z}_{1})|^{2}\right)\right]+\dfrac{\lambda_{0} r}{2}\geq \delta^{2} \right)\\ \label{Eq:Going_Back}
&\leq \exp(-\beta \delta^{2})\cdot \exp\left(\dfrac{\beta \lambda_{0}r}{2}\right)\left( \expect_{Q^{\otimes 2}}\left[\exp\left({1\over \lambda_{0}}|E_{N}(\bm{z}_{0},\bm{z}_{1})|^{2}\right)\right]\right)^{\lambda_{0}\beta}.
\end{align}

To compute the expectation, we now establish the following auxiliary result:
\begin{lemma}
\label{Lemma:Sub-Gaussian_Error}
Consider the error term defined in Eq. \eqref{Eq:Error_Term}. Then, $E_{N}(\bm{z}_{0},\bm{z}_{1})$ is zero-mean sub-Gaussian random variable with the Orlicz norm of $\|E_{N}(\bm{z}_{0},\bm{z}_{1})\|_{\psi_{2}}\leq  {L^{2}\over 2\sqrt{N}}$.
\end{lemma}
\begin{proof}
	The proof is presented in Appendix \ref{App:Sub-Gaussian_Error}.
\end{proof}

We leverage the result of Lemma \ref{Lemma:Sub-Gaussian_Error} to compute an upper bound on the expectation on the right hand side of \eqref{Eq:Going_Back}. In particular, we let $\lambda_{0}=1/\|E_{N}\|^{2}_{\psi_{2}}$, and notice that $\expect\big[\exp\big(|E_{N}|^{2}/\|E_{N}\|^{2}_{\psi_{2}}\big)\big]\leq 2$. Then,
\begin{align}
\nonumber
&\prob\left(\sup_{P\in \mathcal{P}}\left| \int_{\mathcal{Z}\times \mathcal{Z}}E_{N}(\bm{z}_{0},\bm{z}_{1})\mathrm{d}P^{\otimes 2}(\bm{z}_{0},\bm{z}_{1})\right|\geq \delta\right)\\ \nonumber
&\hspace{8mm}\leq \exp(-\beta\delta^{2})\exp\left( \dfrac{\beta r}{2\|E_{N}\|^{2}_{\psi^{2}}}\right)\exp\left({\beta\ln(2)\over \|E_{N}\|^{2}_{\psi_{2}}}\right), 
\end{align}
for all $\beta \in \real_{+}$. Let $\beta={\|E_{N}\|_{\psi_{2}}^{2}\over (r+2\ln(2))}$ with $\|E_{N}\|_{\psi_{2}}={L^{2}\over 2\sqrt{N}}$. Then, we obtain that 
\begin{align}
\prob\left(\sup_{P\in \mathcal{P}}\left| \int_{\mathcal{Z}\times \mathcal{Z}}E_{N}(\bm{z}_{0},\bm{z}_{1})\mathrm{d}P^{\otimes 2}(\bm{z}_{0},\bm{z}_{1})\right|\geq \delta\right)\leq 2\exp\left(- {L^{4}\delta^{2}\over 2N(r+2\ln(2))} \right).
\end{align}
We thus conclude that
\begin{align}
\prob\left(\sup_{P\in \mathcal{P}}\left|T_{\bm{\omega}_{i}}(P)-\widehat{T}_{\bm{\omega}_{i}}(P)\right|\geq \delta \right)\leq 2\exp\left(- {L^{4}\delta^{2}\over 2N(r+2\ln(2))} \right).
\end{align}
Applying the union bound then yields
\begin{align}
\prob\left(\max_{1\leq i\leq M}\sup_{P\in \mathcal{P}}\left|T_{\bm{\omega}_{i}}(P)-\widehat{T}_{\bm{\omega}_{i}}(P)\right|\geq \delta \right)\leq 2\left(8\over \varepsilon \right)^{m}\exp\left(- {L^{4}\delta^{2}\over 2N(r+2\ln(2))} \right).
\end{align}
Therefore, with the probability of at least $1-\rho$, we obtain that
\begin{align}
\max_{1\leq i\leq M}\sup_{P\in \mathcal{P}}\left|T_{\bm{\omega}_{i}}(P)-\widehat{T}_{\bm{\omega}_{i}}(P)\right|\leq \sqrt{{2N(r+2\ln(2))\over L^{4}}\log\left({2\times 8^{m}\over \rho\varepsilon^{m}}\right)}.
\end{align}
It is easy to see that the mapping $\bm{\omega}\mapsto T_{\bm{\omega}}(P)-\widehat{T}_{\bm{\omega}}(P)$ is Lipschitz with the constant $2L^{2}$. Therefore, with the probability of at least $1-\rho$, we have
\begin{align}
\nonumber
\sup_{\bm{\omega}\in \mathrm{S}_{m}^{+}}\sup_{P\in \mathcal{P}}\left|T_{\bm{\omega}}(P)-\widehat{T}_{\bm{\omega}}(P)\right|&\leq \max_{1\leq i\leq M}\sup_{P\in \mathcal{P}}\left|T_{\bm{\omega}_{i}}(P)-\widehat{T}_{\bm{\omega}_{i}}(P)\right|+2L^{2}\varepsilon\\
&\leq \sqrt{{2N(r+2\ln(2))\over L^{4}}\log\left({2\times 8^{m}\over \rho\varepsilon^{m}}\right)}+2L^{2}\varepsilon.
\end{align}
Due to Inequality \eqref{Eq:S1}, we obtain
\begin{align}
\left|\sup_{\bm{\omega}\in \mathrm{S}_{+}^{m}}\inf_{P\in \mathcal{P}}T_{\bm{\omega}}(P)-\sup_{\bm{\omega}\in \mathrm{S}_{+}^{m}}\inf_{P\in \mathcal{P}}\widehat{T}_{\bm{\omega}}(P)\right|\leq \sqrt{{2N(r+2\ln(2))\over L^{4}}\log\left({2\times 8^{m}\over \rho\varepsilon^{m}}\right)}+2L^{2}\varepsilon.
\end{align}
Letting $\varepsilon=\dfrac{1}{2L^{4}}\sqrt{2N(r+2\ln(2))}$ yields
\begin{align}
\left|\sup_{\bm{\omega}\in \mathrm{S}_{+}^{m}}\inf_{P\in \mathcal{P}}T_{\bm{\omega}}(P)-\sup_{\bm{\omega}\in \mathrm{S}_{+}^{m}}\inf_{P\in \mathcal{P}}\widehat{T}_{\bm{\omega}}(P)\right|\leq 2\sqrt{{2N(r+2\ln(2))\over L^{4}}\log\left({2\times (16L^{4})^{m}\over \rho (2N(r+2\ln(2))^{m\over 2}) }\right)}.
\end{align}
Similarly, it can be shown that with the probability of at least $1-\rho$,
\begin{align}
 \left| \widehat{T}_{\bm{\omega}}(\widehat{P}_{\ast}^{n})-T_{\bm{\omega}}(\widehat{P}_{\ast}^{n})\right|\leq 2\sqrt{{2N(r+2\ln(2))\over L^{4}}\log\left({2\times (16L^{4})^{m}\over \rho (2N(r+2\ln(2))^{m\over 2}) }\right)}.
\end{align}

\textbf{Upper Bound on $\mathsf{S}_{2}$:} 
In this case, we employ Talagrand's concentration of measure for Lipschitz functions:
\begin{theorem}\textsc{(Talagrand's Concentration of Measure \cite{talagrand1995concentration})}
	Let $\bm{X}=(X_{1},\cdots, X_{n})$ be a random variable with independent components taking values $|X_{i}|\leq K, i=1,2,\cdots,n$, for $K>0$. Let $f:\real^{n}\rightarrow \real$ be a convex $L$-Lipschitz function. Then, for all $t>0$
	\begin{align}
	\label{Eq:Establish}
	\prob\Big(|f(\bm{X})-\median[f(\bm{X})]|\geq K\delta \Big)\leq 4e^{-\delta^{2}/4L^{2}},
	\end{align}
	where $\median[f(\bm{X})]$ is the median of the random variable $f(\bm{X})$ defined as follows
	\begin{align}
	\median[f(\bm{X})]\df \sup\Big\{\delta\in \real:\prob(f(\bm{X})\leq \delta)\leq {1\over 2} \Big\}. 
	\end{align}
	\hfill $\square$
\end{theorem}

Let us mention that the concentration inequality in Eq. \eqref{Eq:Establish} provides the following bound on the distance between the median and the expectation
\begin{align}
\nonumber
|\expect[f(\bm{X})]-\median[f(\bm{X})]|&\leq \expect[|f(\bm{X})-\median[f(\bm{X})]|\\ \nonumber
&= \int_{0}^{\infty}\prob\left(|f(\bm{X})-\median[f(\bm{X})]|\geq \delta \right)\mathrm{d}\delta\\
&\leq \int_{0}^{\infty} 4e^{-\delta^{2}/4L^{2}K^{2}}=8KL\sqrt{4\pi}.
\end{align}
Hence, we have 
\begin{align}
\median[f(\bm{X})]-8KL\sqrt{4\pi}\leq \expect[f(\bm{X})]\leq \median[f(\bm{X})]+8KL\sqrt{4\pi}.
\end{align}
Therefore, due to parts (iii) and (iv) of \cite[Proposition 2]{naor2008concentration}, concentration around the median implies the concentration around the mean with different constants, \textit{i.e.},
\begin{align}
\prob\Big(|f(\bm{X})-\expect[f(\bm{X})]|\geq K\delta \Big)\leq 2\exp\left(\dfrac{\delta^{2}}{K^{2}L^{2}} \right).
\end{align}
We recall the definition
\begin{align}
\nonumber
\widehat{T}_{\bm{\omega}}(\widehat{P}_{\ast}^{n})&=\sup_{\widehat{P}^{n}\in \mathcal{P}^{n}} \widehat{T}_{\bm{\omega}}(\widehat{P}^{n})\\
&=\min_{\lambda>0}\lambda \log \expect_{\widehat{Q}^{n,\otimes 2}}\left[ \exp\left(-\dfrac{y\tilde{y}}{\lambda}\expect_{\widehat{\mu}^{N}_{\bm{\omega}}}[\varphi(\bm{x};\bm{\xi})\varphi(\tilde{\bm{x}};\bm{\xi})] \right)\right]  +\dfrac{\lambda r}{2}.
\end{align}
Furthermore, let $\bm{z}_{i}=(y_{i},\bm{x}_{i}),i=1,2,\cdots,n$, and
\begin{align}
\label{Eq:Recall_Xij}
X_{ij}\df X(\bm{z}_{i},\bm{z}_{j})\df  \expect_{\widehat{\mu}^{N}_{\bm{\omega}}}\Big[y_{i}y_{j}\varphi(\bm{x}_{i};\bm{\xi})\varphi(\bm{x}_{j};\bm{\xi})\Big].
\end{align}
Consider the mapping
\begin{align}
\label{Eq:Minimizer}
\bm{X}=(X_{ij})_{1\leq i<j\leq n}\mapsto f(\bm{X})\df \min_{\lambda>0}F(\bm{X},\lambda)\df \lambda \log \expect_{\widehat{Q}^{n,\otimes 2}}\left[\exp\left(-{1\over \lambda} X(\bm{z},\tilde{\bm{z}})\right)\right]+\dfrac{\lambda r}{2}.
\end{align}
To show that this mapping is Lipschitz, we first show that the (unique) minimizer of the optimization problem in Eq. \eqref{Eq:Minimizer} is bounded, \textit{i.e.}, there exists $c\df c(n,r)>0$ and $0<C\df C(n,r)<+\infty$ such that $c\leq \lambda_{\ast}\leq C$, where
\begin{align}
\lambda_{\ast}\df \arg\min_{\lambda>0}F(\bm{X},\lambda).
\end{align}
To this end, for a fixed $\bm{X}\in \mathcal{Z}$, consider the partial derivative
\begin{align}
\nonumber
\dfrac{\partial F(\bm{X},\lambda)}{\partial\lambda}&=\log \expect_{\widehat{Q}^{n,\otimes 2}}\left[\exp\left(-{1\over \lambda} X(\bm{z},\tilde{\bm{z}})\right)\right]+\dfrac{r}{2}\\ \label{Eq:optimal_solution}
&\hspace{4mm}+\dfrac{\sum_{1\leq i<j\leq n} \dfrac{1}{\lambda}X(\bm{z}_{i},\bm{z}_{j})\exp\left(-\dfrac{1}{\lambda}X(\bm{z}_{i},\bm{z}_{j})\right)  }{\sum_{1\leq i<j\leq n} \exp\left( -\dfrac{1}{\lambda} X(\bm{z}_{i},\bm{z}_{j})\right)}.
\end{align}
In the limit of $\lambda\downarrow 0$,  the soft-max terms turns into the maximization over the indices $i,j\in \{1,2,\cdots,n\}$, \textit{i.e.},
\begin{align}
\lim_{\lambda\downarrow 0}\dfrac{\partial F(\bm{X},\lambda)}{\partial\lambda}=\lim_{\lambda\downarrow 0} \max_{1\leq i<j\leq n}-{2\over \lambda}{X_{ij}}+\dfrac{r}{2}\rightarrow +\infty.
\end{align}
Moreover, since $|X_{ij}|<L^{2}$ due to Assumption \ref{Assumption:1}, we have
\begin{align}
\lim_{\lambda\rightarrow +\infty}\dfrac{\partial}{\partial\lambda}F(\bm{X},\lambda)=\dfrac{r}{2}>0.
\end{align}
The partial derivative $\partial F(\bm{X},\lambda)/\partial\lambda$ is a continuous function of $\lambda$.  It can thus only vanish in the interior of the non-negative orthant $\lambda_{\ast}\in \real_{+}\backslash \{0,\infty\}$. Therefore, the mapping $f(\bm{X})=F(\bm{X},\lambda_{\ast})$ is Lipschitz with the constant  $L= {C\over n^{2}} \exp\left(\dfrac{2L^{2}}{c}\right)$. Based on Talagrand's concentration inequality, for the anchor nodes $\{\bm{\omega}_{i} \}_{i=1}^{M}$ we have that
\begin{align}
\nonumber
\prob\left(\left|\sup_{\widehat{P}^{n}\in \mathcal{P}^{n}} \widehat{T}_{\bm{\omega}_{i}}(\widehat{P}^{n})-\expect\left[\sup_{\widehat{P}^{n}\in \mathcal{P}^{n}} \widehat{T}_{\bm{\omega}_{i}}(\widehat{P}^{n})\right] \right|\geq \delta \right)\leq 2\exp\left(-\dfrac{\delta^{2}n^{2}}{K^{2}L^{2}} \right),
\end{align}
where $K=C\exp\left({2L^{2}\over c}\right)$. From the union bound, we obtain
\begin{align}
\nonumber
\prob\left(\max_{1\leq i\leq M}\left|\sup_{\widehat{P}^{n}\in \mathcal{P}^{n}} \widehat{T}_{\bm{\omega}_{i}}(\widehat{P}^{n})-\expect\left[\sup_{\widehat{P}^{n}\in \mathcal{P}^{n}} \widehat{T}_{\bm{\omega}_{i}}(\widehat{P}^{n})\right] \right|\geq \delta \right)\leq 2\left({8\over \varepsilon} \right)^{m}\exp\left(-\dfrac{\delta^{2}n^{2}}{K^{2}L^{2}} \right).
\end{align}
Therefore, with the probability of at least $1-\rho$, the following inequality holds
\begin{align}
\max_{1\leq i\leq M}\left|\sup_{\widehat{P}^{n}\in \mathcal{P}^{n}} \widehat{T}_{\bm{\omega}_{i}}(\widehat{P}^{n})-\expect\left[\sup_{\widehat{P}^{n}\in \mathcal{P}^{n}} \widehat{T}_{\bm{\omega}_{i}}(\widehat{P}^{n})\right] \right|\leq  \dfrac{KL}{n}\sqrt{ \log\left({2\times 8^{m}\over \rho\times \varepsilon^{m} }\right)}.
\end{align}
In the following lemma, we characterize an upper bound and a lower bound on the expectation on the left hand side of the previous display:
\begin{lemma}
	\label{lemma:stupid}
For all the realizations of the data samples $(y_{i},\bm{x}_{i})_{1\leq i\leq N}\sim_{\text{i.i.d.}} Q$ as well as random feature samples $(\bm{\xi}_{i}^{k})_{1\leq k\leq N}\sim_{\text{i.i.d.}} \mu_{i}$, the following upper bound holds 
\begin{align}
\label{Bound:1}
\expect_{Q^{\otimes 2}}\left[\sup_{\widehat{P}^{n}\in \mathcal{P}^{n}} \widehat{T}_{\bm{\omega}}(\widehat{P}^{n})\right]\leq \sup_{P\in \mathcal{P}} \widehat{T}_{\bm{\omega}}(P).
\end{align}
Furthermore, the following lower bound holds
\begin{align}
\label{Bound:2}
\expect_{Q^{\otimes 2}}\left[\sup_{\widehat{P}^{n}\in \mathcal{P}^{n}}\widehat{T}_{\bm{\omega}}(\widehat{P}^{n})\right]\geq \sup_{P\in \mathcal{P}} \widehat{T}_{\bm{\omega}}(P)-{C^{3}K^{2}\over n^{2}},
\end{align}
where $K=C\exp({2L^{2}\over c})$.
\end{lemma}
\begin{proof}
	The proof is presented in Appendix \ref{Appendix:Aux_2}. 
\end{proof}

Using the bounds in Eqs. \eqref{Bound:1} and \eqref{Bound:2} of Lemma \ref{lemma:stupid} yields
\begin{align}
\label{Eq:program}
\max_{1\leq i\leq M}\left|\sup_{\widehat{P}^{n}\in \mathcal{P}^{n}} \widehat{T}_{\bm{\omega}_{i}}(\widehat{P}^{n})-\sup_{P\in \mathcal{P}} \widehat{T}_{\bm{\omega}}(P) \right|\leq  \dfrac{KL}{n}\sqrt{ \log\left({2\times 8^{m}\over \rho\times \varepsilon^{m} }\right)}+\dfrac{C^{3}K^{2}}{n^{2}}.
\end{align}
We then obtain
\begin{align}
\nonumber
\sup_{\bm{\omega}\in \mathrm{S}_{m}^{+}}\left|\sup_{\widehat{P}^{n}\in \mathcal{P}^{n}} \widehat{T}_{\bm{\omega}}(\widehat{P}^{n})-\sup_{P\in \mathcal{P}} \widehat{T}_{\bm{\omega}}(P) \right|&\leq \max_{1\leq i\leq M}\left|\sup_{\widehat{P}^{n}\in \mathcal{P}^{n}} \widehat{T}_{\bm{\omega}_{i}}(\widehat{P}^{n})-\sup_{P\in \mathcal{P}} \widehat{T}_{\bm{\omega}}(P) \right|+2L^{2}\varepsilon.
\end{align}
Combining the preceding inequality with Eq. \eqref{Eq:program} yields
\begin{align}
\sup_{\bm{\omega}\in \mathrm{S}_{m}^{+}}\left|\sup_{\widehat{P}^{n}\in \mathcal{P}^{n}} \widehat{T}_{\bm{\omega}}(\widehat{P}^{n})-\sup_{P\in \mathcal{P}} \widehat{T}_{\bm{\omega}}(P) \right|\leq \dfrac{KL}{n}\sqrt{ \log\left({2\times 8^{m}\over \rho\times \varepsilon^{m} }\right)}+\dfrac{C^{3}K^{2}}{n^{2}}+2L^{2}\varepsilon.
\end{align}
Letting $\varepsilon={K\over 2Ln}$ yields
\begin{align}
\nonumber
\mathsf{S}_{2}=\left|\sup_{\bm{\omega}\in \mathrm{S}_{+}^{m}} \inf_{P\in \mathcal{P}}\widehat{T}_{\bm{\omega}}(P)- \widehat{T}_{\bm{\omega}_{\ast}}(\widehat{P}_{\ast}^{n})\right|&\leq \sup_{\bm{\omega}\in \mathrm{S}_{m}^{+}}\left|\sup_{\widehat{P}^{n}\in \mathcal{P}^{n}} \widehat{T}_{\bm{\omega}}(\widehat{P}^{n})-\sup_{P\in \mathcal{P}} \widehat{T}_{\bm{\omega}}(P) \right|\\
&\leq \dfrac{2KL}{n}\sqrt{ \log\left({2\times (16Ln)^{m}\over \rho\times K^{m} }\right)}+\dfrac{C^{3}K^{2}}{n^{2}}.
\end{align}

\subsection{Proof of Theorem \ref{Thm:Min-Max_Rate}}
\label{Appendix:The_Min_Max_Rate}

Let $\Theta$ denotes a set of parameters containing the element $\theta$ which we wish to estimate. Assume there is a class $\mathcal{P} = \{P_{\theta}:\theta\in \Theta\}$ of probability measures on $\mathcal{X}$ indexed by $\Theta$. Suppose $d:\Theta\times \Theta\rightarrow \real_{+}$ is a metric on $\Theta$. LeCam's method now can be encapsulated in the following theorem:
\begin{theorem}\textsc{(LeCam's Lower Bound, see, e.g., \cite{yu1997assouad})}
\label{Thm:LeCam}
Let $P^{0},P^{1}\in \mathcal{P}$. Let $\theta(P)$ denotes a parameter taking values in the metric space $(\Theta,d)$. Then,
\begin{align}
	\label{Eq:LeCam}
\inf_{\widehat{\theta}\in \Theta} \sup_{P\in \mathcal{P}} \expect_{P}\Big[d\Big(\widehat{\theta},\theta(P)\Big)\Big]\geq \dfrac{\Delta}{4}\int_{\mathcal{X}} \min\{p^{0}_{n}(\bm{x}),p^{1}_{n}(\bm{x})\}\mathrm{d}\bm{x}\geq \dfrac{\Delta}{8}e^{-n D_{\mathrm{KL}}(P^{0}||P^{1})},
\end{align}
where $\Delta\df d(\theta(P^{0}),\theta(P^{1}))$. Moreover, $p^{0}_{n}(\bm{x})\df \prod_{i=1}^{n}p^{0}(x_{i})$ and $p^{1}_{n}(\bm{x})\df \prod_{i=1}^{n}p^{1}(x_{i})$ where $p^{0}(x)\df \mathrm{d}P^{0}/\mathrm{d}x$ and $p^{1}(x)\df \mathrm{d}P^{1}/\mathrm{d}x$ are the Lebesgue densities.
\end{theorem}

Now, suppose $P(y)=Q(y)$, and consider the case of a balanced data-set. Then, following Eq. \eqref{Eq:Plainly_explained} and the ensuing discussion in Section \ref{Eq:Connections to the maximum mean discrepancy (MMD)}, we equivalently can consider the minimax estimation rate of the following optimization problem
\begin{align}
	\label{Eq:damn}
(P_{+},P_{-})\mapsto \bm{\omega}(P_{+},P_{-})=\max_{\bm{\omega}\in \mathrm{S}_{m}^{+}} \sum_{i=1}^{m}\omega_{i}\mathrm{MMD}^{2}_{K_{i}}[P_{+},P_{-}],
\end{align}
where $(P_{+},P_{-})\in \mathcal{P}_{\pm}$. Consider the following normal distributions for the marginals:
\begin{subequations}
\begin{align}
P_{+}&=\mathsf{N}(\bm{\mu}_{+}^{P},\sigma^{2}\bm{I}_{d\times d}), \quad  P_{-}=\mathsf{N}(\bm{\mu}_{-}^{P},\sigma^{2}\bm{I}_{d\times d}),\\
Q_{+}&=\mathsf{N}(\bm{\mu}_{+}^{Q},\sigma^{2}\bm{I}_{d\times d}), \quad Q_{-}=\mathsf{N}(\bm{\mu}_{-}^{Q},\sigma^{2}\bm{I}_{d\times d}),
\end{align}
\end{subequations}
where $Q_{+}\otimes Q_{-}$ is the center of the distribution ball $\mathcal{P}_{\pm}$. We impose the following conditions on the means and variances of the normal distributions:
\begin{itemize}
	\item[(C.1)] $\|\bm{\mu}_{-}^{P}-\bm{\mu}_{-}^{Q}\|_{2}^{2}+\|\bm{\mu}_{+}^{P}-\bm{\mu}_{+}^{Q}\|_{2}^{2}\leq 4r\sigma^{2}$,
	\item[(C.2)] $\|\bm{\mu}_{-}^{P}-\bm{\mu}_{+}^{P}\|_{2}^{2}\geq 4\sigma^{2}$ and $\|\bm{\mu}_{-}^{Q}-\bm{\mu}_{+}^{1}\|_{2}^{Q}\geq 4\sigma^{2}$.
\end{itemize}
The KL divergence between $P_{+}\otimes P_{-}$ and $Q_{+}\otimes Q_{-}$ is 
\begin{align}
	\label{Eq:MMD_integral}
D_{\mathrm{KL}}(P_{+}\otimes P_{-}||Q_{+}\otimes Q_{-})=\dfrac{1}{2\sigma^{2}}\|\bm{\mu}_{-}^{P}-\bm{\mu}_{-}^{Q}\|_{2}^{2}+\dfrac{1}{2\sigma^{2}}\|\bm{\mu}_{+}^{P}-\bm{\mu}_{+}^{Q}\|_{2}^{2}.
\end{align}
Therefore, $P_{+}\otimes P_{-}\in \mathcal{P}_{\pm}$ iff (C.1) is satisfied. In addition, from \cite[Eqs. (25)]{tolstikhin2017minimax}, we obtain
\begin{align}
\label{Eq:A_lower_bound}
\mathrm{MMD}_{K_{i}}^{2}[P_{+},P_{-}]&=\int_{0}^{\infty}2\left(\dfrac{1}{1+4t\sigma^{2}}\right)^{d/2}\left(1-\exp\left(-\dfrac{t\|\bm{\mu}_{+}^{P}-\bm{\mu}_{-}^{P}\|_{2}^{2}}{1+4t\sigma^{2}} \right) \right)\mathrm{d}\nu_{i}(t).
\end{align}
To compute an upper bound, we employ the elementary inequality $1-\exp(-x)\leq x, \forall x\geq 0$. We have
\begin{align}
\label{Eq:An_upper_bound}
\mathrm{MMD}_{K_{i}}^{2}[P_{+},P_{-}]\leq 2\|\bm{\mu}_{+}^{P}-\bm{\mu}_{-}^{P}\|_{2}^{2}\int_{0}^{\infty}\dfrac{t}{(1+4t\sigma^{2})^{(d+2)/2}}\nu_{i}(\mathrm{d}t).
 \end{align}
A lower bound is also established in \cite[Eqs. (27)]{tolstikhin2017minimax} for the right hand side of Eq. \eqref{Eq:A_lower_bound},
\begin{align}
	\nonumber
\mathrm{MMD}_{K_{i}}^{2}[P_{+},P_{-}]&\geq 	\|\bm{\mu}_{+}^{P}-\bm{\mu}_{-}^{P}\|_{2}^{2}\int_{\tau_{0}}^{\tau_{1}}\dfrac{t}{(1+4t\sigma^{2})^{(d+2)/2}}\nu_{i}(\mathrm{d}t)\\
	&\geq 	\|\bm{\mu}_{+}^{P}-\bm{\mu}_{-}^{P}\|_{2}^{2}\nu_{i}([\tau_{0},\tau_{1}]) \min\left\{\dfrac{\tau_{0}}{(1+4\tau_{0}\sigma^{2})^{(d+2)/2}},\dfrac{\tau_{1}}{(1+4\tau_{1}\sigma^{2})^{(d+2)/2}}\right\},
\end{align}
where the lower bound holds for any $0<\tau_{0}\leq \tau_{1} \leq {1\over \|\bm{\mu}_{+}^{P}-\bm{\mu}_{-}^{P}\|_{2}^{2}-4\sigma^{2}}=\tau_{P}$. Notice that $\tau_{P}>0$ iff (C.2) is satisfied.  Therefore, from Eqs. \eqref{Eq:A_lower_bound} and \eqref{Eq:An_upper_bound}, we arrive at the following asymptotic result
\begin{align}
\label{Eq:Expected}
\mathrm{MMD}_{K_{i}}^{2}[P_{+},P_{-}]\asymp \|\bm{\mu}_{+}^{P}-\bm{\mu}_{-}^{P}\|_{2}^{2}=2\sigma^{2}D_{\mathrm{KL}}(P_{+}||P_{-}).
\end{align}
Not surprisingly, the result of Eq. \eqref{Eq:Expected} highlights the fact that MMD, as a measure of the distribution distance with respect to a kernel function, is proportional to $f$-divergences. Since the population MMD is non-negative and the optimization of the weights $\bm{\omega}$ is with respect to the probability simplex $\mathrm{S}_{m}^{+}$, the inequalities $J_{k,d}(\tau_{P})\leq {1\over 2}J_{i,d}(\infty)$ and $J_{\ell,d}(\tau_{Q})\leq {1\over 2}J_{j,d}(\infty)$ for all $k\in [M]\backslash \{i\}$ and $\ell\in [M]\backslash \{j\}$ in Theorem \ref{Thm:Min-Max_Rate} are sufficient conditions to ensure that the optimal solution $\bm{\omega}$ of the population MMD optimization in Eq. \eqref{Eq:damn} is a basis vector whose only non-zero coordinate correspond to the basis kernel that yields the largest MMD value. Mathematically,
\begin{subequations}
\begin{align}
\bm{\omega}&=\bm{e}_{i},\quad  \text{for}\ (P_{+},P_{-})\in \mathcal{P}_{\pm} \\
\bm{\omega}&=\bm{e}_{j},\quad  \text{for}\ (Q_{+},Q_{-})\in \mathcal{P}_{\pm},
\end{align}
\end{subequations}
where $\bm{e}_{1},\cdots,\bm{e}_{m}\in \real^{m}$ are the standard basis vectors. Notice that $(\Theta,d)$ here is the simplex indowed with the total variation distance $(\mathrm{S}_{m}^{+},d_{\mathrm{TV}})$. Thus,
\begin{align}
d_{\mathrm{TV}}(\bm{e}_{i},\bm{e}_{j})={1\over 2}\|\bm{e}_{i}-\bm{e}_{j}\|_{1}=1.
\end{align}
From LeCam's minmax rate in Eq. \ref{Eq:LeCam} of Thm. \ref{Thm:LeCam}, we obtain
\begin{align}
\inf_{\widehat{\bm{\omega}}\in \mathrm{S}_{m}^{+}} \sup_{P\in \mathcal{P}} \expect_{P}\Big[d_{\mathrm{TV}}\Big(\widehat{\bm{\omega}}_{n},\bm{\omega}(P)\Big)\Big]\geq \dfrac{1}{8}e^{-nD_{\mathrm{KL}}(P_{+}\otimes P_{-}||Q_{+}\otimes Q_{-})}\geq \dfrac{1}{8}e^{-nr}.
\end{align}

\subsection{Proof of Proposition \ref{Proposition:1}}
\label{Appendix:Proof_of_RGComplexities}

Recall the definition of the function class $\mathcal{F}_{m}(R)$ from Eq. \eqref{Eq:Function_Class}.  We prove Part ($i$) by computing an upper bound on the empirical Rademacher complexity for the class of functions in $\mathcal{F}_{m}(R)$ as follows

\begin{align}
	\nonumber
	\widehat{\mathfrak{R}}_{S}^{n}(\mathcal{F}_{m}(R))&=\dfrac{1}{n}\expect_{P_{\bm{\varepsilon}}}\left[\sup_{f\in \mathcal{F}_{m}(R)}\sum_{i=1}^{n}\varepsilon_{i}f(\bm{x}_{i})\Bigg \vert \bm{x}_{1},\cdots,\bm{x}_{n} \right]\\  \nonumber
	&=\dfrac{1}{n\sqrt{N}}\expect_{P_{\bm{\varepsilon}}}\left[\sup_{\bm{\beta}\in \ball^{2}_{r}(\bm{0}) }\sum_{i=1}^{n}\sum_{k=1}^{N}\sum_{m=1}^{M}\varepsilon_{i}\beta_{m}^{k}\sqrt{\omega_{m}}\phi_{m}(\bm{x}_{i};\bm{\xi}_{m}^{k}) \Bigg \vert \bm{x}_{1},\cdots,\bm{x}_{n} \right]\\  \label{Eq:Supremum}
	&=\dfrac{1}{n\sqrt{N}}\expect_{P_{\bm{\varepsilon}}}\left[\sup_{\bm{\beta}\in \ball^{2}_{r}(\bm{0})}\bm{\varepsilon}^{T}\bm{\Phi}(\bm{\omega})\bm{\beta} \Bigg \vert \bm{x}_{1},\cdots,\bm{x}_{n} \right],
	\end{align}
where in the last equality, $\bm{\Phi}(\bm{\omega})\df (\bm{\varphi}_{\bm{\omega}}(\bm{x}_{i}))_{1\leq i\leq n}\in \real^{n\times mN}$ is the feature matrix with and $r={R\over \sqrt{mN}}$.
From the Cauchy-Schwarz inequality, it is easy to see that the supremum in Eq. \eqref{Eq:Supremum} is attained when $\bm{\beta}$ and $\bm{\varepsilon}^{T}\bm{\Phi}$ are co-linear, \textit{i.e.}, 
\begin{align}
\bm{\beta}= \dfrac{R}{\sqrt{mN}}\cdot {\bm{\varepsilon}^{T}\bm{\Phi}(\bm{\omega})\over \|\bm{\varepsilon}^{T}\bm{\Phi}(\bm{\omega})\|_{2}}. 
\end{align}
Therefore, from Eq. \eqref{Eq:Supremum} we obtain that
\begin{align}
\label{Eq:Expectation_In}
\widehat{\mathfrak{R}}_{S}^{n}(\mathcal{F}_{m}(R))&=\dfrac{R}{nN\sqrt{m}}\expect_{P_{\bm{\varepsilon}}}\left[\|\bm{\varepsilon}^{T}\bm{\Phi}(\bm{\omega})\|_{2} \Bigg \vert \bm{x}_{1},\cdots,\bm{x}_{n} \right].
\end{align}

%
%

We now compute the expectation \eqref{Eq:SUBSUBSUB_1} via the following concentration inequality for the quadratic forms of independent random variables with bounded moments:

\begin{theorem}\textsc{(Concentration Inequality for Quadratic Forms,  \cite[Thm. 3]{bellec2014concentration})}
	\label{Thm:Hanson_Wright_Inequality}	
	Let $\bm{z}=(z_{1},\cdots, z_{n})\in \real^{n}$ be a random vector satisfying the following  Bernstein moment conditions for some $K > 0$:
	\begin{align}
	\label{Eq:Moment_Condition}
	\expect|z_{i}|^{2p}\leq {1\over 2}p!b_{i}^{2}K^{2(p-1)}, \quad p\geq 1.
	\end{align}
	Let $\bm{A}\in \real^{n\times n}$ be a real matrix. Then, for every $t\geq 0$,
	\begin{align}
	\label{Eq:Hanson_Wright_Inequality}
	\prob\Big\{|\bm{z}^{T}\bm{A}\bm{z}-\expect[\bm{z}^{T}\bm{A}\bm{z}]|>t\Big\} 
	\leq \exp\left(-\min\left( \dfrac{t^{2}}{192K^{2}\|\bm{A}\bm{D}_{\bm{b}}\|_{F}^{2}},\dfrac{t}{256K^{2}\|\bm{A}\|_{2}}\right) \right),
	\end{align}	
	where  $\bm{D}_{\bm{b}}=\mathrm{diag}(b_{1},\cdots,b_{n})$, and $\|\bm{A}\|_{F}\df \sum_{i,j=1}^{n}|A_{ij}|^{2}$. \hfill $\square$
\end{theorem}
\begin{remark}
	Despite some resemblance between Inequality \eqref{Eq:Hanson_Wright_Inequality} and the Hanson-Wright Inequality \cite{rudelson2013hanson}, they are different in that the former does not have an implicit universal constant in the upper bound. Indeed, the concentration inequality in Eq. \eqref{Eq:Hanson_Wright_Inequality} is sharper than the Hanson-Wright Inequality, albeit under the additional Bernestein's moment conditions in Eq. \eqref{Eq:Moment_Condition}.
\end{remark}
Now, let us apply the concentration inequality \eqref{Eq:Hanson_Wright_Inequality} in Theorem \ref{Thm:Hanson_Wright_Inequality} to compute the expectation in Eq. \eqref{Eq:Expectation_In}. First, we show that Rademacher random variable $\varepsilon\sim P_{\varepsilon}=\mathrm{Uniform}\{-1,1\}$ satisfies the moment conditions \eqref{Eq:Moment_Condition}. In particular, for all $p\geq 1$, we have $\expect|\varepsilon_{i}|^{2p}=1$. Therefore, the moment condition in Eq. \eqref{Eq:Moment_Condition} is satisfied with $b_{i}=1$ for all $i=1,2,\cdots,n$, and $K=1$. 

Let $\bm{A}=\bm{\Phi}(\bm{\omega})\bm{\Phi}(\bm{\omega})^{T}$ and $\bm{z}=\bm{\varepsilon}$. Then, $\bm{z}^{T}\bm{A}\bm{z}=\bm{\varepsilon}^{T}\bm{\Phi}(\bm{\omega})\bm{\Phi}^{T}(\bm{\omega})\bm{\varepsilon}=\|\bm{\varepsilon}^{T}\bm{\Phi}\|_{2}^{2}$,  $\expect[\bm{z}^{T}\bm{A}\bm{z}]=\expect[\bm{\varepsilon}^{T}\bm{\Phi}(\bm{\omega})\bm{\Phi}^{T}(\bm{\omega})\bm{\varepsilon}]=\|\bm{\Phi}(\bm{\omega})\|_{F}^{2}$. Furthermore, $\bm{D}_{\bm{b}}=\bm{I}_{n}$.

For simplicity of notation, in the sequel we suppress the dependency of the random feature matrix to the mixing coefficients, \textit{i.e.}, $\bm{\Phi}=\bm{\Phi}(\bm{\omega})$. The concentration inequality in Eq. \eqref{Eq:Hanson_Wright_Inequality} turns into
\begin{align}
\prob\Big\{\|\bm{\varepsilon}^{T}\bm{\Phi}\|^{2}_{2}-\|\bm{\Phi}\|_{F}^{2}>t | \bm{x}_{1},\cdots,\bm{x}_{n}\Big\}\leq \exp\left(-\min\left(\dfrac{t^{2}}{192\|\bm{\Phi}\bm{\Phi}^{T}\|_{F}^{2}},\dfrac{t}{256\|\bm{\Phi}\|_{2}} \right) \right).
\end{align}
Also, note that since all Rademacher random variables $\varepsilon_{i}$ have unit variance, we have $K \geq  2^{-{1\over 2}}$. Thus we obtain for any $u\geq  0$ that
\begin{align}
\nonumber
\prob\Big\{\|\bm{\varepsilon}^{T}\bm{\Phi}\|^{2}_{2}-\|\bm{\Phi}\|_{F}^{2}>u|\bm{x}_{1},\cdots,\bm{x}_{n}\Big\}\leq \exp\left(-\min\left(\dfrac{u^{2}}{192\|\bm{\Phi}\bm{\Phi}^{T}\|_{F}^{2}},\dfrac{u}{256\|\bm{\Phi}\|_{2}^{2}} \right)\right). 
\end{align}
Let $\delta \geq 0$ be arbitrary, and let us use this estimate for $u=\delta \|\bm{\Phi}\|^{2}_{F}$. Since $\|\bm{\Phi}^{T}\bm{\Phi}\|_{F}^{2} \leq \|\bm{\Phi}^{T}\|_{2}^{2}\|\bm{\Phi}\|_{F}^{2}=\|\bm{\Phi}\|_{2}^{2}\|\bm{\Phi}\|_{F}^{2}$, it follows that
\begin{align}
\nonumber
\prob\Big\{\|\bm{\varepsilon}^{T}\bm{\Phi}\|_{2}^{2}-\|\bm{\Phi}\|_{F}^{2}\geq \delta \|\bm{\Phi}\|_{F}^{2}|\bm{x}_{1},\cdots,\bm{x}_{n}]\Big\} &\leq \exp\left(- \min\left({\delta\over 192},{\delta^{2}\over 256}\right)\dfrac{\|\bm{\Phi}\|_{F}^{2}}{\|\bm{\Phi}\|_{2}^{2}} \right)\\ \label{Eq:EXO}
&\leq  \exp\left(-{1\over 192} \min\left(\delta,\delta^{2}\right)\dfrac{\|\bm{\Phi}\|_{F}^{2}}{\|\bm{\Phi}\|_{2}^{2}} \right).
\end{align}
Now let $\epsilon \geq 0$ be arbitrary; we shall use this inequality for $\delta = \max(\epsilon,\epsilon^{2})$. Observe
that the (likely) event $|\|\bm{\varepsilon}^{T}\bm{\Phi}\|_{2}^{2}-\|\bm{\Phi}\|_{F}^{2}|\leq \delta \|\bm{\Phi}\|_{F}^{2}$ implies the event $|\|\bm{\varepsilon}^{T}\bm{\Phi}\|_{2}-\|\bm{\Phi}\|_{F}|\leq \epsilon \|\bm{\Phi}\|_{F}$. This can be seen by dividing both sides of the inequalities by $\|\bm{\Phi}\|_{F}^{2}$ and $\|\bm{\Phi}\|_{F}$ respectively, and using the numeric bound $\max(|z-1|, |z- 1|^2) \leq |z^2- 1|$, which is valid for all $z \geq 0$. Using this observation along with the identity $\min\left({\delta},{\delta^{2}}\right)= \epsilon^{2}$, we arrive from Eq. \eqref{Eq:EXO},
\begin{align}
\label{Eq:EXO1}
\prob\Big\{\|\bm{\varepsilon}^{T}\bm{\Phi}\|_{2}-\|\bm{\Phi}\|_{F}\geq \epsilon \|\bm{\Phi}\|_{F}|\bm{x}_{1},\cdots,\bm{x}_{n}\Big\}\leq \exp\left(-{1\over 192} \epsilon^{2}\dfrac{\|\bm{\Phi}\|_{F}^{2}}{\|\bm{\Phi}\|_{2}^{2}} \right).
\end{align}
Letting $\epsilon={\varrho\over \|\bm{\Phi}\|_{F}}$ the yields
\begin{align}
\label{Eq:EXO2}
\prob\Big\{\|\bm{\varepsilon}^{T}\bm{\Phi}\|_{2}-\|\bm{\Phi}\|_{F}\geq \varrho|\bm{x}_{1},\cdots,\bm{x}_{n}\Big\} \leq \exp\left(- {1\over 192}\dfrac{\varrho^{2}}{\|\bm{\Phi}\|_{2}^{2}} \right).
\end{align}
Now, we return to computing the expectation in Eq. \eqref{Eq:Expectation_In}. For the positive random variable $\|\bm{\varepsilon}^{T}\bm{\Phi}\|_{2}$, its expectation can be computed as follows
\begin{align}
\nonumber
\widehat{\mathfrak{R}}_{S}^{n}(\mathcal{F}_{m}(R))&=\dfrac{R}{nN\sqrt{m}} \expect_{P_{\bm{\varepsilon}}}\left[\|\bm{\varepsilon}^{T}\bm{\Phi}\|_{2}\Bigg\vert \bm{x}_{1},\cdots,\bm{x}_{n}\right]\\ \nonumber
&=\dfrac{R}{nN\sqrt{m}}\int_{-\|\bm{\Phi}\|_{F}}^{\infty}\prob[\|\bm{\varepsilon}^{T}\bm{\Phi}\|_{2}-\|\bm{\Phi}\|_{F}>t|\bm{x}_{1},\cdots,\bm{x}_{n}]\mathrm{d} t\\ \nonumber
&\stackrel{\rm(a)}{\leq} \dfrac{R}{nN\sqrt{m}} \int_{-\|\bm{\Phi}\|_{F}}^{\infty} \exp\left(- {1\over 192}\dfrac{t^{2}}{\|\bm{\Phi}\|_{2}^{2}} \right)\mathrm{d}t \\ \label{Eq:Maximize}
&= \dfrac{R}{nN\sqrt{m}} \sqrt{192\pi}{\|\bm{\Phi}\|_{2}}Q\left( -\sqrt{192}\dfrac{\|\bm{\Phi}\|_{F} }{\|{\bm{\Phi}}\|_{2} }  \right),
\end{align}
where $\rm{(a)}$ follows from the concentration inequality in Eq. \eqref{Eq:EXO2}.  Using the basic inequality $\|\bm{\Phi}\|_{F}\leq \sqrt{\text{Rank}(\bm{\Phi})} \|\bm{\Phi}\|_{2}$, with $\text{Rank}(\bm{\Phi})=\min\{mN,n\}$ in conjunction with the fact that $\mathrm{erfc}(x)$ is monotone decreasing for all  $x\in \real$ yields the upper bound \eqref{Eq:Rademacher} on the Rademacher complexity.

We remark that the upper bound we derived here is sharper than that of \cite[Lemma 1 ]{cortes2009new} which is based on the Khintchine-Kahane type inequality. More specifically, the Khintchine inequality yields the following upper bound

{\small\begin{align}
	\nonumber
	\widehat{\mathfrak{R}}_{S}^{n}(\mathcal{F}_{m}(R))&= \dfrac{R}{nN\sqrt{m}}\expect_{P_{\bm{\varepsilon}}}\left[\sqrt{\|\bm{\varepsilon}^{T}\bm{\Phi}\bm{\Phi}^{T}\bm{\varepsilon}} \Bigg \vert \bm{x}_{1},\cdots,\bm{x}_{n} \right]\\ \nonumber
	&\leq \dfrac{R}{nN\sqrt{m}}\left(\dfrac{23}{44}\text{Tr}(\bm{\Phi}\bm{\Phi}^{T}) \right)^{1\over 2}\\ \label{Eq:SUBSUBSUB_1}
	&=\dfrac{R}{nN\sqrt{m}}\sqrt{\dfrac{23}{44} } \|\bm{\Phi}\|_{F}.
	\end{align}\normalsize}
Since $0<\mathrm{erfc}(x)\leq 1$ for $0<x$, and $\|\bm{\Phi}\|_{2}\leq \|\bm{\Phi}\|_{F}$ for any matrix $\bm{\Phi}\in \real^{n\times N}$, the upper bound we established in Eq. \eqref{Eq:Maximize} is sharper than that of Eq. \eqref{Eq:SUBSUBSUB_1} using techniques of  \cite[Lemma 1]{cortes2009new}.  We now prove Part (ii) of Proposition \ref{Proposition:1}. Similar to the derivation in Eq. \eqref{Eq:Supremum}, we have
that

{\small\begin{align}
	\nonumber
	\widehat{\mathfrak{G}}_{S}^{n}(\mathcal{F}_{m}(R))&=\dfrac{1}{n}\expect_{P_{\bm{\sigma}}}\left[\sup_{f\in \mathcal{F}_{m}(R)}\sum_{i=1}^{n}\sigma_{i}f(\bm{x}_{i})\Bigg \vert \bm{x}_{1},\cdots,\bm{x}_{n} \right]\\  \nonumber
	&=\dfrac{R}{nN\sqrt{m}}\expect_{P_{\bm{\sigma}}}\left[\|\bm{\sigma}^{T}\bm{\Phi}\|_{2} \Bigg \vert \bm{x}_{1},\cdots,\bm{x}_{n} \right].
	\end{align}}
To compute the expectation, we use the following standard tail bound for the sum of $\chi^{2}$ random variables due to Laurent and Massart \cite{laurent2000adaptive}:

\begin{theorem}\textsc{(Tail Bound for $\chi^{2}$ random variables \cite{laurent2000adaptive})} Let $z_{1},\cdots,z_{n}$ be independent $\chi^{2}$ random variables, each with one degree of freedom. For any vector $\bm{a}\df (a_{1},\cdots,a_{n})\in \real_{+}^{n}$ with non-negative entries, and any $t>0$,
	\begin{align}
	\label{Eq:Concenteration_Inequality}
	\prob\left\{\sum_{i=1}^{n}a_{i}z_{i}\geq \|\bm{a}\|_{1}+2\|\bm{a}\|_{2}\sqrt{t}+2\|\bm{a}\|_{\infty}t  \right\}\leq e^{-t}.
	\end{align}
	\hfill $\square$
\end{theorem}

Now, consider the eigenvalue decomposition $\bm{\Phi}\bm{\Phi}^{T}=\bm{V}\bm{\Lambda}\bm{V}$, where $\bm{V}$ is a matrix of orthonormal eigenvectors, and $\bm{\Lambda} \df \mathrm{diag}(\lambda_{1}, \cdots,\lambda_{n})$ is the diagonal matrix of the corresponding eigenvalues $\bm{\lambda}\df (\lambda_{1},\cdots,\lambda_{n})$. Due to the rotational invariance of the Gaussian distribution, $\bm{z}\df \bm{\sigma}\bm{V}$ is a zero-mean isotropic multi-variate Gaussian random vector. Thus, $\|\bm{\sigma}^{T}\bm{\Phi}\|_{2}= \bm{z}^{T}\bm{\bm{\Lambda}}\bm{z}= \lambda_{1}z_{1}^{2}+\cdots+\lambda_{n}z_{n}^{2}$, and the $z_{i}^{2}$'s are independent $\chi^{2}$ random variables, each with one degree of freedom. Therefore, using the concentration inequality in Eq. \eqref{Eq:Concenteration_Inequality} with $\|\bm{\lambda}\|_{2}=(\mathrm{Tr}((\bm{\Phi}\bm{\Phi}^{T})^{2}))^{1\over 2}$, $\|\bm{\lambda}\|_{1}=\mathrm{Tr}(\bm{\Phi}\bm{\Phi}^{T})=\|\bm{\Phi}\|_{F}^{2}$, and $\|\bm{\lambda}\|_{\infty}=\|\bm{\Phi}\|_{2}^{2}$ yields for all $t>0$,
\begin{align}
	\label{Eq:Concentration_Inequality_1}
	&\prob\left\{\|\bm{\sigma^{T}}\bm{\Phi}\|_{2}^{2}\geq \|\bm{\Phi}\|_{F}^{2}+\text{Tr}^{1\over 2}((\bm{\Phi}\bm{\Phi}^{T})^{2})\sqrt{t} +2\|\bm{\Phi}\|_{2}^{2}t \Bigg|\bm{x}_{1},\cdots,\bm{x}_{n}  \right\}
	\leq e^{-t}.
	\end{align}
Alternatively, by letting 
\begin{align}
t=Q_{\varepsilon}(\bm{\Phi}) \df\left(\sqrt{{\varepsilon\|\bm{\Phi}\|_{F}^{2} \over 2\|{\bm{\Phi}\|_{2}^{2}}}+{\mathrm{Tr}((\bm{\Phi}\bm{\Phi}^{T})^{2}) \over 4\|{\bm{\Phi}\|_{2}^{4}}}}-{\sqrt{\mathrm{Tr}((\bm{\Phi}\bm{\Phi}^{T})^{2})} \over 2\|{\bm{\Phi}\|_{2}^{2}}}  \right)^{2},
\end{align}
the concentration bound in Eq. \eqref{Eq:Concentration_Inequality_1} takes the following form
\begin{align}
\nonumber
&\prob\Big\{\|\bm{\sigma^{T}}\bm{\Phi}\|_{2}^{2}-\|\bm{\Phi}\|_{F}^{2}\geq \varepsilon\|\bm{\Phi}\|_{F}^{2} \Big\} 
\leq e^{-Q_{\varepsilon}(\bm{\Phi})}.
\end{align}
Due to the inequality $\max\{|z-1|,|z-1|^{2}\}\leq |z^{2}-1|,z>0$, the preceding inequality in turn implies that
\begin{align}
\nonumber
\prob\Big\{\|\bm{\sigma^{T}}\bm{\Phi}\|_{2}-\|\bm{\Phi}\|_{F}\geq \varepsilon \|\bm{\Phi}\|_{F} \Big\}  \leq e^{-Q_{\varepsilon}(\bm{\Phi})}.
\end{align}
Letting $\varepsilon=\varrho/\|\bm{\Phi}\|_{F}$ yields
\begin{align}
\label{Eq:Using}
\prob\Big\{\|\bm{\sigma^{T}}\bm{\Phi}\|_{2}-\|\bm{\Phi}\|_{F}\geq \varrho  \Big\}
\leq \exp\left({-Q_{{\varrho\over \|\bm{\Phi}\|_{F}}}(\bm{\Phi})}\right).
\end{align}
We now leverage Inequality \eqref{Eq:Using}. Then, similar to our derivation in Eq. \eqref{Eq:Maximize}, we obtain that 
\begin{align}
	\nonumber
	\widehat{\mathfrak{G}}_{S}^{n}(\mathcal{F}_{m}(R))&=\dfrac{R}{nD\sqrt{m}}\int_{0}^{\infty}\prob[\| \bm{\sigma}^{T}\bm{\Phi}\|_{2}>t|\bm{x}_{1},\cdots,\bm{x}_{n}]\mathrm{d}t \\ \label{Eq:Pluggg_2}
	&\leq  \dfrac{R}{nD\sqrt{m}}\int_{-\|\bm{\Phi}\|_{F}}^{\infty}\hspace{-4mm}\exp\left({-Q_{{t\over \|\bm{\Phi}\|_{F}}}(\bm{\Phi})}\right)\mathrm{d}t.
	\end{align}
Due to the elementary inequality $\exp\left(-(\sqrt{-a+b^{2}}-b)^{2}\right)\leq \exp\left(-{a^{2}\over 4b^{2}}\right)$ for all $a,b>0$, we obtain that
\begin{align}
\label{Eq:Pluggg_1}
\exp\left({-Q_{{t\over \|\bm{\Phi}\|_{F}}}(\bm{\Phi})}\right)\leq \exp\left(-\dfrac{\|\bm{\Phi}\|_{F}^{2}}{4 \mathrm{Tr}((\bm{\Phi}\bm{\Phi}^{T})^{2})}t^{2} \right).
\end{align}
Plugging Eq. \eqref{Eq:Pluggg_1} into Eq. \eqref{Eq:Pluggg_2} yields
\begin{align}
\nonumber
\widehat{\mathfrak{G}}_{S}^{n}(\mathcal{F}_{m}(R))&\leq  \dfrac{R}{nN\sqrt{m}}\int_{-\|\bm{\Phi}\|_{F}}^{\infty}\hspace{-4mm}\exp\left(-\dfrac{\|\bm{\Phi}\|_{F}^{2}}{4 \mathrm{Tr}((\bm{\Phi}\bm{\Phi}^{T})^{2})}t^{2} \right)\mathrm{d}t\\ \label{Eq:Consequently}
&=\dfrac{R}{nN\sqrt{m}}\dfrac{\sqrt{\pi}}{\sqrt{2}} \dfrac{\mathrm{Tr}^{1\over 2}\left((\bm{\Phi}\bm{\Phi}^{T})^{2} \right)}{\|\bm{\Phi}\|_{F}}\mathrm{erfc}\left(-\dfrac{\|\bm{\Phi}\|_{F}^{2}}{2\mathrm{Tr}^{1\over 2}\left((\bm{\Phi}\bm{\Phi}^{T})^{2} \right)}  \right).
\end{align}
Now, recall that $\|\bm{\lambda}\|_{2}=(\mathrm{Tr}((\bm{\Phi}\bm{\Phi}^{T})^{2}))^{1\over 2}$, and $\|\bm{\lambda}\|_{1}=\mathrm{Tr}(\bm{\Phi}\bm{\Phi}^{T})=\|\bm{\Phi}\|_{F}^{2}$. Then, since $\|\bm{\lambda}\|_{2}\leq \|\bm{\lambda}\|_{1}\leq \sqrt{n}\|\bm{\lambda}\|_{2}$, we obtain that 
\begin{align}
(\mathrm{Tr}((\bm{\Phi}\bm{\Phi}^{T})^{2}))^{1\over 2} \leq \|\bm{\Phi}\|_{F}^{2} \leq \sqrt{n}(\mathrm{Tr}((\bm{\Phi}\bm{\Phi}^{T})^{2}))^{1\over 2}.
\end{align}
Consequently, the upper bound in Eq. \eqref{Eq:Consequently} can be simplified as follows
\begin{align}
\widehat{\mathfrak{G}}_{S}^{n}(\mathcal{F}_{m}(R))\leq \dfrac{R}{nN\sqrt{m}}{\sqrt{\pi\over 2}}\|\bm{\Phi}\|_{F}\mathcal{Q}\left(-\sqrt{{n\over 4}} \right).
\end{align}

\subsection{Proof of Lemma \ref{Lemma:Phi_Norm}}
\label{App:Proof_of_Lemma_Phi_Norm}

We compute a concentration bound for the spectral norm $\|\bm{\Phi}(\bm{\omega})\|_{2}$ of the feature matrix. To achieve this goal, we define the estimation error matrix $\bm{\Delta}_{n}(\bm{\omega})\df  \bm{K}({\bm{\omega}})-{1\over N}(\bm{\Phi}\bm{\Phi}^{T})(\bm{\omega})$. We now use the standard $\varepsilon$-net argument due to \cite{vershynin2010introduction}:
\begin{lemma}\textsc{(Spectral Norm on a Net, \cite[Lemma 5.4.]{vershynin2010introduction})}
	\label{Lemma:Spectral_Norm}
	Let $\bm{A}$ be a symmetric $n\times n$ matrix, and let $N_{\varepsilon}$ be an $\varepsilon$-net of $\mathrm{S}^{n-1}$ for some $\varepsilon \in [0, 1)$. Then
	\begin{align}
	\label{Eq:Inequality}
	\hspace{-4mm}	\|\bm{A}\|_{2}=\sup_{\bm{y}\in \mathrm{S}^{n-1}}|\langle \bm{A}\bm{y},\bm{y} \rangle |\leq (1-2\varepsilon)^{-1}\sup_{\bm{y}\in N_{\varepsilon}}|\langle \bm{A}\bm{y},\bm{y} \rangle|.
	\end{align}
	\hfill $\square$
\end{lemma}
Let $N_{\varepsilon}$ be $\varepsilon$-covering of the sphere $\mathrm{S}^{n-1}$. It is shown in \cite{vershynin2010introduction} that $\log N_{\varepsilon}\leq n\log(1+{2\over \varepsilon})$. From Inequality \eqref{Eq:Inequality}, we obtain that 
\begin{align}
\label{Eq:inner_product}
\|\bm{\Delta}_{n}(\bm{\omega})\|_{2}&\leq (1-2\varepsilon)^{-1} \sup_{\bm{y}\in N_{\varepsilon}}|\langle \bm{\Delta}_{n}(\bm{\omega})\bm{y},\bm{y}  \rangle|.
\end{align}
For each point in the cover $\bm{y}\in N_{\varepsilon}$, the inner product in Eq. \eqref{Eq:inner_product} has the following upper bound
\begin{align}
	\nonumber
	|\langle \bm{\Delta}_{n}(\bm{{\omega}})\bm{y},\bm{y} \rangle|&=\left|\left\langle \bm{K}({\bm{\omega}})-{1\over N}\bm{\Phi}\bm{\Phi}^{T}(\bm{\omega})\bm{y},\bm{y} \right\rangle\right|\\ \nonumber
	&= \left|\sum_{i,j=1}^{n}\left((\bm{K}({\bm{\omega}}))_{ij}-{1\over N}(\bm{\Phi}\bm{\Phi}^{T}(\bm{\omega}))_{ij}\right) y_{i}y_{j} \right|   \\ \nonumber
	&=\left| \sum_{i,j=1}^{n}\Bigg(\sum_{\ell=1}^{m}\omega_{\ell}K_{\ell}(\bm{x}_{i},\bm{x}_{j})-{1\over N}\sum_{k=1}^{N}\sum_{\ell=1}^{m}\omega_{\ell}\varphi_{\ell}(\bm{x}_{i};\bm{\xi}_{\ell}^{k})\varphi_{\ell}(\bm{x}_{j};\bm{\xi}^{k}_{\ell})\Bigg) y_{i}y_{j} \right|\\ \label{Eq:In_Conjunction_2}
	&\leq  \sum_{i,j=1}^{n}|y_{i}y_{j}|\sum_{\ell=1}^{m}\omega_{\ell}\left|K_{\ell}(\bm{x}_{i},\bm{x}_{j})-{1\over N}\sum_{k=1}^{N}\varphi_{\ell}(\bm{x}_{i};\bm{\xi}^{k}_{\ell})\varphi_{\ell}(\bm{x}_{j};\bm{\xi}^{k}_{\ell})\right|. 
	\end{align}
By applying the Cauchy-Schwarz inequality, we obtain that
\begin{align}
\nonumber
&\sum_{i,j=1}^{n}|y_{i}y_{j}|\sum_{\ell=1}^{m}\omega_{\ell}\left|K_{\ell}(\bm{x}_{i},\bm{x}_{j})-{1\over N}\sum_{k=1}^{N}\varphi_{\ell}(\bm{x}_{i};\bm{\xi}^{k}_{\ell})\varphi_{\ell}(\bm{x}_{j};\bm{\xi}^{k}_{\ell})\right| \\ \nonumber
&\leq \left(\sum_{i,j=1}^{n} \left(\sum_{\ell=1}^{m}\omega_{\ell}\left|K_{\ell}(\bm{x}_{i},\bm{x}_{j})-{1\over N}\sum_{k=1}^{N}\varphi_{\ell}(\bm{x}_{i};\bm{\xi}^{k}_{\ell})\varphi_{\ell}(\bm{x}_{j};\bm{\xi}^{k}_{\ell})\right|\right)^{2}\right)^{1\over 2} \cdot  \left(\sum_{i,j=1}^{n}|y_{i}y_{j}|^{2} \right)^{1\over 2}.
\end{align}	
Therefore, from Eq. \eqref{Eq:In_Conjunction_2} we proceed as follows
\begin{align}
\nonumber
|\langle \bm{\Delta}_{n}(\bm{{\omega}})\bm{y},\bm{y} \rangle|&\leq  \left(\sum_{i,j=1}^{n} \left(\sum_{\ell=1}^{m}\omega_{\ell}\left|K_{\ell}(\bm{x}_{i},\bm{x}_{j})-{1\over N}\sum_{k=1}^{N}\varphi_{\ell}(\bm{x}_{i};\bm{\xi}^{k}_{\ell})\varphi_{\ell}(\bm{x}_{j};\bm{\xi}^{k}_{\ell})\right|\right)^{2}\right)^{1\over 2} \|\bm{y}\|_{2}^{2}\\
&=\left(\sum_{i,j=1}^{n} \left(\sum_{\ell=1}^{m}\omega_{\ell}\left|K_{\ell}(\bm{x}_{i},\bm{x}_{j})-{1\over N}\sum_{k=1}^{N}\varphi_{\ell}(\bm{x}_{i};\bm{\xi}^{k}_{\ell})\varphi_{\ell}(\bm{x}_{j};\bm{\xi}^{k}_{\ell})\right|\right)^{2}\right)^{1\over 2},
\end{align}
where the last equality is due to the fact that $\bm{y}\in \mathrm{S}^{n-1}$, and thus $\|\bm{y}\|_{2}=1$. Due to the fact that $\ell_{2}$-norm of a matrix is smaller than its $\ell_{1}$-norm, we can obtain the following  upper bound
\begin{align}
\nonumber
|\langle \bm{\Delta}_{n}(\bm{{\omega}})\bm{y},\bm{y} \rangle|
&\leq \sum_{i,j=1}^{n}\sum_{\ell=1}^{m}\omega_{\ell}\left|K_{\ell}(\bm{x}_{i},\bm{x}_{j})-{1\over N}\sum_{k=1}^{N}\varphi_{\ell}(\bm{x}_{i};\bm{\xi}_{\ell}^{k})\varphi_{\ell}(\bm{x}_{j};\bm{\xi}_{\ell}^{k})\right|.
\end{align}
Applying the Cauchy-Schwarz inequality yields
\begin{align}
\nonumber
|\langle \bm{\Delta}_{n}(\bm{{\omega}})\bm{y},\bm{y} \rangle|&\leq n^{2}\cdot \|\bm{\omega}\|_{2} \cdot\left(\sum_{\ell=1}^{m} \left(\left|K_{\ell}(\bm{x},\bm{y})-{1\over N}\sum_{k=1}^{N}\varphi_{\ell}(\bm{x};\bm{\xi}_{\ell}^{k})\varphi_{\ell}(\bm{y};\bm{\xi}_{\ell}^{k})\right|\right)^{2}\right)^{1\over 2}\\ \nonumber
&\leq \|\bm{\omega}\|_{1}\sum_{\ell=1}^{m}\sum_{i,j=1}^{n}\left|K_{\ell}(\bm{x}_{i},\bm{x}_{j})-{1\over N}\sum_{k=1}^{N}\varphi_{\ell}(\bm{x}_{i};\bm{\xi}_{\ell}^{k})\varphi_{\ell}(\bm{x}_{j};\bm{\xi}_{\ell}^{k})\right|\\
&=\sum_{\ell=1}^{m}\sum_{i,j=1}^{n}\left|K_{\ell}(\bm{x}_{i},\bm{x}_{j})-{1\over N}\sum_{k=1}^{N}\varphi_{\ell}(\bm{x}_{i};\bm{\xi}_{\ell}^{k})\varphi_{\ell}(\bm{x}_{j};\bm{\xi}_{\ell}^{k})\right|.
\end{align}
where in the last inequality, we used the fact that $\bm{\omega}\in \mathrm{S}_{m}^{+}$, and thus $\|\bm{\omega}\|_{1}=1$. Taking the union bound over $N_{\varepsilon}$ with $\varepsilon=1/4$ results in
\begin{align}
	\nonumber
	&\prob\left(\|{\bm{\Delta}_{n}}(\bm{\omega})\|_{2}\geq {\delta\over N} \right)\leq \prob\left(\sup_{\bm{y}\in N_{1/4}}|\langle \bm{\Delta}_{n}(\bm{\omega})\bm{y},\bm{y} \rangle|\geq {\delta\over 2N}\right)\\   \nonumber
	&\leq \exp({3n\log 3})\sup_{\bm{y}\in N_{1/4}}\prob\left(|\langle \bm{\Delta}_{n}(\bm{\omega})\bm{y},\bm{y} \rangle|\geq {\delta\over N}\right)    \\   
	\label{Eq:Nets}
	&\leq \exp({3n\log 3}) \prob\left(\sum_{\ell=1}^{m}\sum_{i,j=1}^{n} \left|K_{\ell}(\bm{x}_{i},\bm{x}_{j})-{1\over N}\sum_{k=1}^{N}\varphi_{\ell}(\bm{x}_{i};\bm{\xi}_{\ell}^{k})\varphi_{\ell}(\bm{x}_{j};\bm{\xi}_{\ell}^{k})\right|\geq \dfrac{\delta}{N}\right).
	\end{align}
Now, define the random variable
\begin{align}
\label{Eq:Zijl}
Z_{ij\ell}\df K_{\ell}(\bm{x}_{i},\bm{x}_{j})-{1\over N}\sum_{k=1}^{N}\varphi_{\ell}(\bm{x}_{i};\bm{\xi}_{\ell}^{k})\varphi_{\ell}(\bm{x}_{j};\bm{\xi}_{\ell}^{k}), \quad \ell=1,2,\cdots,m.
\end{align}
Let $(\bm{\xi}_{\ell}^{1},\cdots,\bm{\xi}_{\ell}^{N})$ and  $(\bm{\xi}_{\ell}^{1},\cdots, \tilde{\bm{\xi}}_{\ell}^{k},\cdots,\bm{\xi}_{\ell}^{N}),k\in\{1,2,\cdots,N\}$ denote two vector of random features that differ in the $k$-th element, and let $Z_{ij\ell}$ and $\widetilde{Z}_{ij\ell}$ denote the associated random variables. Then, by the triangle inequality, we obtain
\begin{align}
\nonumber
|Z_{ij\ell}-\widetilde{Z}_{ij\ell}|&=\dfrac{1}{N}|\varphi_{\ell}(\bm{x}_{i};\bm{\xi}_{\ell}^{k})\varphi_{\ell}(\bm{x}_{j};\bm{\xi}_{\ell}^{k})-\varphi_{\ell}(\bm{x}_{i};\tilde{\bm{\xi}}_{\ell}^{k})\varphi_{\ell}(\bm{x}_{j};\tilde{\bm{\xi}}_{\ell}^{k})|\\ \nonumber
&\leq \dfrac{1}{N}|\varphi_{\ell}(\bm{x}_{i};\bm{\xi}_{\ell}^{k})||\varphi_{\ell}(\bm{x}_{i};\bm{\xi}_{\ell}^{k})|+\dfrac{1}{N} |\varphi_{\ell}(\bm{x}_{i};\widetilde{\bm{\xi}}_{\ell}^{k})||\varphi_{\ell}(\bm{x}_{i};\widetilde{\bm{\xi}}_{\ell}^{k})|\\
&\leq {2L^{2}\over N},
\end{align}
where the last inequality is due to Assumption \ref{Assumption:1}.  Since $\expect[Z_{ij\ell}]=0$ for all $1\leq i,j\leq n$ and $1\leq \ell\leq m$. Applying McDiarmid's Martingale inequality \cite{mcdiarmid1989method} yields
\begin{align}
\prob\left(|Z_{ij\ell}|\geq \delta \right)\leq \exp\left(-\dfrac{N\delta^{2}}{L^{2}} \right).
\end{align}
By Lemma \ref{Lemma:Con_to_Sub_Gaussian} of Appendix \ref{Appendix:Properties of Sub-Exponential and Sub-Gaussian Random Variables}, $Z_{ij\ell}$ is a sub-Gaussian random variable with the Orlicz norm of $\|Z_{ij\ell}\|_{\psi_{2}}\leq{L\over \sqrt{2N}}$.  Therefore, 
\begin{align}
\label{Eq:Concentration_Z}
\prob\left(\sum_{i,j=1}^{n}\sum_{\ell=1}^{m}|Z_{ij\ell}|\geq \delta \right)\leq  \exp\left(-{cN\delta^{2}\over n^{2}mL^{2}} \right),
\end{align}
for some universal constant $c>0$. From Eqs. \eqref{Eq:Nets}, \eqref{Eq:Zijl} we obtain
\begin{align}
\prob\left(\|{\bm{\Delta}_{n}}(\bm{\omega})\|_{2}\geq \delta\right)\leq \exp(3n\log 3)\cdot\exp\left(-{cN\delta^{2}\over n^{2}mL^{2}} \right).
\end{align}
	
Using Weyl's eigenvalue inequality \cite[Corollary III.2.6]{bhatia1997matrix}, the following upper bound can be established on the difference between the spectral norm of the kernel matrix and that of the random feature matrix 
\begin{align}
\left|\lambda_{i}(\bm{K}({\bm{\omega}}))-\lambda_{i}(\bm{\Phi}\bm{\Phi}^{T}(\bm{\omega})/N)\right|\leq \|\bm{\Delta}_{n}(\bm{\omega})\|_{2},
\end{align}
for all $i=1,2,\cdots,n$. Therefore,
\begin{align}
\left|\|\bm{K}({\bm{\omega}})\|_{2}-{1\over N}\|\bm{\Phi}\bm{\Phi}^{T}(\bm{\omega})\|_{2}\right|\leq \|\bm{\Delta}_{n}(\bm{\omega})\|_{2}.
\end{align}
For all $\bm{\omega}\in \mathrm{S}_{m}^{+}$, it then holds that 
\begin{align}
\nonumber
\prob\left(\left|\|\bm{K}({\bm{\omega}})\|_{2}-{1\over N}\|\bm{\Phi}\bm{\Phi}^{T}(\bm{\omega})\|_{2}\right|\geq\delta\right)&\leq \prob\left(\|\bm{\Delta}_{n}(\bm{\omega})\|_{2}\geq\delta\right)\\ \label{Eq:Union_Bound_1}
&\leq \exp(3n\log 3)\exp\left(-{cN\delta^{2}\over n^{2}mL^{2}} \right),
\end{align}
To establish a uniform concentration bound, we use the fact that $\bm{\omega}\in \mathrm{S}_{m}^{+}$ and the simplex $\mathrm{S}_{m}^{+}$ is compact with the diameter $\mathrm{Diam}(\mathrm{S}_{m}^{+})=\max_{\bm{\omega}_{0},\bm{\omega}_{1}\in \mathrm{S}_{m}^{+}}\|\bm{\omega}_{0}-\bm{\omega}_{1}\|_{2}\leq 2$. Therefore, we can find an $\varepsilon$-net that covers $\mathrm{S}_{m}^{+}$, using at most $M=(4\mathrm{Diam}(\mathrm{S}_{+}^{m})/\varepsilon)^{m}=(8/\varepsilon)^{m}$ balls of the radius $\varepsilon>0$. Lets $\{\bm{\omega}_{i}\}_{i=1}^{M}$ denotes the center of these balls.  Now, define the function $T(\bm{\omega})\df \|\bm{K}(\bm{\omega})\|_{2}-N^{-1}\|\bm{\Phi}\bm{\Phi}^{T}(\bm{\omega})\|_{2}$. We have $|T(\bm{\omega})| < \delta$ for all $\bm{\omega}\in \mathrm{S}_{+}^{m}$ if $|T(\bm{\omega}_{i})|< {\delta\over 2}$ and $L_{T}<{\delta\over 2r}$ for all $i=1,2,\cdots,m$. In the sequel, we bound the probability of these two events.

The union bound followed by Eq. \eqref{Eq:Union_Bound_1} applied to the anchors in the $\varepsilon$-net yields
\begin{align}
\prob\left(\cup_{i=1}^{M}\left|T(\bm{\omega}_{i})\right|\geq {\delta\over 2} \right)\leq \left({8\over r}\right)^{m}\cdot \exp(3n\log 3)\cdot\exp\left(-{cN\delta^{2}\over 4n^{2}mL^{2}} \right).
\end{align} 
The function $T(\bm{\omega})$ is differentiatable, and its Lipschitz constant is defined by $L_{T}\df \|\nabla T(\bm{\omega}_{\ast})\|_{2}$, where $\bm{\omega}_{\ast}\df \arg\min_{\bm{\omega}\in \mathrm{S}_{+}^{m}} \|T(\bm{\omega})\|_{2}$. The Lipschitz constant is a random variable of the random feature samples. Moreover, its second moment has the following upper bound
\begin{align}
\nonumber
\expect\Big[L_{T}^{2}\Big]&=\expect[\|\nabla \|\bm{K}(\bm{\omega})\|_{2}-N^{-1}\nabla \|\bm{\Phi}\bm{\Phi}^{T}(\bm{\omega}) \|_{2}  \|_{2}^{2}]\\ \nonumber
&\stackrel{\mathrm{(a)}}{=} \expect\left[ \sum_{i=1}^{m}(\mathrm{Tr}(\bm{K}_{i}\bm{u}\bm{v}^{T})-N^{-1}\mathrm{Tr}( \bm{\Phi}_{i}\bm{\Phi}_{i}^{T}\tilde{\bm{u}}\tilde{\bm{v}}^{T}))^{2} \right],\\  \nonumber
&= \sum_{i=1}^{m}(\mathrm{Tr}(\bm{K}_{i}\bm{u}\bm{v}^{T}))^{2}+N^{-2}\expect\left[\sum_{i=1}^{m}\left(\mathrm{Tr}(\bm{\Phi}_{i}\bm{\Phi}_{i}^{T}\tilde{\bm{u}}\tilde{\bm{v}}^{T})\right)^{2}\right]\\ \label{Eq:Therefore_1}
&\hspace{4mm}-2N^{-1}\sum_{i=1}^{m}\expect\left[\mathrm{Tr}(\bm{\Phi}_{i}\bm{\Phi}_{i}^{T}\tilde{\bm{u}}\tilde{\bm{v}}^{T}))\right]\mathrm{Tr}(\bm{K}_{i}\bm{u}\bm{v}^{T})),
\end{align}
where in $\rm{(a)}$, $\bm{u}$ and $\bm{v}$ (respectively, $\tilde{\bm{u}}$ and $\tilde{\bm{v}}$) are the left and right singular vectors of $\bm{K}(\bm{\omega}_{\ast})$ (respectively, $\bm{\Phi}\bm{\Phi}^{T}(\bm{\omega}_{\ast})$). Moreover,
\begin{align}
(\nabla\|\bm{K}(\bm{\omega}_{\ast})\|_{2})_{i}=\dfrac{\partial }{\partial \omega_{i}}\|\bm{K}(\bm{\omega}_{\ast})\|_{2}=\dfrac{\partial \bm{K}(\bm{\omega}_{\ast})}{\partial \omega_{i}}{\partial \|\bm{K}(\bm{\omega}_{\ast})\|_{2}\over \partial\bm{K}(\bm{\omega}_{\ast})}= \mathrm{Tr}(\bm{K}_{i}\bm{u}\bm{v}^{T}),
\end{align}
for $i=1,2,\cdots,m$. Due to Assumption \ref{Assumption:1}, we now have that 
\begin{subequations}
\begin{align}
\mathrm{Tr}(\bm{K}_{i}\bm{u}\bm{v}^{T})&\leq \|\bm{K}_{i}\|_{F}\|\bm{u}\bm{v}^{T} \|_{F}\leq nL\\
{1\over N}\mathrm{Tr}(\bm{\Phi}_{i}\bm{\Phi}_{i}^{T}\tilde{\bm{u}}\tilde{\bm{v}}^{T})&\leq\left\|{1\over N}\bm{\Phi}_{i}\bm{\Phi}_{i}^{T}\right\|_{F}\|\tilde{\bm{u}}\tilde{\bm{v}}^{T} \|_{F}\leq nL.
\end{align}
\end{subequations}
Therefore, from Eq. \eqref{Eq:Therefore_1}, we now obtain $\expect[L_{T}^{2}]\leq 3mn^{2}L^{2}$. By Markov's inequality, we then get
\begin{align}
\label{Eq:Union_2}
\prob\left(L_{T}^{2}\geq {\delta\over 2r} \right)\leq \dfrac{\expect[4r^{2} L_{T}^{2}]}{\delta^{2}}=\left(\dfrac{12mn^{2}r^{2}L^{2}}{\delta^{2}}\right).
\end{align}
Applying the union bound to Inequalities \eqref{Eq:Union_2} we have
\begin{align}
\nonumber
&\prob\left(\sup_{\bm{\omega}\in \mathrm{S}_{m}^{+}} \left| \|\bm{K}(\bm{\omega}) \|_{2}-\dfrac{1}{N}\|\bm{\Phi}\bm{\Phi}^{T}(\bm{\omega})\|_{2}  \right|\geq \delta \right)\\ \nonumber
&\leq \prob\left(L_{T}^{2}\geq \dfrac{\delta}{2r}\right)+\prob\left(\cup_{i=1}^{M}|T(\bm{\omega}_{i})|\geq \dfrac{\delta}{2} \right)\\
&\leq \left(\dfrac{12mn^{2}r^{2}L^{2}}{\delta^{2}}\right)+\left({8\over r}\right)^{m}\hspace{-2mm} \exp(3n\log 3)\cdot\exp\left(-{cN\delta^{2}\over 4n^{2}mL^{2}} \right).
\end{align}
The upper bound has the form of $\kappa_{1} r^{-m}+\kappa_{2} r^{-2}$. Setting $r= \left({\kappa_{1}\over \kappa_{2}} \right)^{1\over m+2}$ turns this into $2\kappa_{2}^{m\over m+2}\kappa_{1}^{2\over m+2}$. Then,
\begin{align}
\nonumber
&\prob\left(\sup_{\bm{\omega}\in \mathrm{S}_{m}^{+}} \left| \|\bm{K}(\bm{\omega}) \|_{2}-\dfrac{1}{N}\|\bm{\Phi}\bm{\Phi}^{T}(\bm{\omega})\|_{2}  \right|\geq \delta \right)\\
&\leq  2^{14}\left(\dfrac{m^{2}n^{2}L^{2}}{\delta^{2}}\right)\exp\left({6n\log 3\over m+2}\right)\cdot\exp\left(-{cN\delta^{2}\over 2n^{2}m(m+2)L^{2}} \right),
\end{align}
provided that ${12mn^{2}L^{2}\over \delta^{2}}\geq 1$. Therefore, with the probability of at least $1-\rho$, we have that
\begin{align}
\nonumber
&\sup_{\bm{\omega}\in \mathrm{S}_{m}^{+}} \left| \|\bm{K}(\bm{\omega}) \|_{2}-\dfrac{1}{N}\|\bm{\Phi}\bm{\Phi}^{T}(\bm{\omega})\|_{2}\right| \\ \label{Eq:Lambert_W}
&\leq\sqrt{{4n^{2}m^{2}L^{2}  \over cN}}W^{1\over 2}\left( 2^{13}\cdot {cN\over (m+2)\rho}\exp\left({6n\log 3\over m+2}\right)\right),
\end{align}
where $W$ is the Lambert-$W$ function.\footnote{Recall that the Lambert-$W$ functions is the inverse of the function $f(W)=We^{W}$.} Now, since $W(x)\leq \ln(x)$, for $x>e$, we can rewrite the upper bound in Eq. \eqref{Eq:Lambert_W} in terms of the elementary functions
\begin{align}
\nonumber
&\sup_{\bm{\omega}\in \mathrm{S}_{m}^{+}} \left| \|\bm{K}(\bm{\omega}) \|_{2}-\dfrac{1}{N}\|\bm{\Phi}\bm{\Phi}^{T}(\bm{\omega})\|_{2}\right| \leq\sqrt{{4n^{2}m^{2}L^{2}  \over cN}}\ln^{1\over 2}\left({2^{13}cN\over (m+2)\rho}\right)+\sqrt{{24n^{3}m^{2}L^{2}  \over cN(m+2)} }.
\end{align}

\subsection{Proof of Lemma \ref{Lemma:Kn_Norm}}
To establish the result, we need the decoupling technique from the probability theory. To state this result, consider independent Bernoulli random variables $\delta_{1},\cdots,\delta_{n}\in \{0,1\}$ with the expectation $\expect[\delta_{i}]=m/n$, often known as \textit{independent selectors}. Then, we define the subset $T=\{i\in [n]:\delta_{i}=1\}$. Its average size is $\expect[|T|]=\expect[\sum_{i=1}^{n}\delta_{i}]=m$. Now, we are in position to state the following lemma:

\begin{lemma}\textsc{(Decoupling, \cite[Lemma 5.60]{vershynin2010introduction})}
	\label{Lemma:Decoupling}
	Consider a double array of real numbers $\{a_{ij}\}_{i,j=1}^{n}$ such
	that $a_{ii}=0$ for all $i=1,2,\cdots,n$. Then,
	\begin{align}
	\sum_{i,j\in [n]}a_{ij}=4\expect \sum_{i\in [T],j\in [T^{c}]}a_{ij}
	\end{align}
	where $T$ is a random subset of $\{1,2,\cdots,n\}$ with average size $\expect[|T|]=n/2$. In particular,
	\begin{align}
	4\min_{T\subseteq [n]}\sum_{i\in T,j\in T^{c}}a_{ij}\leq \sum_{i,j\in [n]}a_{ij}\leq 4\max_{T\subset [n]}\sum_{i\in T,j\in T^{c}}a_{ij},
	\end{align}
	where the maximum and the minimum are over all subsets $T$ of $[n]$.
	\hfill $\square$
\end{lemma}
To establish the result for this kernels class, we use  the convexity of the spectral norm
\begin{align}
	\left\|\bm{K}(\bm{\omega})\right\|_{2}\leq \sum_{i=1}^{m}\omega_{i}\left\|\bm{K}_{i}\right\|_{2}.
\end{align}
Let $N_{1\over 4}$ be a ${1\over 4}$-net of the unit sphere $\mathrm{S}^{n-1}$ such that $|N_{1\over 4}|\leq  9^{n}$. By Lemma \ref{Lemma:Spectral_Norm}, we then have 
\begin{align}
	\nonumber
	\|\bm{K}_{i}\|_{2}&=\sup_{\bm{z}\in \mathrm{S}^{n-1}} |\langle \bm{z}, \bm{K}_{i}\bm{z} \rangle|\leq 2\max_{\bm{z}\in N_{1/4}}|\langle \bm{z},\bm{K}_{i}\bm{z} \rangle|\\ \label{Eq:the_sum}
	&=2\max_{\bm{z}\in N_{1/4}}\left|\sum_{k=1}^{n}z_{k}^{2}K_{i}(\bm{x}_{k},\bm{x}_{k})+2\sum_{1\leq k<\ell\leq n}z_{k}z_{\ell}K_{i}(\bm{x}_{k},\bm{x}_{\ell})\right|.
\end{align}
We note that the base kernels are centered, \textit{i.e.}, $K_{i}(\bm{x}_{k},\bm{x}_{k})=\psi_{i}(\bm{x}_{k}-\bm{x}_{k})-\psi_{i}(\bm{0})\bm{1}_{\{\bm{x}_{k}=\bm{x}_{k}\}}=0$. Therefore, the decoupling technique in Lemma \ref{Lemma:Decoupling} can be applied to obtain
\begin{align}
\nonumber
\|\bm{K}_{i}\|_{2}&\leq 4\left|\max_{\bm{z}\in N_{1/4}}\expect_{T}\left[\sum_{k\in T,\ell\in T^{c}}z_{k}z_{\ell}K_{i}(\bm{x}_{k},\bm{x}_{\ell})\right]\right|\\
\nonumber
&\leq 4\left|\max_{\bm{z}\in N_{1/4}}\expect_{T}\left[\sum_{k\in T,\ell\in T^{c}}z_{k}z_{\ell}\int_{\Xi}\varphi(\bm{x}_{k};\bm{\xi})\varphi(\bm{x}_{\ell};\bm{\xi})\mu_{i}(\mathrm{d}\bm{\xi})\right] \right|\\ \label{Eq:conditioned}
&\leq 4\max_{\bm{z}\in N_{1/4}}\int_{\Xi}\expect_{T}\left[\left|\sum_{k\in T}z_{k}\varphi(\bm{x}_{k};\bm{\xi})\right|\left|\sum_{\ell\in T^{c}}z_{\ell}\varphi(\bm{x}_{\ell};\bm{\xi}) \right|\right]\mu_{i}(\mathrm{d}\bm{\xi}).
\end{align}
To simplify the notations, let $\phi_{T}(\bm{x},\bm{z};\bm{\xi})\df \sum_{k\in T}z_{k}\varphi(\bm{x}_{k};\bm{\xi})$, and $\phi_{T^{c}}(\bm{x},\bm{z};\bm{\xi})\df \sum_{\ell\in T^{c}}z_{\ell}\varphi(\bm{x}_{\ell};\bm{\xi})$. The  random variables $\phi_{T}(\bm{x},\bm{z};\bm{\xi})$ and $\phi_{T^{c}}(\bm{x},\bm{z};\bm{\xi})$ are not centered. Nevertheless, using the triangle inequality, we can obtain the following centered random variables
\begin{align}
\label{Eq:Labels}
\dfrac{1}{4}\big\|\bm{K}_{i}\big\|_{2}
&\leq \mathsf{A}_{1}+\mathsf{A_{2}}+\mathsf{A}_{3}+\mathsf{A}_{4},
\end{align}
where each term on the right hand side of Eq. \eqref{Eq:Labels} is defined as below
\begin{align}
\nonumber
\mathsf{A}_{1}&\df \max_{\bm{z}\in N_{1/ 4}}\expect_{T,\mu_{i}}\Big[\left|\widehat{\phi}_{T}(\bm{x},\bm{z};\bm{\xi})\right|\left|\widehat{\phi}_{T^{c}}(\bm{x},\bm{z};\bm{\xi})\right| \Big],\\ \nonumber 
\mathsf{A}_{2}&\df \max_{\bm{z}\in N_{1/ 4}}\expect_{T,\mu_{i}}\Big[\expect_{\bm{x}}[|\phi_{T}(\bm{x},\bm{z};\bm{\xi})|] |\widehat{\phi}_{T^{c}}(\bm{x},\bm{z};\bm{\xi})|\Big],\\  \nonumber
\mathsf{A}_{3}&\df \max_{\bm{z}\in N_{1/4}}\expect_{T,\mu_{i}}\Big[\expect_{\bm{x}}[|\widehat{\phi}_{T^{c}}(\bm{x},\bm{z};\bm{\xi})|]|\widehat{\phi}_{T}(\bm{x},\bm{z};\bm{\xi})| \Big],\\ \nonumber
\mathsf{A}_{4}&\df \max_{\bm{z}\in N_{1/4}}\expect_{T,\mu_{i}}\Big[\expect_{\bm{x}}[|\phi_{T^{c}}(\bm{x},\bm{z};\bm{\xi})|] \expect_{\bm{x}}[|\phi_{T^{c}}(\bm{x},\bm{z};\bm{\xi})|]\Big],
\end{align}
where $\widehat{\phi}_{T}(\bm{x},\bm{z};\bm{\xi})\df  {\phi}_{T}(\bm{x},\bm{z};\bm{\xi})-\expect_{\bm{x}}[\phi_{T}(\bm{x},\bm{z};\bm{\xi})]$, and $\widehat{\phi}_{T^{c}}(\bm{x},\bm{z};\bm{\xi})\df  {\phi}_{T^{c}}(\bm{x},\bm{z};\bm{\xi})-\expect_{\bm{x}}[\phi_{T^{c}}(\bm{x},\bm{z};\bm{\xi})]$ are the centered random variables.

In the sequel, we analyze each term separately:

\textbf{Analysis of $\mathsf{A}_{1}$}: We compute a concentration bound using the Chernoff bound
\begin{align}
\nonumber
&\prob\left(\max_{\bm{z}\in N_{1/ 4}}\expect_{T,\mu_{i}}\Big[\left|\widehat{\phi}_{T}(\bm{x},\bm{z};\bm{\xi})\right|\left|\widehat{\phi}_{T^{c}}(\bm{x},\bm{z};\bm{\xi})\right| \Big]\geq \delta \right)\\ \nonumber
&\leq e^{2n\log(3)}\cdot\max_{\bm{z}\in N_{1/4}}\prob\left(\expect_{T,\mu_{i}}\Big[\left|\widehat{\phi}_{T}(\bm{x},\bm{z};\bm{\xi})\right|\left|\widehat{\phi}_{T^{c}}(\bm{x},\bm{z};\bm{\xi})\right| \Big]\geq \delta \right)\\ \nonumber
&\leq  e^{2n\log(3)}\cdot e^{-\beta \delta}\cdot \max_{\bm{z}\in N_{1/4}}\expect_{\bm{x}}\left[e^{\beta \expect_{T,\mu_{i}}\Big[\left|\widehat{\phi}_{T}(\bm{x},\bm{z};\bm{\xi})\right|\left|\widehat{\phi}_{T^{c}}(\bm{x},\bm{z};\bm{\xi})\right| \Big]}\right].
\end{align}
By Jensen's inequality, we have that
\begin{align}
\nonumber
\expect_{\bm{x}}\left[e^{\beta \expect_{T,\mu_{i}}\Big[\left|\widehat{\phi}_{T}(\bm{x},\bm{z};\bm{\xi})\right|\left|\widehat{\phi}_{T^{c}}(\bm{x},\bm{z};\bm{\xi})\right| \Big]}\right]&\leq \expect_{T,\mu_{i},\bm{x}}\left[e^{\beta\left|\widehat{\phi}_{T}(\bm{x},\bm{z};\bm{\xi})\right|\left|\widehat{\phi}_{T^{c}}(\bm{x},\bm{z};\bm{\xi})\right|}\right]\\
\label{Eq:Plug_Inequality_1}
&=\expect_{T,\mu_{i}}\expect_{\bm{x}}\left[e^{\beta \left|\widehat{\phi}_{T}(\bm{x},\bm{z};\bm{\xi})\right|\left|\widehat{\phi}_{T^{c}}(\bm{x},\bm{z};\bm{\xi})\right|}|\bm{\xi},T\right],
\end{align}
where the last step follows by the law of total expectation.  Due to Assumption \ref{Assumption:1}, the random variable $\widehat{\phi}_{T}(\bm{x},\bm{z};\bm{\xi})$ is zero-mean and bounded from above by 
\begin{align}
|\widehat{\phi}_{T}(\bm{x},\bm{z};\bm{\xi})|&\leq 2L\sum_{\ell\in T}|z_{\ell}|\\
&\leq 2L\sqrt{|T|},
\end{align}
where in the last inequality, we used the fact that $\|\bm{z}\|_{2}=1$ as $\bm{z}\in N_{1/4}\subset \mathrm{S}^{d-1}$, and thus $\sum_{\ell\in T}|z_{\ell}|\leq \sqrt{|T|}\left(\sum_{\ell\in T}|z_{\ell}|^{2}\right)^{1\over 2}\leq \sqrt{|T|}\|\bm{z}\|_{2}=\sqrt{|T|}$. Therefore, it is sub-Gaussian with the Orlicz norm of $\|\widehat{\phi}_{T}(\bm{x},\bm{z};\bm{\xi})\|_{\psi_{2}}\leq  2L\sqrt{|T|}$. Similarly, $\|\widehat{\phi}_{T^{c}}(\bm{x},\bm{z};\bm{\xi})\|_{\psi_{2}}\leq 2L\sqrt{|T^{c}|}$. Moreover, conditioned on $\bm{\xi}$ and $T$, the sub-Gaussian random variables $\phi_{T}(\bm{x},\bm{z};\bm{\xi})$ and $\phi_{T^{c}}(\bm{x},\bm{z};\bm{\xi})$ are independent. Consequently, by Laplace's condition for sub-Gaussian random variables, we obtain
\begin{align}
\nonumber
\expect_{\bm{x}}\left[e^{\beta |\widehat{\phi}_{T}(\bm{x},\bm{z};\bm{\xi})| |\widehat{\phi}_{T^{c}}(\bm{x},\bm{z};\bm{\xi})|}\Big| \bm{\xi},T, \bm{x}_{[T^{c}]} \right]&\leq \expect\left[e^{2L^{2}|T|\beta^{2}|\widehat{\phi}_{T^{c}}(\bm{x},\bm{z};\bm{\xi})|^{2}} \right].
\end{align}
Now, since $\phi_{T^{c}}(\bm{x},\bm{z};\bm{\xi})$ is sub-Gaussian, from the Definition \ref{Def:Orlicz} of Orlicz norm we conclude that
\begin{align}
\expect\left[e^{2L^{2}|T|\beta^{2}|\widehat{\phi}_{T^{c}}(\bm{x},\bm{z};\bm{\xi})|^{2}} \right]\leq 2,
\end{align}
for all $T\subset \{1,2,\cdots,n\}$ provided $\beta\leq {1\over 2L^{2}n}$.
Hence,
\begin{align}
\nonumber
\prob\left(\max_{\bm{z}\in N_{1/ 4}}\expect_{T,\mu_{i}}\Big[\left|\widehat{\phi}_{T}(\bm{x},\bm{z};\bm{\xi})\right|\left|\widehat{\phi}_{T^{c}}(\bm{x},\bm{z};\bm{\xi})\right| \Big]\geq \delta \right)&\leq 2e^{2n\log(3)}e^{-\beta \delta},
\end{align} 
for all $\beta<{1\over 2L^{2}n}$.

\textbf{Analysis of $\mathsf{A}_{2}$ and $\mathsf{A}_{3}$:} In the sequel, we analyze $\mathsf{A}_{2}$. The analysis of $\mathsf{A}_{3}$ is similar. By exponential Chebyshev's inequality, we derive
\begin{align}
\nonumber
&\prob\left(\max_{\bm{z}\in N_{1/ 4}}\expect_{T,\mu_{i}}\Big[\expect_{\bm{x}}[|\phi_{T}(\bm{x},\bm{z};\bm{\xi})|] |\widehat{\phi}_{T^{c}}(\bm{x},\bm{z};\bm{\xi})|\Big]\geq \delta \right)\\
&\leq e^{2n\log(3)}e^{-\beta \delta} \max_{\bm{z}\in N_{1/ 4}}\expect_{T,\bm{\xi},\bm{x}}\left[e^{\beta\expect_{\bm{x}}[|\phi_{T}(\bm{x},\bm{z};\bm{\xi})|]|\widehat{\phi}_{T^{c}}(\bm{x},\bm{z};\bm{\xi})|)}\right]\\
&\leq e^{2n\log(3)}e^{-\beta \delta} \max_{\bm{z}\in N_{1/ 4}}e^{2\beta^{2}L^{2}|T^{c}|(\expect_{\bm{x}}[|\phi_{T}(\bm{x},\bm{z};\bm{\xi})|])^{2}},
\end{align}
where the last inequality follows by the Laplace's condition for sub-Gaussian random variables. Due to the fact that $|\phi_{T}(\bm{x},\bm{z};\bm{\xi})|\leq L\sum_{\ell\in T}|z_{\ell}|\leq L\sqrt{|T|}$, we have $(\expect[|\phi_{T}(\bm{x},\bm{z};\bm{\xi})|])^{2}\leq L^{2}|T|$. Therefore, for all $\beta \in \real_{+}$, the following inequality holds,
\begin{align}
\prob\left(\max_{\bm{z}\in N_{1/ 4}}\expect_{T,\mu_{i}}\Big[\expect_{\bm{x}}[|\phi_{T}(\bm{x},\bm{z};\bm{\xi})|] |\widehat{\phi}_{T^{c}}(\bm{x},\bm{z};\bm{\xi})|\Big]\geq \delta \right)&\leq e^{2n\log(3)}e^{-\beta \delta}\expect_{T}\Big[e^{2\beta^{2}L^{4}|T||T^{c}|}\Big]\\ \label{Eq:Mexican}
&\leq e^{2n\log(3)}e^{-\beta \delta}e^{2\beta^{2}L^{4}n^{2}},
\end{align}
where the last inequality follows by the fact that $|T|\leq n $ and $|T^{c}|\leq n$.  Similarly, it can be shown that 
\begin{align}
\prob\left(\max_{\bm{z}\in N_{1/ 4}}\expect_{T,\mu_{i}}\Big[\expect_{\bm{x}}[|\phi_{T^{c}}(\bm{x},\bm{z};\bm{\xi})|] |\widehat{\phi}_{T}(\bm{x},\bm{z};\bm{\xi})|\Big]\geq \delta \right)\leq e^{2n\log(3)}e^{-\beta \delta}e^{2\beta^{2}L^{4}n^{2}}.
\end{align}

\textbf{Analysis of $\mathsf{A}_{4}$}: Due to Assumption \ref{Assumption:1} and the triangle inequality, we have that 
\begin{align}
\nonumber
& \max_{\bm{z}\in N_{1/4}}\expect_{T,\mu_{i}}\Big[\expect_{\bm{x}}[|\phi_{T^{c}}(\bm{x},\bm{z};\bm{\xi})|] \expect_{\bm{x}}[|\phi_{T^{c}}(\bm{x},\bm{z};\bm{\xi})|]\Big]\\ \nonumber
&\leq L^{2}\expect_{T}\left[\sqrt{|T| (n-|T|)}\left(\sum_{t\in T}|z_{k}|^{2}\right)^{1\over 2}\left(\sum_{t\in T^{c}}|z_{k}|^{2}\right)^{1\over 2}\right]\\
&\leq L^{2}\expect_{T}\left[\sqrt{|T|(n-|T|)}\right].
\end{align}
Now, recall that $|T|\sim \mathrm{Bionomial}\Big(n,{1\over 2}\Big)$ is a bionomial random variable. Therefore, by Jensen's inequality and concavity of the square root function $f(x)=\sqrt{x}$, we derive that
\begin{align}
\nonumber
\expect_{T}\left[ \sqrt{|T|(n-|T|)}\right]&\leq \sqrt{\expect_{T}[|T|(n-|T|)]}\\ \nonumber
&\leq \sqrt{n\expect[|T|]-\expect[|T|^{2}]}\\ \nonumber
&=\dfrac{1}{2}n\sqrt{1-\dfrac{1}{n}}.
\end{align}
\textbf{Combining results of $\mathsf{A_{1}}-\mathsf{A}_{4}$.} We now employ a union bound to compute a concentration bound. Combining the concentration inequalities for $\mathsf{A}_{1}-\mathsf{A}_{4}$ yields
\begin{align}
\nonumber
\prob\left({1\over 4}\|\bm{K}_{i}\|_{2}-\dfrac{1}{2}n\sqrt{1-\dfrac{1}{n}}\geq \delta \right)&\leq 2e^{2n\log(3)}e^{-{\beta\delta\over 3}}\left(1+e^{2\beta^{2}L^{4}n^{2}}\right)\\
&\leq 4e^{2n\log(3)}e^{-{\beta\delta\over 3}}e^{2\beta^{2}L^{4}n^{2}},
\end{align}
for all $\beta\leq {1\over 2L^{2}n}$. Assume that $L>1$ and $\delta=\delta(n)=o(n)$, and let $\beta={\delta\over 12L^{4}n^{2}}$. Then,
\begin{align}
\prob\left({1\over 4}\|\bm{K}_{i}\|_{2}-\dfrac{1}{2}n\sqrt{1-\dfrac{1}{n}}\geq \delta \right)\leq 4e^{2n\log(3)}e^{-{\delta^{2}\over  72L^{4}n^{2}}}.
\end{align} 
 Therefore, with the probability of $1-\rho$, we obtain
 \begin{align}
 \|\bm{K}_{i}\|_{2}\leq  2n\sqrt{1-\dfrac{1}{n}}+4\sqrt{\ln\dfrac{3^{2n}\times 4}{\rho}}.
 \end{align}
 This completes the proof of the thesis.$\hfill \blacksquare$

 \subsection{Proof of Theorem \ref{Thm:Bartlett}}
 \label{App:Proof_of_Theorem_Bartlett}
 
We provide a distributionally robust generalization bound which extends the result of Bartlett and Mendelson \cite[Thm. 8]{bartlett2002rademacher}. Since $\tilde{\ell}$ dominates the loss function $\ell$, for all $f\in \mathcal{F}$, we obtain that
 \begin{align}
 \nonumber
 \sup_{P\in \mathcal{P}}\expect_{P}[\ell(Y,f(\bm{X}))]&\leq \sup_{P\in \mathcal{P}}\expect_{P}[\tilde{\ell}(Y,f(\bm{X}))]\\ \label{Eq:plug_and_play}
 &= \min_{\lambda\in \Lambda}\lambda\log\expect_{Q}\left[\exp\left({1\over \lambda}\tilde{\ell}(Y,f(\bm{X})) \right)\right]+\dfrac{\lambda r}{2}.
 \end{align}
Now, we define
 \begin{align}
 \nonumber
\expect_{Q}\left[\exp\left({1\over \lambda}\tilde{\ell}(Y,f(\bm{X})) \right)\right]&\leq \expect_{\widehat{Q}^{n}}\left[\exp\left({1\over \lambda}\tilde{\ell}(Y,f(\bm{X})) \right)\right]\\ \nonumber
&\hspace{4mm}+\sup_{h\in  \pi \circ \mathcal{L} \circ \mathcal{F}}\left(\expect_{Q}\left[h\right]-\expect_{\widehat{Q}^{n}}\left[h\right]\right)\\ \nonumber
&\leq \expect_{\widehat{Q}^{n}}\left[\exp\left({1\over \lambda}\tilde{\ell}(Y,f(\bm{X})) \right)\right]\\ \nonumber
&\hspace{4mm}+\sup_{h\in  \pi \circ \mathcal{L} \circ \mathcal{F}}\left(\expect_{Q}\left[h\right]-\expect_{\widehat{Q}^{n}}\left[h\right]\right)\\ \label{Eq:Paranthesis}
&\hspace{4mm}+\expect_{Q}[\tilde{\ell}(y,0)]-\expect_{\widehat{Q}^{n}}[\tilde{\ell}(y,0)],
 \end{align}
where $\mathcal{L}\circ \mathcal{F}= \{(x,y)\mapsto \tilde{\ell}(y,f(x))-\tilde{\ell}(y,0):f\in \mathcal{F} \}$, and $\pi\circ \mathcal{F}'=\{ x\mapsto \exp(-f(x)/\lambda)-1: f\in \mathcal{F}'\}.$
Consider the second term inside the parenthesis in Eq. \eqref{Eq:Paranthesis}. When $(Y_{i},\bm{X}_{i})$ changes, the supremum changes by no more than $2/n$. We leverage McDiarmid's Martingale inequality to obtain
\begin{align}
\nonumber
\sup_{h\in  \pi \circ \mathcal{L} \circ \mathcal{F}}\left(\expect_{Q}\left[h\right]-\expect_{\widehat{Q}^{n}}\left[h\right]\right)\leq \expect_{Q}\left[\sup_{h\in  \pi \circ \mathcal{L} \circ \mathcal{F}}\left(\expect_{Q}\left[h\right]-\expect_{\widehat{Q}^{n}}\left[h\right]\right)\right] +\sqrt{2\ln(2/\delta)\over n}.
\end{align}
A similar argument in conjunction with the fact that $\expect_{Q}\left[\expect_{\widehat{Q}^{n}}[\tilde{\ell}(y,0)]\right]=\expect_{Q}[\tilde{\ell}(y,0)]$ establishes the following inequality 
\begin{align}
\nonumber
\expect_{Q}\left[\exp\left({1\over \lambda}\tilde{\ell}(Y,f(\bm{X})) \right)\right]\leq &\expect_{\widehat{Q}^{n}}\left[\exp\left({1\over \lambda}\tilde{\ell}(Y,f(\bm{X})) \right)\right]\\
&+\expect_{Q}\left[\sup_{h\in  \pi \circ \mathcal{L}\circ \mathcal{F}}\left(\expect_{Q}\left[h\right]-\expect_{\widehat{Q}^{n}}\left[h\right]\right)\right]+\sqrt{8\ln(2/\delta)\over n},
\label{Eq:sovereign}
\end{align}
with the probability of at least $1-\delta$. We now invoke the following symmetrisation approach of van der Vaart and Wellner \cite[Lemma 2.3.1]{van1996weak}.
We reproduce it here for the sake of readability:
\begin{theorem}\textsc{(Symmetrization, \cite[Lemma 2.3.1]{van1996weak})}
	\label{Theorem:Symmetrization}
	Let $X_{1},\cdots,X_{n}\sim_{\text{i.i.d.}} P$ denote the random variables drawn i.i.d.  from the distribution $P$ with the support $\mathcal{X}=\mathrm{supp}(P)$. Furthermore, let $\mathcal{F}$ denotes the class of real valued functions on $\mathcal{X}$. Then
	\begin{align}
		\label{Eq:Inequality_Symmetrization}
		\hspace{-2mm}\expect_{P}\left[\sup_{f\in \mathcal{F}}\Big|\expect_{\widehat{P}^{n}}[f(X)]-\expect_{P}[f(X)]\Big|\right]\leq 2\expect_{P}\left[\sup_{f\in \mathcal{F}}\left|\dfrac{1}{n}\sum_{i=1}^{n}\varepsilon_{i}f(X_{i})\right| \right],
	\end{align}
	where $\varepsilon_{1},\cdots, \varepsilon_{n}\in \{-1,+1\}$ is a Rademacher sequence, independent of $X_{1},\cdots,X_{n}$.\hfill $\square$
\end{theorem}
Using the symmetrization technique, it can be shown that the expectation on the right hand side of Eq. \eqref{Eq:sovereign} is proportional to the Rademacher complexity, see \cite[Thm. 8]{bartlett2002rademacher}. In particular,
\begin{align}
\label{Eq:plug_2}
\expect_{Q}\left[\sup_{h\in  \pi \circ \mathcal{L}\circ \mathcal{F}}\left(\expect_{Q}\left[h\right]-\expect_{\widehat{Q}^{n}}\left[h\right]\right)\right]=2\mathfrak{R}^{n}(\pi\circ \mathcal{L}\circ \mathcal{F}).
\end{align}
Plugging Eq. \eqref{Eq:plug_2} into Eq. \eqref{Eq:sovereign}, we obtain the following inequality
\begin{align}
\nonumber
\expect_{Q}\left[\exp\left({1\over \lambda}\tilde{\ell}(y,f(\bm{x})) \right)\right]\leq &\expect_{\widehat{Q}^{n}}\left[\exp\left({1\over \lambda}\tilde{\ell}(y,f(\bm{x})) \right)\right]\\
\label{Eq:Then_From}
&+2\mathfrak{R}^{n}(\pi\circ \mathcal{L} \circ \mathcal{F})+\sqrt{8\ln(2/\delta)\over n}.
\end{align}
We now wish to derive the Rademacher complexity in terms of the function class $\mathcal{F}$ instead of the composite function class $\pi\circ \mathcal{L} \circ \mathcal{F}$. To this end, we invoke \cite[Theorem 12, part 4]{bartlett2002rademacher} and the fact that $\psi(y,\cdot):\real\rightarrow\real_{+}$
is $L_{\psi}$-Lipschitz  for all $y\in \mathcal{Y}$, and $\exp(-x/\lambda)$ is $1/\lambda^{2}$-Lipschitz for $x>0$.
\begin{align}
\mathfrak{R}^{n}(\pi\circ\mathcal{L} \circ \mathcal{F})\leq {L_{\psi}\over \lambda^{2}}\mathfrak{R}^{n}(\mathcal{F}).
\end{align}
Therefore, from Eq. \eqref{Eq:Then_From}, we obtain that
\begin{align}
\nonumber
\expect_{Q}\left[\exp\left({1\over \lambda}\tilde{\ell}(y,f(\bm{x})) \right)\right]\leq &\expect_{\widehat{Q}^{n}}\left[\exp\left({1\over \lambda}\tilde{\ell}(y,f(\bm{x})) \right)\right]+{2L_{\psi}\over \lambda^{2}}\mathfrak{R}^{n}(\mathcal{F})+\sqrt{8\ln(2/\delta)\over n}.
\end{align}
Taking the logarithm from both sides and multiplying by $\lambda>0$ now yields
\begin{align}
\nonumber
&\lambda\log \expect_{Q}\left[\exp\left({1\over \lambda}\tilde{\ell}(y,f(\bm{x})) \right)\right]\\
\label{Eq:proceed}
&\leq \lambda \log \expect_{\widehat{Q}^{n}}\left[\exp\left({1\over \lambda}\tilde{\ell}(y,f(\bm{x})) \right)\right]+\lambda\log\left(1+ \dfrac{{2L_{\psi}\over \lambda^{2}}\mathfrak{R}^{n}(\mathcal{F})+\sqrt{8\ln(2/\delta)/n}}{\expect_{\widehat{Q}^{n}}\left[\exp\left({1\over \lambda}\tilde{\ell}(y,f(\bm{x})) \right)\right]}\right).  
\end{align}
We proceed from Eq. \eqref{Eq:proceed} using the basic inequality $\log(1+x)\leq x$ for $x\geq -1$, as well as the fact that $\expect_{\widehat{Q}^{n}}\left[\exp\left({1\over \lambda}\tilde{\ell}(y,f(\bm{x})) \right)\right]\geq 1$ for all $\lambda>0$. The latter inequality holds since $\tilde{\ell}(y,f(\bm{x}))\geq \ell(y,f(\bm{x})) \geq 0$ for all $(\bm{x},y)\in \mathcal{X}\times \mathcal{Y}$. From Eq. \eqref{Eq:proceed}, we then obtain that for all $\lambda> 0$ that
\begin{align}
\nonumber
&\lambda\log \expect_{Q}\left[\exp\left({1\over \lambda}\tilde{\ell}(y,f(\bm{x})) \right)\right]\\
\nonumber
&\stackrel{\rm{(a)}}{\leq}  \lambda \log \expect_{\widehat{Q}^{n}}\left[\exp\left({1\over \lambda}\tilde{\ell}(y,f(\bm{x})) \right)\right]+\lambda\log\left(1+ {{2L_{\psi}\over \lambda^{2}}\mathfrak{R}^{n}(\mathcal{F})+\sqrt{8\ln(2/\delta)/n}}\right)\\ \nonumber
&\stackrel{\rm{(b)}}{\leq} \lambda \log \expect_{\widehat{Q}^{n}}\left[\exp\left({1\over \lambda}\tilde{\ell}(y,f(\bm{x})) \right)\right]+{2L_{\psi}\over \lambda}\mathfrak{R}^{n}(\mathcal{F})+\lambda \sqrt{{8\ln(2/\delta)\over n}}.
\end{align}
where in $\rm{(a)}$ we used the fact that
\begin{align}
\log\left(1+ \dfrac{{2L_{\psi}\over \lambda^{2}}\mathfrak{R}^{n}(\mathcal{F})+\sqrt{8\ln(2/\delta)/n}}{\expect_{\widehat{Q}^{n}}\left[\exp\left({1\over \lambda}\tilde{\ell}(y,f(\bm{x})) \right)\right]}\right)	\leq \log\left(1+ {{2L_{\psi}\over \lambda^{2}}\mathfrak{R}^{n}(\mathcal{F})+\sqrt{8\ln(2/\delta)/n}}\right),
\end{align}
since as we noted earlier $\expect_{\widehat{Q}^{n}}\left[\exp\left({1\over \lambda}\tilde{\ell}(y,f(\bm{x})) \right)\right]\geq 1$. Furthermore, in $\rm{(b)}$, we once again used the basic inequality $\log(1+x)\leq x$ for $x\geq -1$.
The Taylor expansion with respect to $\lambda>0$ yields
\begin{align}
\nonumber
\lambda \log \expect_{\widehat{Q}^{n}}\left[\exp\left({1\over \lambda}\tilde{\ell}(y,f(\bm{x})) \right)\right]&=\lambda \log  \expect_{\widehat{Q}^{n}}\left[\left(1+\sum_{p=1}^{\infty} \dfrac{\tilde{\ell}^{p}(y,f(\bm{x}))}{p!\lambda^{p}}  \right)\right]\\
&\leq \expect_{\widehat{Q}^{n}}\left[\sum_{p=1}^{\infty}\dfrac{\tilde{\ell}^{p}(y,f(\bm{x}))}{p!\lambda^{p-1}}\right],
\end{align}
where the last step follows again using the basic inequality $\log(1+x)\leq x$ for $x\geq -1$. Plugging the upper bound in Equation \eqref{Eq:plug_and_play} yields
\begin{align}
\nonumber
&\sup_{P\in \mathcal{P}}\expect_{P}[\ell(Y,f(\bm{X}))]\\
&\leq \inf_{\lambda>0}\Bigg(\expect_{\widehat{Q}^{n}}\left[\sum_{p=1}^{\infty}\dfrac{\tilde{\ell}^{p}(y,f(\bm{x}))}{p!\lambda^{p-1}}\right]+{2L_{\psi}\over \lambda}\mathfrak{R}^{n}(\mathcal{F})+\lambda \sqrt{{8\ln(2/\delta)\over n}}+\dfrac{\lambda r}{2}\Bigg).
\end{align} 
Now, we let $\lambda=B$ where $B\df \sup_{\bm{x},y\in \mathcal{X}\times \mathcal{Y}} \tilde{\ell}(y,f(\bm{x}))$. Then, $\tilde{\ell}^{p-1}(y_{i},f(\bm{x}_{i}))/\lambda^{p-1}<1$ for all $(y_{i},\bm{x}_{i})_{1\leq i\leq n}$. We then obtain
\begin{align}
\nonumber
\sup_{P\in \mathcal{P}}\expect_{P}[\ell(Y,f(\bm{X}))]&\leq (e-1)
\expect_{\widehat{Q}^{n}}\left[{\tilde{\ell}(y,f(\bm{x}))}\right]\\
&\hspace{4mm}+{2L_{\psi}\over B}\mathfrak{R}^{n}(\mathcal{F})+B \sqrt{{8\ln(2/\delta)\over n}}+\dfrac{Br}{2}.
\end{align}
 $\hfill \blacksquare$
 \section{Proof of Auxiliary Results}
 
\subsection{Proof of Lemma \ref{Lemma:Sub-Gaussian_Error}}
\label{App:Sub-Gaussian_Error} 

The proof follows by applying McDiarmid's Martingale inequality. In particular, let $(\bm{\xi}^{k}_{i})_{1\leq i\leq m}^{1\leq k\leq N}$ and $(\tilde{\bm{\xi}}^{k}_{i})_{1\leq i\leq m}^{1\leq k\leq N}$ denote two vectors that differ only in $(i_{0},k_{0})$ coordinate for some $i_{0}\in \{1,2,\cdots,m\}$ and $k_{0}\in \{1,2,\cdots,N\}$. That is,
\begin{align}
\widetilde{\bm{\xi}}_{i}^{k}=\bm{\xi}_{i}^{k}, \quad \forall i\in \{1,2,\cdots,m\}\backslash \{i_{0}\}, \forall k\in \{1,2,\cdots,N\}\backslash\{k_{0}\}.
\end{align}
Let $E_{N}(\bm{z},\tilde{\bm{z}})$ and $E_{N}'(\bm{z},\tilde{\bm{z}})$ denote the associated functions. Then, 
\begin{align}
\nonumber
\left|E_{N}(\bm{z},\tilde{\bm{z}})-E'_{N}(\bm{z},\tilde{\bm{z}})\right|&\leq \dfrac{\omega_{i_{0}}}{N}\left|\expect_{P^{\otimes 2}}\left[y\tilde{y}\varphi(\bm{x};\bm{\xi}_{i_{0}}^{k_{0}})\varphi(\tilde{\bm{x}};\bm{\xi}_{i_{0}}^{k_{0}}) \right]-\expect_{P^{\otimes 2}}\left[y\tilde{y}\varphi(\bm{x};\widetilde{\bm{\xi}}_{i_{0}}^{k_{0}})\varphi(\tilde{\bm{x}};\widetilde{\bm{\xi}}_{i_{0}}^{k_{0}}) \right]\right|\\
&\leq \dfrac{2L^{2}\omega_{i_{0}}}{N},
\end{align}
where the last step follows by Assumption \ref{Assumption:1}. By McDiarmid's Martingale inequality and the fact that $\expect[E_{N}(\bm{z},\tilde{\bm{z}})]=0$, we then have that
\begin{align}
\prob(|E_{N}(\bm{z},\tilde{\bm{z}})|\geq \varepsilon)\leq 2\exp\left(-\dfrac{2N\varepsilon^{2}}{\|\bm{\omega}\|_{2}^{2}L^{4}}\right).
\end{align}
Since $\bm{\omega}=(\omega_{1},\cdots,\omega_{m})\in \mathrm{S}_{m}^{+}$, $\|\bm{\omega}\|_{2}\leq \|\bm{\omega}\|_{1}= 1$, and we can simplify the upper bound to obtain
\begin{align}
\prob(|E_{N}(\bm{z},\tilde{\bm{z}})|\geq \varepsilon)\leq 2\exp\left(-\dfrac{2N\varepsilon^{2}}{L^{4}}\right).
\end{align}
From Lemma \ref{Lemma:Con_to_Sub_Gaussian}, we thus conclude $\|E_{N}\|_{\psi_{2}}\leq {L^{2}\over 2\sqrt{N}}$. $\hfill \blacksquare$

\subsection{Proof of Lemma \ref{lemma:stupid}}
\label{Appendix:Aux_2}

To establish the proof, in the sequel, we treat the upper and lower bounds separately:

\subsubsection{The Upper Bound}Recall the definition of $X(\bm{z}_{i},\bm{z}_{j})$ from Eq. \eqref{Eq:Recall_Xij}.  From Equation \eqref{Eq:State_2}, we have that
\begin{align}
\nonumber
\expect_{Q^{\otimes 2}}\left[\sup_{\widehat{P}^{n}\in \mathcal{P}^{n}}\widehat{T}_{\bm{\omega}}(\widehat{P}^{n})\right]&=\expect_{Q^{\otimes 2}}\left[\inf_{\lambda>0} \widehat{H}_{n,N}(\lambda) \right]\\ 
&=\expect_{Q^{\otimes 2}}\left[\min_{\lambda>0} \lambda \log \dfrac{2}{n(n-1)}\sum_{1\leq i<j\leq n}  \exp\left(-\dfrac{X(\bm{z}_{i},\bm{z}_{j})}{\lambda } \right)  +\dfrac{\lambda r}{2}  \right]\\ \nonumber
&\leq \inf_{\lambda>0}\expect_{Q^{\otimes 2}}\left[\lambda\log \dfrac{2}{n(n-1)}\sum_{1\leq i<j\leq n}\exp\left(-\dfrac{X(\bm{z}_{i},\bm{z}_{j})}{\lambda }  \right)+\dfrac{\lambda r}{2} \right]\\ \nonumber
&\stackrel{\rm{(a)}}{\leq} \inf_{\lambda>0} \lambda\log \expect_{Q^{\otimes 2}}\left[\dfrac{2}{n(n-1)}\sum_{1\leq i<j\leq n}\exp\left(-\dfrac{X(\bm{z}_{i},\bm{z}_{j})}{\lambda } \right) \right]+\dfrac{\lambda r}{2} \\ \nonumber
&=\inf_{\lambda>0} \lambda\log \expect_{Q^{\otimes 2}}\left[\exp\left(-\dfrac{X(\bm{z},\tilde{\bm{z}})}{\lambda} \right)  \right]+\dfrac{\lambda r}{2}\\
&=\sup_{P\in \mathcal{P}} \widehat{T}_{\bm{\omega}}(P),
\end{align}
where $\rm{(a)}$ follows from Jensen's inequality, and concavity of the $\log$ function.

\subsubsection{The Lower Bound}: To establish the lower bound, we invoke Lemma \ref{Proposition:1} to obtain
\begin{align}
\nonumber
\expect_{Q^{\otimes 2}}\left[\sup_{\widehat{P}^{n}\in \mathcal{P}^{n}}\widehat{T}_{\bm{\omega}}(\widehat{P}^{n})\right]&=\expect_{Q^{\otimes 2}}\left[\inf_{\lambda>0} \widehat{H}_{n,N}(\lambda) \right]\\  \nonumber
&=\expect_{Q^{\otimes 2}}\left[\inf_{\lambda>0}\left(\widehat{H}_{n,N}(\lambda)-\expect\big[\widehat{H}_{n,N}(\lambda)\big]+\expect\big[\widehat{H}_{n,N}(\lambda)\big]\right)\right]\\
\label{Eq:Compute_a_lower_bound}
&\geq \inf_{\lambda>0} \expect_{P^{\otimes 2}}\left[\widehat{H}_{n,N}(\lambda)\right]-\expect\left[\sup_{\lambda>0} \big|\widehat{H}_{n,N}(\lambda)-\expect\big[\widehat{H}_{n,N}(\lambda)\big] \big| \right].
\end{align}
To compute a lower bound on the first term, we leverage the Gumbel max perturbation in Eq. \eqref{Eq:Gumbel} to obtain
\begin{align}
\nonumber
 \inf_{\lambda>0} \expect_{Q^{\otimes 2}}\left[\widehat{H}_{n,N}(\lambda)\right]&=\inf_{\lambda>0}\expect_{Q^{\otimes 2}}\left[ \lambda \log \dfrac{2}{n(n-1)}\sum_{1\leq i<j\leq n}  \exp\left(-\dfrac{X(\bm{z}_{i},\bm{z}_{j})}{\lambda } \right)  +\dfrac{\lambda r}{2}\right]\\  \nonumber
 &=\inf_{\lambda>0} \expect_{Q^{\otimes 2},\nu}\left[\max_{1\leq i<j\leq n}\left\{-{X(\bm{z}_{i},\bm{z}_{j})}+\lambda\zeta_{ij}\right\}\right]+\lambda \log \dfrac{2}{n(n-1)}  +\dfrac{\lambda r}{2},\\   \label{Eq:For_the_Sake}
&\geq \expect_{Q^{\otimes 2}}[-X(\bm{z},\tilde{\bm{z}})].
\end{align}
The last term in Eq. \eqref{Eq:For_the_Sake} can be written as the limit
\begin{subequations}
 \begin{align}
\expect_{Q^{\otimes 2}}[-X(\bm{z},\tilde{\bm{z}})]&=\lim_{\lambda \downarrow 0}\lambda \log \expect_{Q^{\otimes 2}}\left[ \exp\left(-{X(\bm{z},\tilde{\bm{z}})\over \lambda} \right)\right]\\
&=\lim_{\lambda \downarrow 0}\lambda \log \expect_{Q^{\otimes 2}}\left[ \exp\left(-{X(\bm{z},\tilde{\bm{z}})\over \lambda} \right)\right]+\lambda \log \dfrac{2}{n(n-1)}+\dfrac{\lambda r}{2}.
 \end{align}
 \end{subequations}
We thus proceed as follows
 \begin{align}
 \nonumber
 \inf_{\lambda>0} \expect_{Q^{\otimes 2}}\left[\widehat{H}_{n,N}(\lambda)\right]&\geq \lim_{\lambda \downarrow 0}\lambda \log \expect_{Q^{\otimes 2}}\left[ \exp\left(-{X(\bm{z},\tilde{\bm{z}})\over \lambda} \right)\right]+\lambda \log \dfrac{2}{n(n-1)}  +\dfrac{\lambda r}{2}
 \\ \nonumber
 &\geq \inf_{\lambda>0} \log \expect_{Q^{\otimes 2}}\left[ \exp\left(-{X(\bm{z},\tilde{\bm{z}})\over \lambda} \right)\right]+\lambda \log \dfrac{2}{n(n-1)}  +\dfrac{\lambda r}{2}\\  \label{Eq:LL1}
 &=\sup_{P\in \mathcal{P}}\widehat{T}_{\bm{\omega}}(P) .
\end{align}
To compute a lower bound for the second term of Eq. \eqref{Eq:Compute_a_lower_bound}, we once again leverage Talagrand's concentration of measure in Appendix \ref{Appendix:Proof_of_Lemma_Consistency} . In particular, we recall from Appendix \ref{Appendix:Proof_of_Lemma_Consistency} that $c\leq 
\lambda_{\ast}\leq C$. Moreover, recall that the mapping 
\begin{align}
\bm{X}=(X_{ij})_{1\leq i<j\leq n}\mapsto f(\bm{X})=\widehat{H}_{n,N}(\lambda)=\min_{\lambda>0}\lambda\log \expect_{\widehat{Q}^{n,\otimes 2}}\left[\exp\left(-{1\over \lambda}X(\bm{z},\tilde{\bm{z}}) \right)\right]+\dfrac{\lambda r}{2},
\end{align}
is $C^{2}n^{2}\exp\left({2L^{2}\over c} \right)$-Lipschitz. Therefore, based on Talagrand's concentration of measure, for any $\lambda\in [c,C]$, we have
\begin{align}
\prob\left(|\widehat{H}_{n,N}(\lambda)-\expect[\widehat{H}_{n,N}(\lambda)]|\geq \varepsilon \right)\leq 2\exp\left(-\dfrac{\delta^{2}n^{2}}{K^{2}L^{2}} \right),
\end{align}
where $K$ . Therefore, $H_{n,N}(\lambda)$ is sub-Gaussian with the Orlicz norm of $\|H_{n,N}(\lambda)\|_{\psi_{2}}\leq nKL$. Now, consider a $\varepsilon$-cover $\mathcal{N}_{\varepsilon}=\{\lambda_{i}\}_{i=1}^{N_{\varepsilon}}$ of the interval $[c,C]$, where $N_{\varepsilon}\leq \dfrac{(C-c)}{\varepsilon}$. The maximal inequality for the sub-Gaussian random variables in Lemma \ref{Lemma:Generalized Sub-Gaussian Maximal Inequality} yields
\begin{align}
\nonumber
\expect\left[\max_{i\leq N_{\varepsilon}}\Big|\widehat{H}_{n,N}(\lambda_{i})-\expect[\widehat{H}_{n,N}(\lambda_{i})]\Big| \right]
&\leq {KL\over n}\sqrt{2\log 2N_{\varepsilon}}\\
&\leq {KL\over n}\sqrt{2\log \left(\dfrac{2(C-c)}{\varepsilon} \right)}.
\end{align}
Due to the Lipschitz continuity of the mapping $\lambda\mapsto \widehat{H}_{n,N}(\lambda)$, we have that
\begin{align}
\nonumber
\expect\left[\sup_{\lambda\in [c,C]}\left|\widehat{H}_{n,N}(\lambda)-\expect[\widehat{H}_{n,N}(\lambda)] \right| \right]&\leq \expect\left[\max_{i\leq N_{\varepsilon}}\Big|\widehat{H}_{n,N}(\lambda_{i})-\expect[\widehat{H}_{n,N}(\lambda_{i})]\Big| \right]+\varepsilon {C^{2}K^{2}\over n^{2}}\\ \label{Eq:LL2}
&\leq {KL\over n}\sqrt{2\log {2(C-c)\over \varepsilon}}+{C^{2}K^{2}\varepsilon \over n^{2}}.
\end{align}
Plugging Inequalities \eqref{Eq:LL1}-\eqref{Eq:LL2} in Eq. \eqref{Eq:Compute_a_lower_bound} yields
\begin{align}
\expect_{Q^{\otimes 2}}\left[\sup_{\widehat{P}^{n}\in \mathcal{P}^{n}}\widehat{T}_{\bm{\omega}}(\widehat{P}^{n})\right]\geq \sup_{P\in \mathcal{P}} \widehat{T}_{\bm{\omega}}(P)-\dfrac{KL}{n}\sqrt{2\log \dfrac{2(C-c)}{\varepsilon}}-\dfrac{C^{2}K^{2}\varepsilon}{n^{2}}.
\end{align}
Letting $\varepsilon=2(C-c)$ yields
\begin{align}
\nonumber
\expect_{Q^{\otimes 2}}\left[\sup_{\widehat{P}^{n}\in \mathcal{P}^{n}}\widehat{T}_{\bm{\omega}}(\widehat{P}^{n})\right]&\geq \sup_{P\in \mathcal{P}} \widehat{T}_{\bm{\omega}}(P)-{2C^{3}K^{2}\over n^{2}}.
\end{align}
\hfill $\blacksquare$

\subsection{Proof of Lemma \ref{Appendix:Metric_Entropy_Integral}}
\label{Appendix:Metric_Entropy_Integral}

Define the following empirical process
\begin{align}
\label{Eq:empirical_process}
\mathbb{G}_{s,n}f=\mathbb{P}f-\mathbb{P}_{s,n}f,
\end{align}
where
\begin{align}
\mathbb{P}f\df \expect_{P}f(\bm{Z}_{1:n}), \quad \mathbb{P}_{s,n}f\df U_{s}(f,\bm{Z}_{1:n}).
\end{align}
We wish to bound $\|\mathbb{G}_{s,n}f\|_{\mathcal{F}}\df \sup_{f\in \mathcal{F}}|\mathbb{G}_{s,n}f|$ with a high probability. Without loss of generality, we suppose $\mathcal{F}$ contains the zero function, since we can always augment $\mathcal{F}$ with the zero function without changing the order of its covering number. We apply Markov's inequality to obtain
\begin{align}
\prob\left(\|\mathbb{G}_{s,n}f\|_{\mathcal{F}}\geq t \right)\leq \dfrac{\expect_{P}[\|\mathbb{G}_{s,n}f\|^{p}_{\mathcal{F}}]}{t^{p}}=\dfrac{\mathbb{P} \|\mathbb{G}_{s,n}f\|^{p}_{\mathcal{F}}}{t^{p}}, 
\end{align}
for $p\geq 1$. The following lemma is an extension of the symmetrization technique in Theorem \ref{Theorem:Symmetrization} for the $U$-statistic of the order $s$:
\begin{lemma}\textsc{(Symmetrization of the $U$-statistic)}
\label{Lemma:Symmetrization_for_U_statistic}
Consider the empirical process $\mathbb{G}_{s,n}f$ characterized in Eq. \eqref{Eq:empirical_process}. Then, the following inequality holds
\begin{align}
\label{Eq:Chaining}
\mathbb{P}\left(\sup_{f\in \mathcal{F}}|\mathbb{G}_{s,n}f|\right)^{p}\leq 2^{sp} \mathbb{P}\left(\sup_{f\in \mathcal{F}} | \mathbb{P}^{\bm{\varepsilon}}_{s,n}f|\right)^{p},
\end{align}
where
\begin{align}
\mathbb{P}^{\bm{\varepsilon}}_{s,n}f\df \dfrac{1}{{n \choose s}}\sum_{1\leq i_{1}\leq \cdots\leq i_{s}\leq n} \varepsilon_{i_{1}}\varepsilon_{i_{2}}\cdots\varepsilon_{i_{s}}f(\bm{Z}_{i_{1}},\bm{Z}_{i_{2}},\cdots,\bm{Z}_{i_{s}}),
\end{align}
is the symmetrized $U$-statistic, and $\varepsilon_{1},\cdots\varepsilon_{n}\sim_{\text{i.i.d.}}\mathrm{Uniform} \{-1,+1\}$ are the Rademacher random variables.
\end{lemma}
The proof of Lemma \ref{Lemma:Symmetrization_for_U_statistic} for the special case of $s=2$ can be found in \cite[Chapter 3]{pena1999decoupling}.

In the sequel, we find an upper bound for the right hand side of Eq. \eqref{Eq:Chaining} in Lemma \ref{Lemma:Symmetrization_for_U_statistic}.  To compute the right hand side of Eq. \eqref{Eq:Chaining}, we levereage a standard \textit{chaining argument} similar to \cite{khosravi2019non}. Specifically, let $\bar{\mathcal{F}}=\{f_{1},f_{2},\cdots, f_{\mathcal{N}(\varepsilon,\mathcal{F},L_{2}(P))}\}$ denotes a minimal $\varepsilon$-covering of $\mathcal{F}$. Then, using the one step discretization yields
\begin{align}
\label{Eq:one_step_discretization}
\sup_{f,g\in \mathcal{F}} | \mathbb{P}^{\bm{\varepsilon}}_{s,n}(f-g)|\leq 2\max_{f,g\in \bar{\mathcal{F}}}| \mathbb{P}^{\bm{\varepsilon}}_{s,n}(f-g)|+2\sup_{f,g\in \mathcal{F}:\|f-g\|_{L_{r}(P)}\leq \varepsilon} | \mathbb{P}^{\bm{\varepsilon}}_{s,n}(f-g)|.
\end{align}
Using the basic inequality $\big(\sum_{i=1}^{N}a_{i}\big)^{p}\leq N^{p-1}\sum_{i=1}^{N}a_{i}^{p}$ and subsequently taking the expectation yields
\begin{align}
\nonumber
\mathbb{P}\left(\sup_{f,g\in \mathcal{F}} | \mathbb{P}^{\bm{\varepsilon}}_{s,n}(f-g)|\right)^{p}\leq &2^{p}\mathbb{P}\left(\max_{f,g\in \bar{\mathcal{F}}}| \mathbb{P}^{\bm{\varepsilon}}_{s,n}(f-g)|\right)^{p}\\
\label{Eq:boring_meeting}
&+2^{p}\mathbb{P}\left(\sup_{f,g\in \mathcal{F}:\|f-g\|_{L_{r}(P)}\leq \varepsilon} | \mathbb{P}^{\bm{\varepsilon}}_{s,n}(f-g)|\right)^{p}.
\end{align}
We first bound the first term on the right hand side of Eq. \eqref{Eq:boring_meeting}. Suppose that
\begin{align}
\delta\df \sup_{f,g\in \mathcal{F}}\|f-g\|_{L_{r}(P)}<+\infty.
\end{align}
For each $m=1,2,\cdots$ let $\varepsilon_{m}\df 2^{-m}\delta$, and let $\bar{\mathcal{F}}_{m}\subset \mathcal{F}$ denotes the $\varepsilon_{m}$-covering of $\bar{\mathcal{F}}$. Futhermore, define the mapping $\pi_{m}:\bar{\mathcal{F}}\rightarrow \bar{\mathcal{F}}_{m}$
\begin{align}
\pi_{m}(f)\df \arg\min_{g\in \bar{\mathcal{F}}_{m}}\|f-g\|_{L_{r}(P)}, \quad f\in \bar{\mathcal{F}}.
\end{align}
For each $f\in \bar{\mathcal{F}}$, define the sequence of points $(f_{1},\cdots,f_{M})$ in $\mathcal{F}$ such that $f_{M}=f$ and $f_{m}=\pi_{m}(f_{m+1})$. Here, $M$ is the smallest number such that given the $\varepsilon$-covering $\{f_{1},\cdots,f_{N}\}$ of $\mathcal{F}$ with $N=\mathcal{N}(\varepsilon,\mathcal{F},L_{2}(P))$, the norm balls $\ball^{r}_{\varepsilon_{L}}(f_{i})$ do not intersect. That is,
\begin{align}
\label{Eq:smallest}
\|f-g\|_{L_{r}(P)}>\varepsilon_{M}=2^{-M}\delta, \quad \text{for all}\ f,g\in \bar{\mathcal{F}}.
\end{align}
Since $M$ is the smallest integer satisfying Eq. \eqref{Eq:smallest}, we know that there exists some $f,g\in \mathcal{F}$ such that $\|f-g\|_{L_{r}(P)}<\varepsilon_{M-1}=2^{-(M-1)}\delta$, since otherwise, $M-1$ will be the smallest such integer instead. At the same time, we know that $\bar{\mathcal{F}}$ is a $\varepsilon$-covering of $\mathcal{F}$, so the following inequality holds
\begin{align}
\varepsilon \leq \|f-g\|_{L_{r}(P)}\leq 2^{-(M-1)}\delta, \quad \text{for all}\ f,g\in \bar{\mathcal{F}}.
\end{align}
Therefore, $M\leq 1+\log_{2}(\delta/\varepsilon)$. Using the linearity of the empirical process $\mathbb{P}^{\bm{\varepsilon}}_{s,n}$ and the triangle inequality yields
\begin{align}
| \mathbb{P}^{\bm{\varepsilon}}_{s,n}(f-f_{1})|\leq \sum_{m=2}^{M}\big|\mathbb{P}^{\bm{\varepsilon}}_{s,n}(f_{m}-f_{m-1})\big|.
\end{align}
Similarly, for another function $g\in \bar{\mathcal{F}}$, we obtain
\begin{align}
| \mathbb{P}^{\bm{\varepsilon}}_{s,n}(g-g_{1})|\leq \sum_{m=2}^{M}\big| \mathbb{P}^{\bm{\varepsilon}}_{s,n}(g_{m}-g_{m-1})\big|.
\end{align}
Thus, for $f,g\in \bar{\mathcal{F}}$ we obtain
\begin{align}
\nonumber
\max_{f,g\in \bar{\mathcal{F}}}| \mathbb{P}^{\bm{\varepsilon}}_{s,n}(f-g)\big|\leq &\max_{f_{1},g_{1}\in\mathcal{F}_{1}}| \mathbb{P}^{\bm{\varepsilon}}_{s,n}(f_{1}-g_{1})|\\
&+2\sum_{m=2}^{M}\max_{f\in \mathcal{F}_{m}}| \mathbb{P}^{\bm{\varepsilon}}_{s,n}(f-\pi_{m-1}(f))|.
\end{align}
We raise both sides to the power of $p\geq 1$, use the basic inequality $\big(\sum_{i=1}^{N}a_{i}\big)^{p}\leq N^{p-1}\sum_{i=1}^{N}a_{i}^{p}$, and then take the expectation to obtain
\begin{align}
\nonumber
\mathbb{P}\left(\max_{f,g\in \bar{\mathcal{F}}}| \mathbb{P}^{\bm{\varepsilon}}_{s,n}(f-g)\big| \right)^{p}\leq &M^{p-1}\mathbb{P}\left(\max_{f_{1},g_{1}\in\mathcal{F}_{1}}| \mathbb{P}^{\bm{\varepsilon}}_{s,n}(f_{1}-g_{1})| \right)^{p} \\ \label{Eq:power}
&+2^{p}M^{p-1}\sum_{m=2}^{M}\mathbb{P}\left(\max_{f\in \mathcal{F}_{m}}| \mathbb{P}^{\bm{\varepsilon}}_{s,n}(f-\pi_{m-1}(f))|\right)^{p}.
\end{align}
To obtain a bound for the second term, we first recognize that
\begin{align}
\label{Eq:recognize}
\max_{f\in \mathcal{F}_{m}}\|f-\pi_{m-1}(f)\|_{L_{r}(P)}\leq  2^{-m}\delta,
\end{align}
due to the fact that $\pi_{m-1}(f)$ is the best approximation of $f$ from $\mathcal{F}_{m}$. Let $\kappa=\lceil n/s \rceil$ and for any fixed Rademacher sequence $\varepsilon_{1},\cdots,\varepsilon_{n}\in \{-1,+1\}^{n}$, define
\begin{align}
\label{Eq:non_overlap}
\psi_{m}(\bm{Z}_{1},\cdots,\bm{Z}_{n})\df {1\over \kappa} \sum_{j=1}^{\kappa}\varepsilon_{s(j-1)+1}\varepsilon_{s(j-1)+2}\cdots\varepsilon_{sj}g_{m}(\bm{Z}_{s(j-1)+1},\cdots,\bm{Z}_{sj}),
\end{align}
where $g_{m}\df f-\pi_{m-1}f$. We now have
\begin{align}
\kappa \sum_{(\sigma_{1},\cdots,\sigma_{n})\in \Pi_{n}} \psi_{m}(\bm{Z}_{\sigma_{1}},\cdots,\bm{Z}_{\sigma_{n}}) = \kappa s!(n-s)!\sum_{i_{1}<i_{2}<\cdots<i_{m}}\varepsilon_{i_{1}}\varepsilon_{i_{2}}\cdots\varepsilon_{i_{s}}g_{m}(\bm{Z}_{i_{1}},\cdots,\bm{Z}_{i_{s}}),
\end{align}
where $\Pi_{n}$ is the set of the permutations of the set $\{1,2,\cdots,n\}$. Therefore,
\begin{align}
\mathbb{P}_{s,n}^{\bm{\varepsilon}}(f-\pi_{m-1}f)=\dfrac{1}{n!} \sum_{(\sigma_{1},\cdots,\sigma_{n})\in \Pi_{n}} \psi_{m}(\bm{Z}_{\sigma_{1}},\cdots,\bm{Z}_{\sigma_{n}}). 
\end{align}
The function $\psi_{m}(\bm{Z}_{\sigma_{1}},\cdots,\bm{Z}_{\sigma_{n}})$ is the average of $\kappa$ i.i.d. random variables since the terms in the sum in Eq. \eqref{Eq:non_overlap} are non-overlapping. Let $r=2$. Due to Eq. \eqref{Eq:recognize}, we have that
\begin{align}
\nonumber
\|\varepsilon_{i}\varepsilon_{i+1}\cdots \varepsilon_{i+s}g_{m}(\bm{Z}_{i},\cdots,\bm{Z}_{i+s})\|_{L_{2}(P\otimes Q)}&=\|g_{m}(\bm{Z}_{i},\cdots,\bm{Z}_{i+s})\|_{L_{2}(P)}\\
&= \|f-\pi_{m}(f)\|_{L_{2}(P)}\leq \delta 2^{-m},
\end{align}
for all $f\in \mathcal{F}_{m}$, where $Q=\mathrm{Uniform}\{-1,+1\}$. Furthermore, $\mathbb{P}(\varepsilon_{i}\varepsilon_{i+1}\cdots \varepsilon_{i+s}g_{m}(\bm{Z}_{i},\cdots,\bm{Z}_{i+s}))=0$. Therefore, each term in the numerator of Eq. \eqref{Eq:non_overlap} is zero mean with the variance bounded by $\sigma^{2}=\delta^{2} 2^{-2m}$. Thus, $\psi_{m}(\bm{Z}_{\sigma_{1}},\cdots,\bm{Z}_{\sigma_{n}})$ is sub-Gaussian, \textit{i.e.}, 
\begin{align}
\label{Eq:Sub_Gaussian}
\expect\left[\exp\Big(s \psi_{m}(\bm{Z}_{\sigma_{1}},\cdots,\bm{Z}_{\sigma_{n}})\Big)\right]\leq \exp\left(\dfrac{\delta^{2} 2^{-2m} s^{2}}{2\kappa} \right), \quad \forall s>0.
\end{align}
The inequality in the previous display also implies
\begin{align}
\nonumber
\expect\big[\exp\left(s \mathbb{P}_{s,n}^{\bm{\varepsilon}}(f-\pi_{m-1}f)\right) \big]&=\expect\left[\exp\left( {s\over n!}\sum_{(\sigma_{1},\cdots,\sigma_{n})\in \Pi_{n}} \psi_{m}(\bm{Z}_{\sigma_{1}},\cdots,\bm{Z}_{\sigma_{n}}) \right)\right]\\ \nonumber
&\stackrel{\rm{(a)}}{\leq} \dfrac{1}{n!}\sum_{(\sigma_{1},\cdots,\sigma_{n})\in \Pi_{n}} \expect\left[\exp\left(s  \psi_{m}(\bm{Z}_{\sigma_{1}},\cdots,\bm{Z}_{\sigma_{n}}) \right)\right]\\
&\stackrel{\rm(b)}{\leq} \exp\left(\dfrac{\delta^{2} 2^{-2m} s^{2}}{2\kappa} \right), \quad \forall s>0,
\end{align}
where $\rm{(a)}$ is due to Jensen's inequality, and $\rm{(b)}$ is due to Eq. \eqref{Eq:Sub_Gaussian}. Therefore, $\mathbb{P}_{s,n}^{\bm{\varepsilon}}(f-\pi_{m-1}f)$ is also sub-Gaussian. We leverage the sub-Gaussian maximal inequality in Lemma \ref{Lemma:Generalized Sub-Gaussian Maximal Inequality} to obtain
\begin{align}
\label{Eq:2}
\mathbb{P}\left(\max_{f\in \mathcal{F}_{m}}| \mathbb{P}^{\bm{\varepsilon}}_{s,n}(f-\pi_{m-1}(f))|\right)^{p}\leq \dfrac{\delta^{p} 2^{-pm}}{\kappa^{p/2}} \left(\log \mathcal{N}\big(\delta 2^{-m}, \mathcal{F},L_{2}(P)\big)\right)^{p/2}.
\end{align}
Similarly, it can be shown that
\begin{align}
\label{Eq:3}
\mathbb{P}\left(\max_{f_{1},g_{1}\in\mathcal{F}_{1}}| \mathbb{P}^{\bm{\varepsilon}}_{s,n}(f_{1}-g_{1})| \right)^{p}\leq \dfrac{(\delta/2)^{p}}{\kappa^{p/2}} \left(\log \mathcal{N}\big(\delta/2, \mathcal{F},L_{2}(P)\big)\right)^{p/2}.
\end{align}
Substituing Eqs. \eqref{Eq:2} and \eqref{Eq:3} into Eq. \eqref{Eq:power} to obtain
\begin{align}
\label{Eq:Approximate_the_sum}
\mathbb{P}\left(\max_{f,g\in \bar{\mathcal{F}}}| \mathbb{P}^{\bm{\varepsilon}}_{s,n}(f-g)\big| \right)^{p}\leq 2^{p+1}M^{p-1}\sum_{m=1}^{M}\dfrac{(\delta 2^{-m})^{p}}{\kappa^{p/2}}\left(\log \mathcal{N}\big(\delta 2^{-m}, \mathcal{F},L_{2}(P)\big)\right)^{p/2}.
\end{align}
The function $s\mapsto (\log N(s,\mathcal{F},L_{2}(P)))^{p/2}$ is monotone increasing on $[0,\delta]$. Therefore, we can upper bound the sum in Eq. \eqref{Eq:Approximate_the_sum} with an integral as follows
\begin{align}
\nonumber
\mathbb{P}\left(\max_{f,g\in \bar{\mathcal{F}}}| \mathbb{P}^{\bm{\varepsilon}}_{s,n}(f-g)\big| \right)^{p}&\leq \dfrac{2^{p+2}}{\kappa^{p/2}}M^{p-1}\sum_{m=1}^{M}\int_{\delta 2^{-(m+1)}}^{\delta 2^{-m}} \left(\log \mathcal{N}(s,\mathcal{F},L_{2}(P)) \right)^{p/2}\mathrm{d}s\\
&\leq \dfrac{2^{p+2}}{\kappa^{p/2}}M^{p-1}\int_{\delta 2^{-(M+1)}}^{\delta/2}\left(\log \mathcal{N}(s,\mathcal{F},L_{2}(P)) \right)^{p/2}\mathrm{d}s.
\end{align}
We further proceed using the inequality $\varepsilon/4\leq \delta/2^{M+1}$ to plug for the value of $M$
\begin{align}
\label{Eq:COMB_00}
\mathbb{P}\left(\max_{f\in \bar{\mathcal{F}}}| \mathbb{P}^{\bm{\varepsilon}}_{s,n}(f-g)\big| \right)^{p}\leq \dfrac{2^{p+2}}{\kappa^{p/2}}\left( \dfrac{2\delta}{\varepsilon \log(2)}\right)^{p-1}\int_{\varepsilon/4}^{\delta/2}\left(\log \mathcal{N}(s,\mathcal{F},L_{2}(P)) \right)^{p/2}\mathrm{d}s.
\end{align}
Plugging Eq. \eqref{Eq:COMB_00} into Eq. \eqref{Eq:one_step_discretization} yields
\begin{align}
\nonumber
\mathbb{P}\left(\sup_{f,g\in \mathcal{F}}| \mathbb{P}^{\bm{\varepsilon}}_{s,n}(f-g)\big| \right)^{p}&\leq 2^{p}\mathbb{P}\left(\sup_{f,g\in \mathcal{F}:\|f-g\|_{L_{r}(P)}\leq \varepsilon} | \mathbb{P}^{\bm{\varepsilon}}_{s,n}(f-g)|\right)^{p}\\ \label{Eq:Upper_Bound_2}
&\hspace{4mm}+\dfrac{2^{p+2}}{\kappa^{p/2}}\left( \dfrac{2\delta}{\varepsilon \log(2)}\right)^{p-1}\int_{\varepsilon/4}^{\delta/2}\left(\log \mathcal{N}(s,\mathcal{F},L_{2}(P)) \right)^{p/2}\mathrm{d}s.
\end{align}
Now consider the following two cases:
\begin{itemize}[leftmargin=*]
\item $p=1$: In this case, the upper bound in Eq. \eqref{Eq:Upper_Bound_2} reduces to
\begin{align}
\nonumber
\mathbb{P}\left(\sup_{f,g\in \mathcal{F}}| \mathbb{P}^{\bm{\varepsilon}}_{s,n}(f-g)\big| \right)&\leq 2\mathbb{P}\left(\sup_{f,g\in \mathcal{F}:\|f-g\|_{L_{r}(P)}\leq \varepsilon} | \mathbb{P}^{\bm{\varepsilon}}_{s,n}(f-g)|\right)\\
&\hspace{4mm}+\dfrac{4}{\sqrt{\kappa}}\int_{\varepsilon/4}^{\delta/2}\sqrt{\log \mathcal{N}(s,\mathcal{F},L_{2}(P))}\mathrm{d}s.
\end{align}
Using the monotone convergance theorem \cite{endou2008lebesgue} yields
\begin{align}
\nonumber
\lim_{\varepsilon\downarrow 0}\mathbb{P}\left(\sup_{f,g\in \mathcal{F}:\|f-g\|_{L_{r}(P)}\leq \varepsilon} | \mathbb{P}^{\bm{\varepsilon}}_{s,n}(f-g)|\right)&=\mathbb{P}\left(\lim_{\varepsilon\downarrow 0} \sup_{f,g\in \mathcal{F}:\|f-g\|_{L_{r}(P)}\leq \varepsilon} | \mathbb{P}^{\bm{\varepsilon}}_{s,n}(f-g)|\right)\\&=0.
\end{align}
Therefore,
\begin{align}
\mathbb{P}\left(\sup_{f,g\in \mathcal{F}}| \mathbb{P}^{\bm{\varepsilon}}_{s,n}(f-g)\big| \right)\leq \dfrac{4}{\sqrt{\kappa}}\int_{0}^{\delta/2}\sqrt{\log \mathcal{N}(s,\mathcal{F},L_{2}(P))}\mathrm{d}s.
\end{align}
\item $p>1$: In this case, letting $\varepsilon\downarrow 0$ blows up the upper bound in Eq. \eqref{Eq:Upper_Bound_2}. However, we note that due to sub-Gaussianity of the empirical process, we have
\begin{align}
\mathbb{P}\left(\sup_{f,g\in \mathcal{F}:\|f-g\|_{L_{r}(P)}\leq \varepsilon} | \mathbb{P}^{\bm{\varepsilon}}_{s,n}(f-g)|\right)^{p}=\Omega(\varepsilon^{p}).
\end{align}
\end{itemize}

Let $g=0$ (zero function) and consider $p=1$ in the sequel.  Using the symmetrization result of Lemma \ref{Lemma:Symmetrization_for_U_statistic} and Markov inequality yields
\begin{align}
\prob\left(\|\mathbb{G}_{s,n}f\|_{\mathcal{F}}\geq t \right)\leq \dfrac{1}{t} \dfrac{4}{\sqrt{\kappa}}\int_{0}^{\delta/2}\sqrt{\log \mathcal{N}(s,\mathcal{F},L_{2}(P))}\mathrm{d}s.
\end{align}
for all $p\geq 1$. Therefore, with the probability of at least $1-\varrho$, we obtain
\begin{align}
\|\mathbb{G}_{s,n}f\|_{\mathcal{F}}\leq  \dfrac{1}{\varrho} \dfrac{4}{\sqrt{\kappa}}\int_{0}^{\delta/2}\sqrt{\log \mathcal{N}(s,\mathcal{F},L_{2}(P))}\mathrm{d}s.
\end{align}

\hfill $\blacksquare$
\subsection{Proof of Lemma \ref{Lemma:Bartlett}}
\label{Appendix:Proof_of_Lemma_Bartlett}

Consider the function class $\mathcal{F}'$ defined in Lemma \ref{Lemma:Bartlett}. Let $\bm{f}_{\bm{\omega}}\df (\sqrt{\omega_{1}}f^{k}_{i})_{i=1,2,\cdots,m}^{k=1,2,\cdots,N}, f^{k}_{i} \in \mathcal{F}$ and consider the $\epsilon$-covering balls 
\begin{align}
\ball_{2r}^{\epsilon}(f^{0}_{\ell})\df \left\{f\in \mathcal{F}:  \|f-f^{0}_{\ell}\|_{L_{2r}(Q)}\leq \epsilon \right\},
\end{align}
for $\ell\in \{1,2,\cdots,\mathcal{N}(\epsilon,\mathcal{F},L_{r}(Q))\}$. Let $\pi:\{1,2,\cdots,N\}\times \{1,2,\cdots,m\}\mapsto \{1,2,\cdots,\mathcal{N}(\epsilon,\mathcal{F},L_{r}(Q))\}$ denotes a mapping that assigns the $(i,k)$ index of $f_{i}^{k}$ to the center of a covering ball, \textit{i.e.}, $f^{i}_{k}\in \ball_{2r}^{\epsilon}(f^{0}_{\pi(i,k)})$. Since $\psi$ is $L_{\psi}$-Lipschitz, we obtain the following inequality
\begin{align}
\nonumber
&\left|\psi\Big({1\over N}y\tilde{y}\langle \bm{f}_{\bm{\omega}}(\bm{x}),\bm{f}_{\bm{\omega}}(\tilde{\bm{x}})\rangle\Big)-\psi\Big({1\over N}y\tilde{y}\langle \bm{f}^{0}_{\pi}(\bm{x}),\bm{f}^{0}_{\pi}(\tilde{\bm{x}})\rangle\Big)\right|\\ \nonumber
&\hspace{40mm} \leq {L_{\psi}\over N} \Big|y\tilde{y}\big\langle \bm{f}(\bm{x}),\bm{f}(\tilde{\bm{x}})\big\rangle-y\tilde{y}\big\langle \bm{f}_{\pi}^{0}(\bm{x}),\bm{f}_{\pi}^{0}(\tilde{\bm{x}})\big\rangle\Big|\\ \label{Eq:proceed}
&\hspace{40mm} ={L_{\psi}\over N}\Big|\big\langle \bm{f}(\bm{x}),\bm{f}(\tilde{\bm{x}})\big\rangle-\big\langle \bm{f}^{0}_{\pi}(\bm{x}),\bm{f}^{0}_{\pi}(\tilde{\bm{x}})\big\rangle\Big|,
\end{align}
where $\bm{f}_{\pi}^{0}\df (f^{0}_{\pi(i,k)})_{i=1,2,\cdots,m}^{k=1,2,\cdots,N}$, and in the last step follows by the fact that $|y\tilde{y}|=1$ due to binary class labels $y,\tilde{y}\in \{-1,+1\}$. We proceed from Eq. \eqref{Eq:proceed} by applying the triangle inequality
\begin{align}
\nonumber
&\Big|\psi\Big({1\over N}y\tilde{y}\langle \bm{f}_{\bm{\omega}}(\bm{x}),\bm{f}_{\bm{\omega}}(\tilde{\bm{x}})\rangle\Big)-\psi\Big({1\over N}y\tilde{y}\langle \bm{f}^{0}_{\pi}(\bm{x}),\bm{f}^{0}_{\pi}(\tilde{\bm{x}})\rangle\Big)\Big|\\ \nonumber
&\leq {L_{\psi}\over N}  \Big| \langle \bm{f}_{\bm{\omega}}(\bm{x}),\bm{f}_{\bm{\omega}}(\tilde{\bm{x}})\big\rangle-\big\langle \bm{f}_{\bm{\omega}}(\bm{x}),\bm{f}^{0}_{\pi}(\tilde{\bm{x}})\big\rangle \Big|+{L_{\psi}\over N} \Big| \big\langle \bm{f}_{\bm{\omega}}(\bm{x}),\bm{f}^{0}_{\pi}(\tilde{\bm{x}})\big\rangle-\big\langle \bm{f}^{0}_{\pi}(\bm{x}),\bm{f}^{0}_{\pi}(\tilde{\bm{x}})\big\rangle  \Big|\\ \nonumber
&\leq {L_{\psi}\over N} \sum_{k=1}^{N}\sum_{i=1}^{m}\omega_{i}|f^{k}_{i}(\bm{x})| \big|f^{k}_{i}(\tilde{\bm{x}})-f^{0}_{\pi(i,k)}(\tilde{\bm{x}})\big|+ {L_{\psi}\over N} \sum_{i=1}^{N}|f^{0}_{\pi(i)}(\tilde{\bm{x}})| \big|f^{k}_{i}(\bm{x})-f^{0}_{\pi(i,k)}(\bm{x})\big|\\ \nonumber
&\leq {L_{\psi}\over N}F(\bm{x})\sum_{k=1}^{N}\sum_{i=1}^{m}\omega_{i}\big|f_{i}^{k}(\tilde{\bm{x}})-f^{0}_{\pi(i,k)}(\tilde{\bm{x}})\big|+{L_{\psi}\over N}F(\tilde{\bm{x}})\sum_{i=1}^{N}\sum_{i=1}^{m}\omega_{i}\big|f_{i}^{k}(\bm{x})-f^{0}_{\pi(i,k)}(\bm{x})\big|,
\end{align}
where $F(\bm{x})\df \sup_{f\in \mathcal{F}}|f(\bm{x})|$ is the envelope function. We raise both sides of the inequality to the power of $r\geq 1$ and use the basic inequality $\left|\sum_{i=1}^{n}x_{i}\right|^{r}\leq n^{r-1}\sum_{i=1}^{n}|x_{i}|^{r}$ to obtain
\begin{align}
\nonumber
&\Big|\phi\Big({1\over N}y\tilde{y}\langle \bm{f}_{\bm{\omega}}(\bm{x}),\bm{f}_{\bm{\omega}}(\tilde{\bm{x}})\rangle\Big)-\phi\Big({1\over N}y\tilde{y}\langle \bm{f}^{0}_{\pi}(\bm{x}),\bm{f}^{0}_{\pi}(\tilde{\bm{x}})\rangle\Big)\Big|^{r}
\\  \nonumber
&\leq {(2Nm)^{r-1}\over N^{r}} L_{\psi}^{r}F^{r}(\bm{x})\sum_{k=1}^{N}\sum_{i=1}^{m}\omega^{r}_{i}\big|f_{i}^{k}(\tilde{\bm{x}})-f^{0}_{\pi(i,k)}(\tilde{\bm{x}})\big|^{r}\\
&\hspace{4mm}+{(2Nm)^{r-1}\over N^{r}} L_{\psi}^{r}F^{r}(\tilde{\bm{x}})\sum_{k=1}^{N}\sum_{i=1}^{m}\omega^{r}_{i} \big|f_{i}^{k}(\bm{x})-f^{0}_{\pi(i,k)}(\bm{x})\big|^{r}.
\end{align}
Taking the expectation with respect to the product of the joint distribution $P^{\otimes 2}, P\in \mathcal{M}(\mathcal{X}\times \mathcal{Y})$ with the marginal $Q\in \mathcal{M}(\mathcal{X})$ now yields
\begin{align}
\nonumber
&\Big\|\phi\Big({1\over N}y\tilde{y}\langle \bm{f}_{\bm{\omega}}(\bm{x}),\bm{f}_{\bm{\omega}}(\tilde{\bm{x}})\rangle\Big)-\phi\Big({1\over N}y\tilde{y}\langle \bm{f}^{0}_{\pi}(\bm{x}),\bm{f}^{0}_{\pi}(\tilde{\bm{x}})\rangle\Big)\Big\|^{r}_{L_{r}(P)}
\\ &\leq {(Nm)^{r-1}\over N^{r}}2^{r}L_{\psi}^{r}  \sum_{k=1}^{N}\sum_{i=1}^{m}\omega^{r}_{i}\left \|F\cdot \Big|f_{i}^{k}-f^{0}_{\pi(i,k)}\Big|\right\|^{r}_{L_{r}(Q)}.
\end{align}
Using the fact that $|F|=F$ yields
\begin{align}
\nonumber
&\Big\|\phi\Big({1\over N}y\tilde{y}\langle \bm{f}_{\bm{\omega}}(\bm{x}),\bm{f}_{\bm{\omega}}(\tilde{\bm{x}})\rangle\Big)-\phi\Big({1\over N}y\tilde{y}\langle \bm{f}^{0}_{\pi}(\bm{x}),\bm{f}^{0}_{\pi}(\tilde{\bm{x}})\rangle\Big)\Big\|^{r}_{L_{r}(P)}\\ \nonumber
&\stackrel{\rm{(a)}}{\leq} {(2Nm)^{r-1}\over N^{r}} \sum_{k=1}^{N}\sum_{i=1}^{m}\omega_{i}^{r}\|F^{2} \|_{L_{r}(Q)}^{r/2}\cdot  \|(f_{i}^{k}-f^{0}_{\pi(i,k)})^{2}\|^{r/2}_{L_{r}(Q)}\\ \label{Eq:Obtain}
&= {(Nm)^{r-1}\over N^{r}}2^{r}L^{r} \sum_{k=1}^{N}\sum_{i=1}^{m}\omega_{i}^{r}\|F\|_{L_{2r}(Q)}^{r}\cdot  \|f_{i}^{k}-f^{0}_{\pi(i,k)}\|^{r}_{L_{2r}(Q)},
\end{align}
where $\rm{(a)}$ is due to the generalized H\"{o}lder inequality $\||f|^{\theta}\cdot |g|^{1-\theta}\|_{p}\leq \|f\|_{p_{1}}^{\theta}\cdot \|g\|_{p_{2}}^{1-\theta}$ such that ${1\over p}={\theta\over p_{1}}+{(1-\theta)\over p_{2}}$ for $\theta\in (0,1)$. Due to the fact that $f_{i}^{k}\in \ball_{2r}^{\epsilon}\Big(f^{0}_{\pi(i,k)}\Big)$, we obtain from Eq. \eqref{Eq:Obtain} that
\begin{align}
\Big\|\phi\Big({1\over N}y\tilde{y}\langle \bm{f}_{\bm{\omega}}(\bm{x}),\bm{f}_{\bm{\omega}}(\tilde{\bm{x}})\rangle\Big)-\phi\Big({1\over N}y\tilde{y}\langle \bm{f}^{0}_{\pi}(\bm{x}),\bm{f}^{0}_{\pi}(\tilde{\bm{x}})\rangle\Big)\Big\|^{r}_{L_{r}(P)}\leq m^{r-1}(2L\|F\|_{L_{2r}(Q)}\epsilon\|\bm{\omega}\|_{r})^{r}.
\end{align}
Since $\bm{\omega}\in \mathrm{S}^{+}_{m}$, $\|\bm{\omega} \|_{r}\leq \|\bm{\omega}\|_{1}=1$ for $r\geq 1$. Therefore, since $m\geq 1$, we can simplify the upper bound as follows
\begin{align}
\nonumber
\Big\|\phi\Big({1\over N}y\tilde{y}\langle \bm{f}_{\bm{\omega}}(\bm{x}),\bm{f}_{\bm{\omega}}(\tilde{\bm{x}})\rangle\Big)-\phi\Big({1\over N}y\tilde{y}\langle \bm{f}^{0}_{\pi}(\bm{x}),\bm{f}^{0}_{\pi}(\tilde{\bm{x}})\rangle\Big)\Big\|^{r}_{L_{r}(P)}\leq (2mL\|F\|_{L_{2r}(Q)}\epsilon)^{r}.
\end{align}
Therefore, $\phi\in \ball_{r}^{\epsilon'}(\phi_{k}^{0})$, where $\epsilon'\df (2mL\|F\|_{L_{2r}(Q)}\epsilon)^{r}$, $\phi_{k}^{0}\df \phi\Big({1\over N}y\tilde{y}\langle \bm{f}^{0}_{\pi}(\bm{x}),\bm{f}^{0}_{\pi}(\tilde{\bm{x}})\rangle\Big)$, and the index $k\in \{0,1,\cdots, (\mathcal{N}(\epsilon,\mathcal{F},L_{2r}(Q)))^{mN}\}$. Consequently,
\begin{align}
\mathcal{N}((2mL\|F\|_{L_{2r}(Q)}\epsilon)^{r},\mathcal{F}',L_{r}(P))\leq (\mathcal{N}(\epsilon,\mathcal{F},L_{2r}(Q)))^{mN}.
\end{align}
Alternatively,
\begin{align}
\log \mathcal{N}(\epsilon',\mathcal{F}',L_{r}(P))\leq mM\log \left(\mathcal{N}\left(\dfrac{(\epsilon')^{1\over r}}{2mL\|F\|_{L_{2r}(Q)}},\mathcal{F},L_{2r}(Q)\right)\right).
\end{align}
\hfill $\blacksquare$
\section{Useful Properties of the Sub-Exponential and Sub-Gaussian Random Variables}
\label{Appendix:Properties of Sub-Exponential and Sub-Gaussian Random Variables}
Recall the notions of the sub-Gaussian and sub-exponential norms from Definition \ref{Def:Orlicz}. In the sequel we collect several useful lemmas regarding the properties of the sub-Gaussian and sub-exponential random variables:

\begin{lemma} \textsc{(Sum of Sub-Gaussian and Sub-Exponential Random Variables)}
	\label{Lemma:Sum_Sub_Gaussian}	
	There exists a universal constant $C_{s}$, such that the sum of two dependent sub-Gaussian random variables with parameters $K_{1}$ and $K_{2}$ is $C_{s}(K_{1} +K_{2})$-sub-Gaussian, and the sum of two dependent sub-exponential random variables with parameters $K_{1}$ and $K_{2}$ is $C_{s}(K_{1}+K_{2})$-sub-exponential. $\hfill \square$
\end{lemma}

\begin{lemma}\textsc{(Norm of Difference of Sub-Gaussian Random Variables)}
	\label{Lemma:Norm_of_Difference}
	Suppose $\bm{X}\in \real^{d}$ and $\bm{Y}\in \real^{d}$ are two sub-Gaussian random vectors such that $\|\bm{X}\|_{\psi_{2}}= K_{1}$ and $\|\bm{Y}\|_{\psi_{2}}= K_{2}$. Then, $Z=\|\bm{X}-\bm{Y}\|^{2}_{2}$ is sub-exponential with the Orlicz constant of
	\begin{align}
	\label{Eq:Orlicz_Bound_Difference}
	\|Z\|_{\psi_{1}}\leq 2K_{1}+2K_{2}+K_{1}K_{2}.
	\end{align}
	Furthermore, its moments satisfy the following upper bounds
	\begin{align}
	\label{Eq:Satisfy_moments}
	\expect[|Z|^{p}]\leq 2^{2p-1}d^{p}C_{2p}\|\bm{X}\|_{\psi_{2}}^{2p}+2^{2p-1}d^{p}C_{2p}\|\bm{Y}\|^{2p}_{\psi_{2}}.
\end{align}
$\hfill \square$
\end{lemma} 

\textit{Proof}. We begin the proof by expanding the squared $\ell_{2}$-norm as follows
\begin{align}
\label{Eq:expansion}
Z=\|\bm{X}\|_{2}^{2}+\|\bm{Y}\|_{2}^{2}-2\langle \bm{X},\bm{Y} \rangle
\end{align}
We now prove that each term on the right hand side is sub-exponential. First, consider the squared terms $\|\bm{X}\|_{2}^{2}$ and $\|\bm{Y}\|_{2}^{2}$. We can assume $\|\bm{X}\|_{\psi_{2}}=1$ by scale invariance. Let $\bm{W}\sim \mathsf{N}(0,\bm{I}_{d\times d})$ be independent of $\bm{X}$. It is known that $\expect[e^{\lambda Z^{2}}]=\dfrac{1}{\sqrt{1-2\lambda}}, Z\sim \mathsf{N}(0,1)$, and $\lambda<1/2$. So, we have that
\begin{align*}
\expect[\lambda\|\bm{W}\|_{2}^{2}]=\prod_{i=1}^{d}\expect[\lambda W_{i}^{2}]=\dfrac{1}{(1-2\lambda)^{d/2}},
\end{align*}
for any $\lambda<1/2$. Now, for any $\lambda>0$, we evaluate the quantity $\expect[\sqrt{2\lambda}\langle \bm{X},\bm{W} \rangle]$ in two ways by conditioning on $\bm{X}$ and then $\bm{W}$. We have
\begin{align}
\nonumber 
\expect[\exp(\sqrt{2\lambda}\langle \bm{X},\bm{W} \rangle)]&=\expect\left[\expect\left[\exp(\sqrt{2\lambda}\langle \bm{X},\bm{W} \rangle)\right]|\bm{X}\right]\\ \nonumber
&= \expect\left[\exp(\|\sqrt{2\lambda}\bm{X}\|_{2}^{2}/2)\right]\\ \label{Eq:COMBI_1}
&=\expect[e^{\lambda\|\bm{X}\|_{2}^{2}}].
\end{align}
On the other hand, we have
\begin{align}
\nonumber
\expect\Big[\exp(\sqrt{2\lambda}\langle \bm{X},\bm{W} \rangle)\Big]&=\expect\left[\expect\left[\exp(\sqrt{2\lambda}\langle \bm{X},\bm{W} \rangle)\right]\Bigg|\bm{W}\right]\\ \nonumber
&\leq \expect\left[e^{\|\sqrt{2\lambda}\bm{X}\|_{2}/2}\right]\\  \nonumber
&=\expect\left[e^{\lambda\|\bm{W}\|_{2}^{2}}\right]\\  \label{Eq:COMBI_2}
&=\dfrac{1}{(1-2\lambda)^{d/2}}.
\end{align}
Now, by Jensen's inequality $\expect\Big[e^{\lambda \|\bm{X}\|_{2}^{2}}\Big]\geq e^{\lambda\expect[\|\bm{X}\|_{2}^{2}]}$. Further, putting together Eqs. \eqref{Eq:COMBI_1} and \eqref{Eq:COMBI_2}, yields
\begin{align*}
\expect\left[e^{\lambda (\|\bm{X}\|_{2}^{2}-\expect[\|\bm{X}\|_{2}^{2}])}\right]&\leq \expect\left[e^{\lambda\|\bm{X}\|_{2}^{2}}\right]\\
&\leq \dfrac{1}{(1-2\lambda)^{d/2}}\\
&\leq \exp(4\lambda^{2}),
\end{align*}  
\textit{i.e.}, $\|\bm{X}\|_{2}^{2}$ satisfies the Laplace's property for $\lambda<1/2$, and is thus sub-exponential with the Orlicz norm of $\| \|\bm{X}\|_{2}^{2}\|_{\psi_{1}}\leq 2K_{1}$. Similarly, $\| \|\bm{Y}\|_{2}^{2}\|_{\psi_{1}}\leq 2K_{2}$. Now, consider the inner product term $\langle \bm{X},\bm{Y} \rangle$. Using the iterated law of expectation and the Laplace inequality of sub-Gaussian random variables, we obtain that for all $\lambda>0$ 
\begin{align}
\nonumber
\expect\Big[\exp(\lambda\langle \bm{X},\bm{Y} \rangle)\Big]&=\expect[\expect\Big[\exp(\lambda \langle \bm{X},\bm{Y} \rangle)|\bm{Y}\Big]]\\
\nonumber
&=\expect\left[\expect\left[\exp\left({\lambda\|\bm{Y}\|_{2}} \left\langle \bm{X},{\bm{Y}\over \|\bm{Y}\|_{2}} \right\rangle\right)\Bigg|\bm{Y}\right]\right]\\
&= \expect\left[\exp\left({1\over 2}\lambda^{2}K_{1}^{2} \|\bm{Y}\|_{2}^{2}\right)\right],
\end{align}
where we used the fact that $\bm{X}\in \real^{d}$ is a sub-Gaussian vector if $\langle \bm{u},\bm{X}\rangle$ is sub-Gaussian for any vector $\bm{u}\in \mathrm{S}^{d-1}$. Now, since $\|\bm{Y}\|_{2}^{2}$ is $2K_{2}$-sub-exponential, we conclude that if $\lambda\leq \dfrac{1}{2K_{1}^{2}K^{2}_{2}}$, we have that 
\begin{align}
\nonumber
\expect\Big[\exp(\lambda\langle \bm{X},\bm{Y} \rangle)\Big]&= \expect\Big[\exp({1\over 2}\lambda^{2}K_{1}^{2} \|\bm{Y}\|_{2}^{2})\Big],\\ \nonumber
&=\sum_{p\geq 0}\dfrac{\lambda^{2p}K_{1}^{2p}}{2^{p}p!}\expect\left[\|\bm{Y}\|_{2}^{2p}\right]\\ \nonumber
&\leq \sum_{p\geq 0}\dfrac{\lambda^{2p}K_{1}^{2p}}{p!}(2K_{2})^{2p}\\
&=\exp(\lambda^{2}K_{1}^{2}K_{2}^{2}),
\end{align}
where the inequality follows from Definition \ref{Definition:Alternative} of the sub-exponential random variable. Thus, we proved that $\|\langle \bm{X},\bm{Y} \rangle\|_{\psi_{1}}\leq K_{1}K_{2}$. Now, from the expansion in Eq. \eqref{Eq:expansion}, we observe that $Z$ is the sum of sub-exponential random variables. By Lemma \ref{Lemma:Sum_Sub_Gaussian}, $Z$ is also sub-exponential with the constant $\|Z\|_{\psi_{1}}=2K_{1}+2K_{2}+K_{1}K_{2}$.

To compute the moments of the random variable $Z$, we use the following expansion
\begin{align}
\nonumber
\expect[|Z|^{p}]&=\expect\left[\|\bm{X}-\bm{Y}\|^{2p}_{2}\right]\\
\nonumber
&\stackrel{\rm{(a)}}{\leq} \expect\left[(\|\bm{X}\|_{2}+\|\bm{Y}\|_{2})^{2p}\right]\\ \label{Eq:Hey_Hey}
&\stackrel{\rm{(b)}}{\leq} 2^{2p-1}\expect\Big[\|\bm{X}\|_{2}^{2p}\Big]+2^{2p-1}\expect\Big[\|\bm{Y}\|_{2}^{2p}\Big].
\end{align}
where $\rm{(a)}$ is due to the triangle inequality, and $\rm{(b)}$ is due to the basic inequality $|x-y|^{p}\leq 2^{p-1}(|x|^{p}+|y|^{p})$. Now, for a sub-Gaussian random vector $\bm{X}\in \real^{d}$, the following inequality holds for all $\bm{u}\in \real^{d}$ independent of $\bm{X}$ (cf. \cite[Theorem 2.1]{boucheron2013concentration}),
\begin{align}
\expect\big[\langle \bm{u},\bm{X} \rangle|^{2p}\big]\leq C_{2p}\|\bm{u}\|_{2}^{2p}\|\bm{X}\|_{\psi_{2}}^{2p},
\end{align}
where $C_{2p}=p!2^{p+1}$. Letting $\bm{u}=\bm{e}_{j}\in \real^{d}$, yields that
\begin{align}
\label{Eq:Plugged}
\expect\big[|X_{j}|_{2}^{2p}\big]\leq C_{2p}\|\bm{X}\|_{\psi_{2}}^{2p},
\end{align}
where we used the fact that $\|\bm{e}_{j}\|_{2}=1$. Now, since
\begin{align}
\expect[\|\bm{X}\|_{2}^{2p}]=\expect\left[\left(\sum_{j=1}^{d}X^{2}_{j}\right)^{p}\right]\leq d^{p-1}\sum_{j=1}^{d} \expect\big[|X_{j}|_{2}^{2p}\big]\leq C_{2p}d^{p}\|\bm{X}\|_{\psi_{2}}^{2p}.
\end{align}
Plugging Inequality \eqref{Eq:Plugged} into \eqref{Eq:Hey_Hey} yields
\begin{align}
\expect[|Z|^{p}]\leq 2^{2p-1}d^{p}C_{2p} \|\bm{X}\|_{\psi_{2}}^{2p}+2^{2p-1}d^{p}C_{2p}\|\bm{Y}\|^{2p}_{\psi_{2}}.
\end{align} 
\hfill $\blacksquare$

The proof of lemma \ref{Lemma:Sum_Sub_Gaussian} follows directly from the equivalet definition of sub-Gaussian and sub-exponential random variables based on the Orlicz norms.

\begin{lemma}\textsc{(Bernstein inequality for sub-exponential random variables)} 
\label{Lemma:Bernstein inequality}	
Let $X_{1}, \cdots,X_{n}$ be independent sub-exponential random variables with $\|X_{i}\|_{\psi_{1}}\leq b$. Then, there exists a universal constant $c>0$ such that, for all $\delta>0$,
\begin{align}
	\label{Eq:Bernstein_Inequality}
	\prob\left(\left|\sum_{i=1}^{n}(X_{i}-\expect[X_{i}])\right|\geq \delta\right)\leq \exp\left(-c\cdot \min\left({\delta^{2}\over nb^{2}},{\delta\over b}\right)\right).
\end{align}
\end{lemma}

The following lemma provides the Orlicz norm of a sub-Gaussian random variable satisfying a concentration inequality:
\begin{lemma}
	\label{Lemma:Con_to_Sub_Gaussian}
Let $Z$ denotes a random variable with zero mean $\expect[Z]=0$ satisfying the following concentration inequality:
\begin{align}
\prob\left(|Z|\geq \delta\right)\leq \exp(-\lambda\delta^{2}).
\end{align}
Then, $Z$ is sub-Gaussian with the Orlicz norm of $\|Z\|_{\psi_{2}}\leq 1/\sqrt{2\lambda}$.
\end{lemma}

\textit{Proof}: We establish an upper bound on the $2p$-moments of the random variable $|Z|$ as follows
\begin{align}
\nonumber
\expect[|Z|^{2p}]&=\int_{0}^{\infty}2pu^{2p-1} \prob(|Z|\geq u)\mathrm{d}u\\
\nonumber
&\leq \int_{0}^{\infty}2pu^{2p-1}\exp(-\lambda u^{2})\mathrm{d}u\\
&\leq p\lambda^{-p}  \Gamma\left(p \right),
\end{align}
for $p\in \integer$, where we used the fact that $\Gamma(x)\df \int_{0}^{\infty}t^{x-1}\exp(-t)\mathrm{d}t$. Now, since $\Gamma(p)=(p-1)!$, we have 
\begin{align}
\expect[\exp(|Z|^{2}/\beta^{2})]&=\sum_{p=0}^{\infty}\dfrac{1}{p!\beta^{2p}}\expect[|Z|^{2p}]\\
&\leq \sum_{p=0}^{\infty}\left(\dfrac{1}{\beta^{2}\lambda}\right)^{p}\\
&\leq \dfrac{1}{1-\beta^{2}\lambda}.
\end{align}
Therefore, when $\beta\leq1/\sqrt{2\lambda}$, we have  $\expect[\exp(|Z|^{2}/\beta^{2})]\leq 2$, \textit{i.e.}, $\|Z\|_{\psi_{2}}\leq 1/\sqrt{2\lambda}$. \hfill $\blacksquare$
		
\begin{lemma}
\label{Lemma:Generalized Sub-Gaussian Maximal Inequality}
\textsc{(Sub-Gaussian Maximal Inequality)} Let $X_{1},\cdots,X_{n}$ denotes centered sub-Gaussian random variables with the variance proxy of $\sigma^{2}$. Then, for any $p\geq 1$, we have 
\begin{align}
\expect\Big[\max_{1\leq i\leq N} |X_{i}|^{p}\Big]\leq \sigma^{p}\left(\log(2N)\right)^{p\over 2}.
\end{align}
\begin{proof}
The proof for the case of $p=1$ can be found in literature; see, \textit{e.g.}, \cite{van1996maximal}. To establish the proof for the general case, consider the Orlicz function $\psi(x)=e^{x^{2\over p}}-1$. Since for $p\geq 2$, $\psi$ is a non-negative, covex, strictly increasing function on $\real_{+}$, we have
\begin{align}
\nonumber
\psi\left(\expect\left[\max_{1\leq i\leq N} \dfrac{|X_{i}|^{p}_{i}}{\sigma^{p}}\right] \right)&\leq \expect\left[\max_{1\leq i\leq N}\psi \left(\dfrac{|X_{i}|^{p}_{i}}{\sigma^{p}}  \right)\right]\\
\nonumber
&\leq \sum_{i\leq N}\expect\left[\psi\left(\dfrac{|X_{i}|^{p}}{\sigma^{p}} \right)\right]\\
\nonumber
&\leq \sum_{i\leq N}\expect\left[\exp\left(\dfrac{|X_{i}|^{2}}{\sigma^{2}} \right)\right]-1\\
\nonumber
&\leq N,
\end{align}
where the last step follows from the definition of the sub-Gaussian random variable. Taking the inverse we obtain
\begin{align}
\nonumber
\expect\left[\max_{1\leq i\leq N} {|X_{i}|^{p}_{i}}\right] &\leq \sigma^{p}\psi^{-1}(N)\\ \nonumber
&\leq \sigma^{p}(\log(N+1))^{p\over 2}\\
&\leq \sigma^{p}(\log(2N))^{p\over 2}.
\end{align}

\end{proof}
\end{lemma}

\bibliographystyle{amsplain}
\bibliography{mybib19}

\end{document}